\documentclass{article}

\makeatletter
\def\input@path{{./styles/neurips/}{./styles/shared/}}
\makeatother

\PassOptionsToPackage{numbers,sort&compress}{natbib}
\usepackage[main,final]{styles/neurips/neurips_2026}
\usepackage[utf8]{inputenc}
\usepackage[T1]{fontenc}
\usepackage{microtype}
\usepackage{graphicx}
\usepackage{wrapfig}
\usepackage{needspace}
\graphicspath{{./}}
\usepackage{subcaption}
\usepackage{booktabs}
\usepackage{makecell}
\usepackage{hyperref}
\usepackage{url}
\usepackage{xcolor}
\usepackage{amsmath}
\usepackage{amssymb}
\usepackage{amsfonts}
\usepackage{mathtools}
\usepackage{amsthm}
\usepackage[capitalize,noabbrev]{cleveref}

\newcommand{\initA}{\texttt{Init[A]}\xspace}
\newcommand{\initAB}{\texttt{Init[AB]}\xspace}
\newcommand{\initAA}{\texttt{Init[$\text{A}^2$]}\xspace}

\newcommand{\initB}{\texttt{Init[B]}\xspace}
\newlength{\appendixtocitemsep}
\newlength{\appendixtocsectionsep}
\newcommand{\appendixtocsection}[2]{%
  \par\addvspace{\appendixtocsectionsep}%
  \noindent\textbf{\hyperref[#1]{#2}}\par\vspace{0.1em}%
}
\newcommand{\appendixtocitem}[2]{%
  \noindent\hspace*{1.5em}-- \hyperref[#1]{#2}\par\vspace{\appendixtocitemsep}%
}

\theoremstyle{plain}
\newtheorem{condition}{Condition}

\input{watml-math}

\usepackage[textsize=tiny]{todonotes}

\title{Rethinking LoRA Initialization for Robust Asymmetric Learning Rates}

\author{%
  Disen Liao \\
  University of Waterloo \\
  Vector Institute \\
  \texttt{d7liao@uwaterloo.ca}
  \And
  Yaoliang Yu \\
  University of Waterloo \\
  Vector Institute \\
  \texttt{yaoliang.yu@uwaterloo.ca}
}
\begin{document}

\maketitle

\begin{abstract}
    Low-Rank Adaptation (LoRA) is often highly sensitive to learning-rate (LR) choices. Recent empirical studies suggest that, once LRs are carefully tuned, vanilla LoRA can remain competitive with many proposed LoRA variants. This shifts the central practical question from designing ever more variants to reducing the cost and brittleness of LR selection. LoRA+ addresses this issue with asymmetric LRs, using a larger LR for the up-projection matrix $B$ than for the down-projection matrix $A$. However, the ratio $\eta_B / \eta_A$ itself becomes an additional hyperparameter: large ratios can improve adaptation, but under the standard LoRA initialization they can also cause unstable training and ratio-dependent performance collapse. We identify the source of this brittleness through a signal--noise decomposition of LoRA dynamics. Under standard \initA, the update contains an initialization-induced noise term whose magnitude can be affected by the LR for $B$, $\eta_B$, the initialization scale $\beta$, and the adapter rank $r$. In practical LLM fine-tuning, this non-vanishing noise can limit the performance gains expected from LoRA+ style asymmetric LRs. To suppress this effect, we propose \initAA, a simple initialization that scales the variance of $A$ as $\sigma_A^2 \propto n^{-2}$ while keeping $B_0=0$. Under matched rank and training time for both initializations, this reduces the initialization-noise energy by an additional factor of $1/n$, stabilizing asymmetric LoRA training while preserving zero initial adapter output. Empirically, \initAA substantially widens the stable LR region across initialization scales, adapter ranks, and LoRA update multipliers. Across tasks and models, \initAA reduces best-ratio drift, prevents the high-ratio collapse observed with standard \initA, and makes the width-based guideline $\eta_B \approx n \eta_A$ a useful starting point for finite pretrained models.
\end{abstract}

\section{Introduction}
Low-Rank Adaptation (LoRA) \citep{hulora} is a standard method for parameter-efficient fine-tuning of large language models (LLMs) \citep{houlsby2019parameter,li2021prefix,lester2021power}. Given a width-$n$ LoRA layer with pretrained weight $W\in\Rbb^{n\times n}$ and input $\underline{Z} \in \Rbb^{n}$, LoRA uses
$\bar{Z}=W\underline{Z}+\frac{\alpha}{r}BA\underline{Z}$,
where $A\in\Rbb^{r\times n}$ and $B\in\Rbb^{n\times r}$ are trainable with learning rates (LRs) $\eta_A$ and $\eta_B$. Here $r \ll n$ is the adapter rank and $\alpha/r$ is the LoRA update multiplier. We use $\beta$ to denote the initialization-magnitude hyperparameter: standard LoRA initialization, denoted \initA, initializes $B_0=0$ and initializes $A_0$ randomly with mean zero and standard deviation $\beta n^{-1/2}$, so the adapter preserves the pretrained model output at step zero.

\begin{figure}[!t]
    \centering
    \includegraphics[width=1\linewidth]{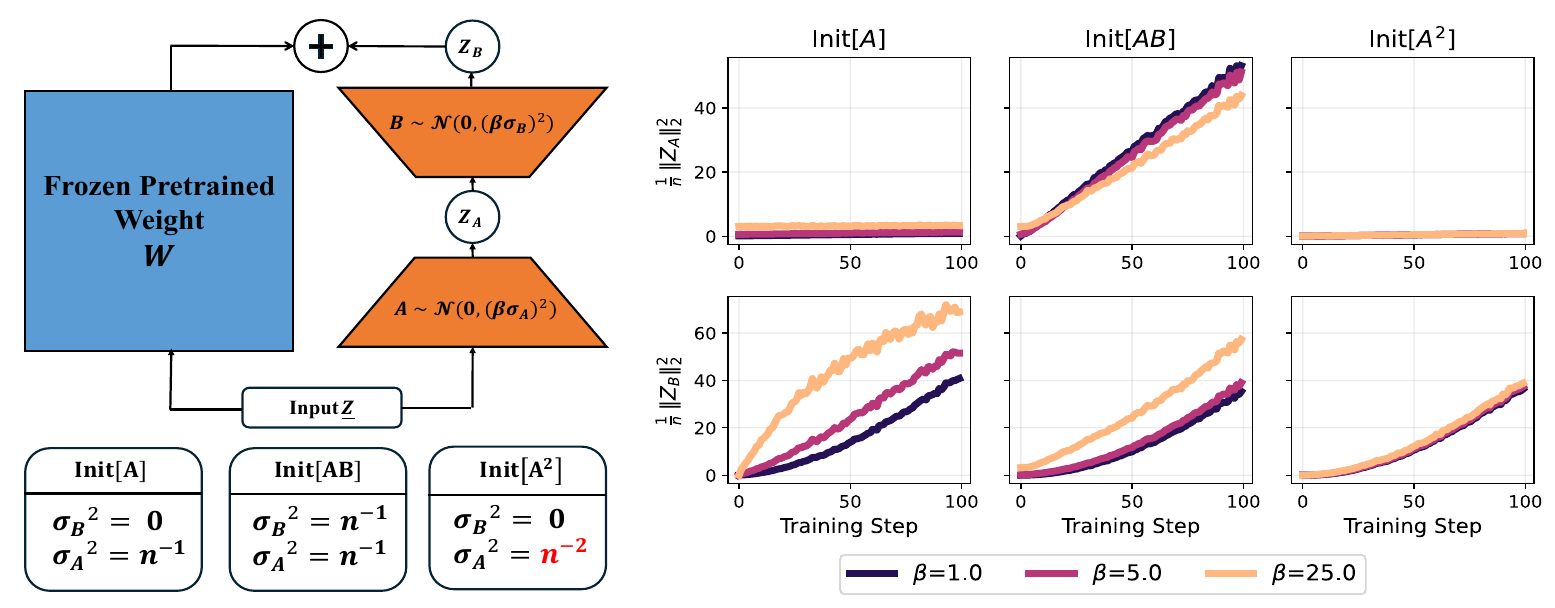}
    \caption{Overview of LoRA initialization and feature-norm stability. \textbf{Left}: LoRA branch notation and the initialization scalings compared in this work, including standard \initA, \initAB, and the proposed \initAA. \textbf{Right}: feature-norm trajectories at width $n=4096$ for different initialization scales $\beta$. \initAA keeps the LoRA features nearly invariant to $\beta$, while the alternative scalings exhibit stronger initialization-dependent growth.}
    \label{fig:intro_overview}
    \label{fig:norms}
    \vspace{-1.5em}
\end{figure}

Despite its simplicity, LoRA is sensitive to hyperparameters (HPs). The rank $r$ controls the expressive capacity of the adapter \citep{zengexpressive}, but its effect is
inseparable from the update multiplier $\lambda_r=\alpha/r$. Kalajdzievski
~\citep{kalajdzievski2023rank} shows that the conventional choice $\alpha/r$
is not rank-stabilized: it can shrink adapter activations and gradients as
$r$ grows, slowing learning at high ranks and motivating the alternative
$\alpha/\sqrt r$. LRs introduce a further source of brittleness. Recent empirical work shows that, after sufficiently broad LR sweeps, vanilla LoRA can remain competitive with many proposed LoRA variants \citep{lee2026learning}. This shifts the practical question from only designing new adapters to understanding LoRA HP selection.

LoRA+ \citep{hayou2024lora+} gives a principled route toward understanding LR selection by assigning different LRs to the two LoRA matrices. Its infinite-width feature learning analysis suggests that efficient LoRA feature
updates require an asymmetric regime with $\eta_B \gg \eta_A$ and, under idealized scaling assumptions, a ratio $\eta_B/\eta_A \approx n$. We write this ratio as $\rho=\eta_B/\eta_A$. The asymmetric regime is attractive because $B$ directly controls the LoRA branch output, so increasing $\eta_B$ can accelerate adaptation. In finite pretrained LLMs, however, $\rho$ often becomes another hyperparameter to tune: large ratios can improve performance, but under \initA they can also trigger instability. We address the gap between the theoretical LR guidance and practical LLM fine-tuning in this paper.

We show that the fragility of asymmetric LRs comes from an initialization-dependent noise term that survives even though the LoRA branch has zero output at initialization. Although $B_0=0$, the random matrix $A_0$ affects the branch as soon as $B$ moves away from zero. A signal--noise decomposition of the LoRA feature update reveals two components: a task-driven signal from optimizer updates and a noise term from the interaction between the learned $B$ and the random initial $A$. Both terms scale with $\eta_B$, so increasing $\eta_B$ amplifies useful adaptation and initialization noise at the same time. In the small-LR regime common in pretrained LLM fine-tuning, the signal can remain controlled while the \initA noise contribution does not vanish with width, making large $\rho$, large initialization scale $\beta$, and large update multiplier $\alpha/r$ especially brittle.

We propose \initAA, a minimal initialization change that suppresses this noise at its source.\footnote{Our code is available at \url{https://github.com/liaodisen/LoRA-initA2}.} \initAA keeps the standard zero initialization for $B$, but initializes $A$ with one additional factor of width scaling: instead of $\sigma_A^2=\Theta(n^{-1})$, it uses $\sigma_A^2=\Theta(n^{-2})$. This preserves the LoRA architecture, training objective, and zero initial adapter output. The purpose is not to make the adapter inactive, but to prevent the random initial projection from being amplified by $\eta_B$, $\alpha/r$, or $\beta$ before a task-driven signal has formed. For comparable rank and time scaling of the optimizer updates, this reduces the initialization-noise energy by a factor of $1/n$. We do not claim that $\rho=n$ is always the finite-model optimum; instead, \initAA reduces the need for exhaustive two-dimensional LR search by making the high-ratio regime substantially more stable.

We validate this claim through LR sweeps and fine-tuning experiments. On a controlled toy problem, \initAA improves hyperparameter transfer across widths and keeps internal and output LoRA feature norms stable as $\beta$ grows. On encoder-only RoBERTa fine-tuning and decoder-only arithmetic and commonsense benchmarks, \initAA preserves the benefit of $\eta_B>\eta_A$ while widening the stable LR region and reducing best-ratio drift. These results position \initAA as a lightweight initialization rule for making LoRA+-style asymmetric LR tuning practical in finite pretrained models.

Our contributions are threefold:
\vspace{-0.4em}
{\setlength{\leftmargini}{2em}

\begin{itemize}

\item We identify an LR-ratio failure mode in LoRA: asymmetric LRs with $\eta_B>\eta_A$ can improve adaptation, but under standard \initA the ratio $\rho$ becomes brittle because $\eta_B$ amplifies both the useful signal and an initialization-dependent noise term.

\item We introduce \initAA, a simple LoRA initialization rule that sets $\sigma_A^2=\Theta(n^{-2})$ while keeping $B_0=0$. It preserves the feature-learning property of standard \initA under the same asymmetric LR scaling, while suppressing the initialization-noise term without changing the LoRA architecture, objective, zero-output initialization, or training cost.

\item We show theoretically and empirically that \initAA widens the stable LR region, improves robustness to large initialization scales and LoRA update multipliers, and makes asymmetric LR selection more predictable across models and tasks.

\end{itemize}

}

\section{Background}
\subsection{LoRA Parameterization and Features}

Following \citet{hayou2024lora+}, we consider a general $L$-layer neural network of the following form:
\begin{equation}
\label{eq:lora_formulation}
    \begin{cases}
        Y_{0} = W_{\text{in}} X, \\
        Y_\ell (X) = \Fcal_\ell(W_\ell, Y_{\ell - 1}(X)), \ell \in [L], \\
        Y_{\text{out}}(X) = W_{\text{out}} Y_L (X),
    \end{cases}
\end{equation}
where $X \in \Rbb^{d_{\text{in}}}$ denotes the input vector, $n$ denotes the network width, and $\Fcal$ denotes a nonlinear mapping function. $W_\ell \in \Rbb^{n \times n}$ is the weight matrix of the $\ell$-th hidden layer. $W_{\text{in}} \in \Rbb^{n \times d_{\text{in}}}$ and $W_{\text{out}} \in \Rbb^{d_{\text{out}} \times n}$ are the input and output layers, respectively. To fine-tune this model, LoRA is applied to a subset of $\{W_{\text{out}}, W_1, \dots, W_L, W_{\text{in}}\}$. For clarity, the analysis below treats square width-$n$ layers; the rectangular implementation is specified in \cref{app:init_implementation}. LoRA is defined as:

\begin{definition}[LoRA]
\label{def:lora}
For a width-$n$ LoRA layer with frozen pretrained weight $W \in \Rbb^{n \times n}$, we seek an update $W \leftarrow W + \Delta W$ for a specific downstream task. LoRA parameterizes the update as $\Delta W = \frac{\alpha}{r} B A$. Here, only $B \in \Rbb^{n \times r}$ and $A \in \Rbb^{r \times n}$ are trainable. The rank satisfies $r \ll n$, and $\alpha/r$ is the LoRA update multiplier.
\end{definition}

For a given LoRA layer, let $\underline{Z}$ be its input and let $\bar{Z}$ be its output. By \cref{def:lora}, $\bar{Z} = W \underline{Z} + \frac{\alpha}{r} B A \underline{Z}$. Since $W$ is frozen during fine-tuning, we focus on the scaled LoRA branch update, $\frac{\alpha}{r} B A \underline{Z}$. Following \citet{hayou2024lora+}, we define the LoRA features:

\begin{definition}[LoRA Features]
\label{def:lora_feature}
Given a LoRA layer with input $\underline{Z}$ and matrices $A, B$, we define LoRA features as $Z_A = A \underline{Z}$ and $Z_B = \frac{\alpha}{r} B Z_A = \frac{\alpha}{r} B A \underline{Z}$. The superscript $t$ (e.g., $Z_A^t, Z_B^t$) denotes the feature value at the $t$-th step of fine-tuning.
\end{definition}

\subsection{Prior Work on LoRA with Scaling Theory}

The theoretical analysis of deep neural networks is often impeded by nonlinearities and high-dimensional interactions. To make the behavior tractable, prior work often studies these models in infinite-width or infinite-depth limits, leveraging asymptotic properties governed by the Central Limit Theorem (CLT) and the Law of Large Numbers (LLN), reviewed in \cref{app:two_scaling_laws}. A critical objective within this framework is to identify scaling laws that avoid the ``lazy'' kernel regime, thereby preserving the model's ability to perform nontrivial feature learning \citep{yang2021tensor} as dimensions grow. This scaling-limit perspective has facilitated the derivation of feature-preserving HP transfer rules for Transformers \citep{yang2022tensor, dey2025don} and Residual Networks \citep{bordelon2023depthwise, yang2023feature}, as well as the analysis of low-precision training stability \citep{blakeu}. More recently, these techniques have been adapted to the analysis of LoRA, offering insights into initialization schemes and LR scalings that support feature adaptation \citep{hayou2024impact, hayou2024lora+, libeyond}. We discuss more related work in \cref{app:related_work}.

\paragraph{LoRA Feature Learning.} To derive optimal scaling rules, \citet{hayou2024lora+} analyze the fine-tuning process through a \emph{coordinate-wise} feature learning perspective. When $\alpha/r=1$, the update $\Delta Z_B^t:=Z_B^{t+1}-Z_B^t$ decomposes into three interaction terms, where $\Delta Z_A^t:=Z_A^{t+1}-Z_A^t$:
\begin{equation}
\Delta Z_B^t = \underbrace{B_t \Delta Z_A ^t}_{\delta_1^t} + \underbrace{\Delta B_t Z_A ^t}_{\delta_2^t} + \underbrace{\Delta B_t \Delta Z_A^t}_{\delta_3^t}.
\end{equation}
\begin{condition}
\label{cond:hayou_criteria}
\citet{hayou2024lora+} propose two conditions on this update:

{\setlength{\leftmargini}{2em}
\begin{enumerate}
    \item \textbf{Stability:} for all $t > 1$, $Z_B^t$ has $\Theta(1)$ coordinate size as the width $n \rightarrow \infty$.
    \item \textbf{Efficiency:} for all $t > 1$, both $\delta_1^t$ and $\delta_2^t$ have $\Theta(1)$ coordinate size as the width $n \rightarrow \infty$.
\end{enumerate}
}
\end{condition}
To systematically analyze these constraints, \citet{hayou2024lora+} employ the $\gamma$-calculus, which quantifies the scaling of a variable $v$: if $v = \Theta(n^\lambda)$, then $\gamma[v] = \lambda$. For a matrix or vector $X$, $\gamma[X] = \lambda$ means that each coordinate scales as $\Theta(n^\lambda)$ (see \cref{app:gamma} for details). By applying this framework to \cref{cond:hayou_criteria}, \citet{hayou2024lora+} demonstrate that standard symmetric LRs cannot satisfy both stability and efficiency simultaneously. Consequently, they advocate for asymmetric configurations, specifically \initA ($B_0 = 0, \sigma_A^2 = \Theta(1/n)$) paired with $\eta_A = \Theta(1/n)$ and $\eta_B = \Theta(1)$. Using a similar $\gamma$-calculus analysis under the equal-LR setting (i.e., $\eta_A = \eta_B = \eta^*$), \citet{hayou2024impact} show that \initA is better than \initB because it can tolerate a larger LR. Based on related analysis, \citet{libeyond} challenge the necessity of zero initialization. Introducing the concept of \textit{robustness}, they investigate whether initialization can preserve stability even when HPs deviate from optimal $\gamma$-scales. Their findings suggest that fine-tuning is more stable under $\eta_A = \eta_B$ when both matrices are initialized nonzero (i.e., \initAB).

\section{Analysis}
Previous work analyzes LoRA through a feature learning perspective, but achieving feature learning can be challenging in practice. Although \citet{hayou2024lora+} propose the scaling $\eta_B \gg \eta_A$, the best ratio for a finite model still requires costly tuning. Furthermore, practical deployments may use a large LoRA update multiplier $\alpha/r$ to encourage rapid adaptation, even when LRs are small. In this section, we show that standard \initA then has initialization-induced noise that can accumulate over training, whereas our proposed \initAA suppresses this source of instability.

\subsection{Update Decomposition}

We consider the LoRA setup defined in \cref{eq:lora_formulation}. We explicitly parameterize the initialization magnitude using a scalar $\beta$. Let $\tilde A_0$ be a base initialization matrix with entries i.i.d. $\Ncal(0, \sigma_A^2)$. The trainable matrices are initialized as $B_0 = 0$ and $A_0 = \beta \tilde A_0$.

\paragraph{Normalized optimizer model.}
We separate the scalar LR scale from the matrix-valued optimizer direction. Specifically, let $U_A^t \in \Rbb^{r \times n}$ and $U_B^t \in \Rbb^{n \times r}$ denote the directions produced by the optimizer at step $t$, after any coordinate-wise normalization or adaptive preconditioning but before multiplication by the LoRA-specific LRs:
\[
A_{t+1} = A_t - \eta_A U_A^t, \qquad
B_{t+1} = B_t - \eta_B U_B^t,
\]
where the entries of $U_A^t$ and $U_B^t$ have coordinate size $O(1)$ uniformly in the width $n$. This formulation abstracts away the optimizer-specific normalization while keeping the LoRA-specific LR scale explicit. The algebraic identities derived below (\cref{lemma:feature_decomposition}) hold for any realized sequence of such update directions; optimizer-specific assumptions are needed only when estimating their scaling with width. SignSGD is the special case $U=\operatorname{sign}(\nabla \Lcal)$, while Adam and AdamW also fit this normalized-direction model under the non-degeneracy conditions given in \cref{app:adam_extension}.
\paragraph{Update-Input alignment.} 
We make an assumption on the alignment between block $A$ update and input $\underline{Z}$ following the LoRA+ analysis \citep{hayou2024lora+}. To see this, 
for a single token trained with loss $\Lcal$: 
\begin{equation*}
\nabla_A \Lcal_t=\frac{\alpha}{r}B_t^\top d\bar Z_t\, \underline{Z}^\top,
\end{equation*}
so each row of the $A$-gradient is a scalar multiplication of the input. We therefore assume optimizer normalization preserves this input coupling:
\[
    |(U_A^t)_{i,:}\underline Z|=\Theta(n), \qquad i\in[r].
\]
Sufficient conditions for SignSGD, Adam, and AdamW are given in \cref{app:adam_extension}. However, this condition is not generally held in the minibatch setting,  the measured minibatch alignment in \cref{app:empirical_width_alignment,fig:width_alignment} varies by module and training stage, so this condition is not asserted generically for minibatch fine-tuning.

Using these update equations, we can unroll $A_t$ and $B_t$ and separate the contribution of the random initialization $A_0$ (noise) from that of the accumulated optimizer updates in both LoRA factors (signal). The following lemma gives exact decompositions of both the LoRA feature and its one-step update. The alignment condition above is not needed for the decomposition itself and enters only in the subsequent scaling analysis. 

\begin{lemma}[Signal--noise decompositions of LoRA features]
\label{lemma:feature_decomposition}
For any step $t \ge 1$, both the one-step LoRA feature update and the feature itself decompose into initialization-noise and optimizer-driven signal terms:
\begin{align*}
    \Delta Z_B^t
    &= \underbrace{- \frac{\alpha\beta}{r} \eta_B U_B^t \tilde A_0 \underline{Z}}_{\Delta T_{\text{noise}}^t}
    + \underbrace{\frac{\alpha}{r} \eta_A \eta_B S_t \underline{Z}}_{\Delta T_{\text{signal}}^t},\\
    Z_B^t
    &= \underbrace{- \frac{\alpha \beta}{r} \eta_B \Big(\sum_{j=0}^{t-1} U_B^j\Big) \tilde A_0 \underline{Z}}_{T_{\text{noise}}^t}
    + \underbrace{\frac{\alpha}{r} \eta_A \eta_B \Big(\sum_{j=0}^{t-1} U_B^j\Big)\Big(\sum_{i=0}^{t-1} U_A^i\Big) \underline{Z}}_{T_{\text{signal}}^t},
\end{align*}
where
$S_t := (\sum_{j=0}^{t-1} U_B^j) U_A^t + U_B^t (\sum_{i=0}^{t-1} U_A^i) + U_B^t U_A^t$.
\end{lemma}

Lemma~\ref{lemma:feature_decomposition} separates the LoRA feature $Z_B^t$ into an initialization-dependent component $T_{\text{noise}}^t$ and an optimizer-driven component $T_{\text{signal}}^t$. To compare their magnitudes across model widths and training time, we consider the width-normalized squared Euclidean norm $E_{\text{total}}^t := \frac{1}{n} \| Z_B^t \|_2^2$, which is the mean squared magnitude per output coordinate. We refer to this quantity as feature energy. It can be decomposed into three terms:
\begin{equation*}
    E_{\text{total}}^t=E_N^t+E_S^t+E_X^t.
\end{equation*}
$E_N^t := \frac{1}{n} \|T_{\text{noise}}^t\|_2^2$ is the noise energy, quantifying the contribution carried by the initial random projection $\tilde{A}_0 \underline{Z}$. It isolates the component directly targeted by reducing the initialization variance. The signal energy, $E_S^t:=\frac{1}{n}\|T_{\text{signal}}^t\|_2^2$, measures the contribution from the interaction of accumulated $A$- and $B$-updates. Comparing it with $E_N^t$ reveals whether the optimizer-driven signal dominates the initialization noise. The cross term, $E_X^t:=\frac{2}{n}\langle T_{\text{noise}}^t,T_{\text{signal}}^t\rangle$, accounts for reinforcement or cancellation between the two components. Unlike the component energies, it is signed and cannot generally be discarded, since the optimizer directions depend on the initialization.

We can give an upper and a lower bound for the feature energy: by Cauchy--Schwarz, $|E_X^t|\le 2\sqrt{E_N^tE_S^t}$, giving that
\[
(\sqrt{E_N^t}-\sqrt{E_S^t})^2\le E_{\text{total}}^t\le(\sqrt{E_N^t}+\sqrt{E_S^t})^2.
\]
These bounds show that when $E_N^t/E_S^t\to0$, both the noise and cross term become negligible relative to the signal, so $E_{\text{total}}^t/E_S^t\to1$. \Cref{fig:llama_feature_decomposition} illustrates this signal dominance at finite width: under \initAA, the noise term is dominated by the signal in both Llama models, and the total feature norm closely tracks the signal norm. Measurements of those quantities are reported in \cref{app:empirical_energy_decomposition}.

\begin{figure}[t]
    \centering
    \includegraphics[width=\linewidth]{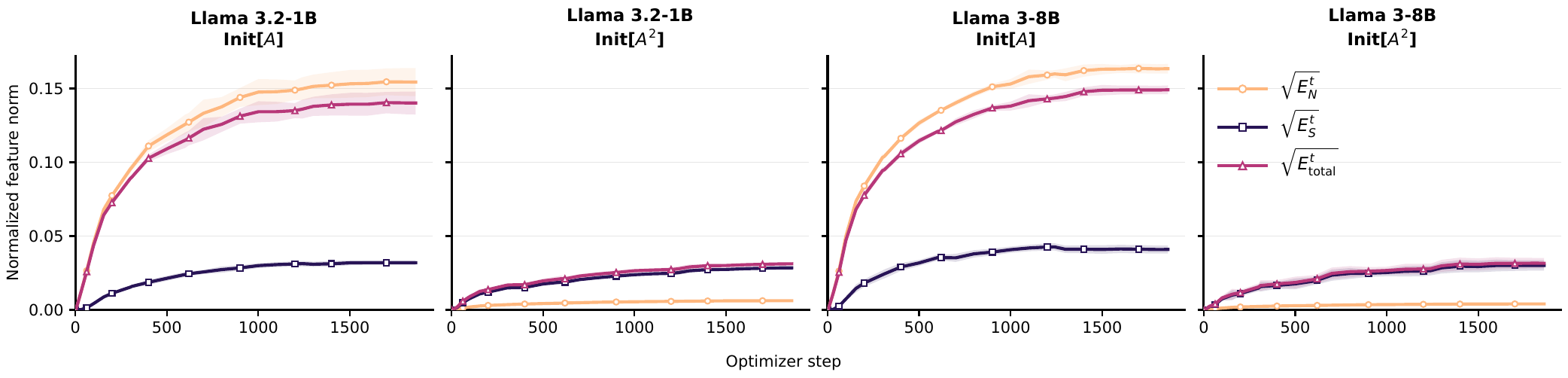}
    \caption{LoRA feature decomposition during fine-tuning of Llama 3.2-1B and Llama 3-8B. Curves show the RMS magnitudes $\sqrt{E_N^t}$ (noise), $\sqrt{E_S^t}$ (signal), and $\sqrt{E_{\text{total}}^t}$ (total). Standard \initA is noise-dominated, whereas under \initAA the noise is subdominant and the total norm closely follows the signal norm.}
    \vspace{-0.5cm}
    \label{fig:llama_feature_decomposition}
\end{figure}

\subsection{Robustness Analysis}

We aim to keep the LoRA feature $Z_B$ bounded as training progresses. To define this form of robustness, we first choose a scaling for the training horizon $T(n)$ as a function of model width $n$.

According to the neural scaling laws derived by \citet{kaplan2020scaling}, optimal training steps $S$ scale very weakly with model size. Specifically, they find $S \propto C_{\min}^{0.03}$ and $N_{\text{param}} \propto C_{\min}^{0.73}$, for some fixed efficient computation budget $C_{\min}$, which implies $S \propto N_{\text{param}}^{0.04}$. Since the parameter count $N_{\text{param}}$ scales quadratically with width ($N_{\text{param}} \propto n^2$), the training steps scale as $S \propto n^{0.08}$. Even under less optimal, fixed-compute regimes, training horizons grow sub-linearly. This motivates studying sub-$\sqrt{n}$ horizons, but we do not claim that the pretraining scaling will directly determine the fine-tuning schedules. Thus, we assume the training horizon satisfies $T(n) = o(\sqrt{n})$.

\begin{definition}[Robustness of LoRA Dynamics] Let $T(n)$ be a time horizon scaling as $T(n) = o(\sqrt{n})$. The LoRA dynamics are robust up to horizon $T(n)$ if:
    \begin{equation*}
         \sup_{0 \le t \le T(n)} \frac{1}{n}\| Z_B^t \|_2^2 = O(1), \quad \text{as } n \to \infty.
    \end{equation*}
\end{definition}
This condition is similar to the stability condition in \cref{cond:hayou_criteria}.
At a given step, a large $\beta$ can make the initialization-dependent component dominate under matched contraction scales.

\begin{theorem}[Component Scaling (Informal)]
\label{theorem:lora_feature_beta}
Under the initialization and normalized updates above, fix $\underline Z$ with $\|\underline Z\|_2^2=\Theta(n)$ and write $C_t:=\sum_{s<t}U_B^s$, $u:=\tilde A_0\underline Z$, $v_t:=(\sum_{s<t}U_A^s)\underline Z$, and $\rho_d(x):=\|x\|_2/\sqrt d$. For the analyzed steps $t\ge2$, assume $\rho_r(u)=\Theta(\sigma_A\sqrt n)$, $\rho_r(v_t)=\Theta(nt)$, $\rho_n(C_tu)=\Theta(tr^{\gamma_N}\rho_r(u))$, and $\rho_n(C_tv_t)=\Theta(tr^{\gamma_S}\rho_r(v_t))$, with fixed exponents $\gamma_N,\gamma_S$ and constants uniform in $n,r,t$. Set $m_N:=(\alpha/r)r^{\gamma_N}$ and $m_S:=(\alpha/r)r^{\gamma_S}$. Then
\[
E_N^t=\Theta(\eta_B^2\beta^2m_N^2\sigma_A^2nt^2),\qquad
E_S^t=\Theta((\eta_A\eta_Bm_Snt^2)^2).
\]
The full statement and proof are in \cref{app:derivation_of_theorem}.
\end{theorem}

\paragraph{Interpretation.}
The exponents $\gamma_N, \gamma_S$ describe realized rank contraction for $C_t u$ and $C_t v_t$, respectively: $\gamma=1$ is coherent, LLN-like scaling and $\gamma=1/2$ is CLT-like scaling, without asserting either probabilistic limit law. The noise and signal need not share an exponent. The explicit factors of $t$ are separate temporal-accumulation assumptions, not consequences of rank coherence or per-step alignment. The probe measurements in \cref{fig:rank_scaling_coherence} motivate model- and stage-dependent exponents but do not verify these accumulated contractions. At fixed rank and $\alpha/r=\Theta(1)$, both $m_N,m_S=\Theta(1)$, preserving the width comparison. When rank varies, the energy ratio also contains $r^{2(\gamma_N-\gamma_S)}$. These interactions motivate joint tuning of rank, multiplier, and LRs \citep{chen2026learningrate} and are consistent with \cref{fig:loss_plot}.

For fixed rank, $\alpha/r=\Theta(1)$, and finite time, practitioners often take $\eta_A\eta_B\ll n^{-1}$ to keep the signal small under the stated assumptions. Standard \initA nevertheless leaves a persistent initialization-dependent component because $\sigma_A^2=\Theta(n^{-1})$:
$$E_{N,\initA}^t = \Theta(\eta_B^2 \beta^2 m_N^2 t^2).$$
For fixed rank and fixed $\eta_B,\beta,t$, this noise does not vanish with width. At fixed $\alpha/r$, its rank factor is $r^{2\gamma_N}$: $r^2$ for LLN-like contraction and $r$ for CLT-like contraction. Because $\eta_B$ multiplies both components, the LR for block $B$ controls their common scale. We examine this by freezing block $A$ and sweeping only $\eta_B$ in \cref{app:block_b_important,fig:freeze_A}.

Our proposed \initAA strategy sets $\sigma_A^2 = \Theta(n^{-2})$, giving
$$E_{N,\initAA}^t = \Theta\left(\frac{\eta_B^2 \beta^2 m_N^2 t^2}{n}\right).$$
Under matched projection, contraction, and temporal-accumulation scales, this gives an additional factor $1/n$ in noise energy while preserving the signal's scaling order; it does not assert identical optimizer trajectories. Nor does it imply a $1/n$ reduction in total energy: when $E_N^t/E_S^t\to0$, the total becomes signal-dominated. For $\eta_B,\beta=\Theta(1)$, $m_NT(n)=o(\sqrt n)$ suffices for noise to vanish uniformly on the analyzed window, reducing to $T(n)=o(\sqrt n)$ at fixed rank. Robustness additionally requires $\eta_A\eta_Bm_SnT(n)^2=O(1)$ to control the signal. Our finite-width experiments test dominance and stability, not the asymptotic $1/n$ rate itself.
Under the finite-time assumptions below, the extra factor $1/n$ makes the initialization noise subdominant to the signal. The same conditional scaling calculation permits stronger suppression with \texttt{Init[$\text{A}^3$]} and \texttt{Init[$\text{A}^4$]} (\ie, $\sigma_A^2=1/n^3$ or $1/n^4$), evaluated in panel~(a) of \cref{fig:llama_additional_subfigures}.

\subsection{\initAA Still Enables Feature Learning}

Having established conditions for suppressing initialization noise and controlling feature energy, we next verify that \initAA still supports nontrivial feature learning, which underlies HP transfer \citep{yang2022tensor} and LoRA rank transfer to full fine-tuning \citep{chen2026learningrate}. A potential concern is that reducing $\sigma_A^2$ might suppress the signal required for adaptation. In the fixed-rank setting below, the accumulated optimizer updates maintain the signal while \initAA suppresses initialization noise.

\begin{theorem}[Finite-Time Stability and \initAA Efficiency (Informal)]
\label{theorem:initAA_feature_learning}
In the fixed-rank setting of \cref{theorem:lora_feature_beta}, let $t\ge2$ be finite,
$\beta,\alpha/r=\Theta(1)$, $\eta_A=\Theta(n^{-1})$, and $\eta_B=\Theta(1)$.
Assume $C_tv_t$, $C_tU_A^t\underline Z$, and $U_B^tv_t$ have
$\Theta(n)$ coordinate magnitudes. Then \initAA makes the
initialization-dependent terms $o(1)$ while preserving $\Theta(1)$
feature updates, and therefore satisfies the stability and efficiency
conditions in \cref{cond:hayou_criteria}. Standard \initA guarantees
bounded feature updates, but a non-degenerate lower bound additionally
requires a non-cancellation condition.
\end{theorem}

The proof is given in \cref{app:derivation_of_theorem}. Under \initAA,
the initialization noise satisfies $E_N^t=\Theta(n^{-1})$, whereas the
learned signal remains $E_S^t=\Theta(1)$. Hence the signal dominates and
$E_{\mathrm{total}}^t/E_S^t\to1$. Under standard \initA, both components
can remain $\Theta(1)$, so the total feature norm is bounded but may become
small through cancellation. Thus \initAA guarantees non-degenerate feature
learning under the stated assumptions, whereas \initA requires an additional
non-cancellation assumption.

\subsection{Internal Stability and Bottleneck Dynamics}

A recently proposed alternative, \initAB \citep{libeyond}, seeks to improve robustness by initializing both matrices with $\Theta(n^{-1/2})$ scaling and employing symmetric LRs $\eta_A = \eta_B = \Theta(n^{-1/2})$. While this approach balances the global update norm, it inadvertently compromises the stability of the low-rank bottleneck: it fails to constrain the internal LoRA features $Z_A$ to $\Theta(1)$, leading to feature explosion in the large-width limit.

For stable feature learning, it is crucial that not only the final output $Z_B$ but also the intermediate representation $Z_A$ evolves at a rate independent of the model width $n$. Consider the update to the internal feature $\Delta Z_A^t$:
\begin{equation*}
    \Delta Z_A^t = (A_{t+1} - A_t) \underline{Z} = -\eta_A U_A^t \underline{Z}.
\end{equation*}
By the assumed alignment between the processed $A$-update and the input, $U_A^t\underline{Z}$ has $\Theta(n)$ coordinate size in the low-rank dimension. Consequently, the magnitude of the internal feature shift scales as:
\begin{equation*}
    \|\Delta Z_A^t\|_2 \propto \eta_A \cdot n.
\end{equation*}
Under the \initAB regime where $\eta_A = \Theta(n^{-1/2})$, the internal update scales as $\Theta(\sqrt{n})$, leading to feature updates that explode with model width. In contrast, both the standard \initA and our proposed \initAA adopt the asymmetric scaling $\eta_A = \Theta(n^{-1})$. This choice ensures $\|\Delta Z_A^t\|_2 = \Theta(1)$, maintaining stable feature evolution within the interaction bottleneck. We postulate that this key difference, preserving internal stability, is essential for consistent performance across varying model scales, a hypothesis supported by the empirical analysis in \cref{fig:norms} and subsequent experiments. In \cref{app:comparison_with_AB,app:math_equal_lr}, experimental results demonstrate that \initAA outperforms \initAB in most configurations, even under the condition $\eta_A = \eta_B$. Moreover, \initAB exhibits notable performance degradation for large values of $\beta$.

The analysis above is deliberately stated through the signal-noise decomposition of the LoRA feature dynamics, because this decomposition isolates the initialization-dependent term that causes high-ratio instability. A complementary view is given in \cref{app:spectral_derivation}, where the spectral condition for feature learning \citep{yang2023feature} leads to the same \initAA scaling by controlling the worst-case spectral norm. \Cref{app:spectral_conditions} explains why $\gamma$-calculus can miss aligned directions in which the random initialization is amplified, and why the spectral condition again points to the \initAA variance scaling $\sigma_A^2 = \Theta(n^{-2})$ as the choice that makes the initialization noise asymptotically subdominant.

\section{Experiments}
We evaluate whether \initAA improves HP transfer and robustness in two settings: a controlled toy MLP and LLM fine-tuning. The LLM experiments cover both encoder-only models (RoBERTa-base on GLUE) and decoder-only models (Llama arithmetic and commonsense benchmarks, with additional backbones in the appendix). Across experiments, we vary the initialization scale $\beta$, LR ratio $\eta_B/\eta_A$, adapter rank $r$, and LoRA multiplier $\alpha/r$; \initA uses base variance $\sigma_A^2=1/n$ and \initAA uses $\sigma_A^2=1/n^2$, both with $B_0=0$. Full training details are in \cref{app:experimental_setup}.

\subsection{Toy Example}

\begin{wrapfigure}{r}{0.50\textwidth}
    \vspace{-1.2em}
    \centering
    \includegraphics[width=1\linewidth]{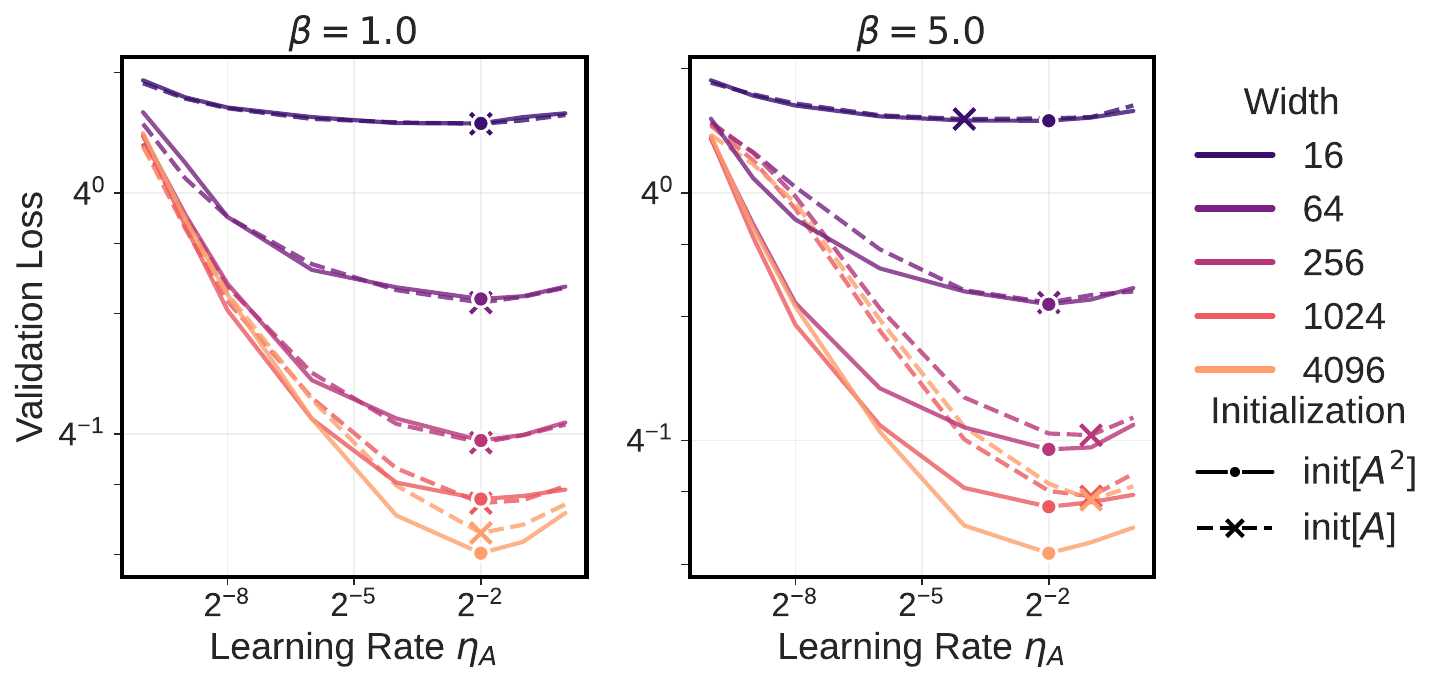}
    \caption{HP transfer on the MNIST toy example, measured by validation loss across widths and LRs.}
    \label{fig:lr_sweep}
\end{wrapfigure}
\paragraph{Setup.} To test HP transfer under \initAA, we consider a simplified MLP equipped with LoRA ($\alpha/r = 1$) on MNIST \citep{lecun1998mnist}:
\begin{equation*}
    Y = W_{\text{out}} \Fcal ((W_0 + B A), \Fcal(W_{\text{in}}, X)),
\end{equation*}
where $W_{\text{in}} \in \Rbb^{n \times d}$, $W_{\text{out}} \in \Rbb^{10 \times n}$, $W_0 \in \Rbb^{n \times n}$, $A \in \Rbb^{r \times n}$, and $B \in \Rbb^{n \times r}$. Here $\Fcal(W,z)=\operatorname{ReLU}(Wz)$. We fix $W_0=0$ and freeze $W_{\text{in}}$ and $W_{\text{out}}$ after the $\mu$P initializations $(W_{\text{out}})_{ij} \sim \Ncal(0,n^{-2})$ and $(W_{\text{in}})_{ij} \sim \Ncal(0,1)$ \citep{yang2022tensor}. To isolate the clean non-EMA normalized-update regime from our theory, we disable optimizer momentum and moving averages. We compare \initA and \initAA across $n \in \{16, 64, 256, 1024, 4096\}$ with a fixed LR ratio $\eta_B=n\eta_A$ and $\eta_A \in \{2^{-7}, \cdots, 2^0\}$. For each width, we report the best validation loss over the LR grid, averaged over three seeds, and track the squared norms of the internal feature $Z_A$ and LoRA feature $Z_B$ to assess stability.

\paragraph{Results.} \cref{fig:lr_sweep} illustrates the HP transfer capabilities. At a low initialization scale ($\beta = 1.0$), both methods exhibit consistent optimal LRs across widths. However, at a larger scale ($\beta = 5.0$), the optimal LR for \initA drifts significantly at the largest width ($n = 4096$), and convergence degrades. In contrast, \initAA maintains stable HP transfer and achieves consistently lower validation loss, with the performance gap widening as $\beta$ increases. This stability is further corroborated by the feature norm evolution (\cref{fig:norms}), where \initAA preserves the stability of $Z_A$ while the evolution of $Z_B$ is not affected by $\beta$.

\subsection{LLM Experiments}
\label{subsec:llama}
\begin{figure}[!t]
    \centering
    \includegraphics[width=1\linewidth]{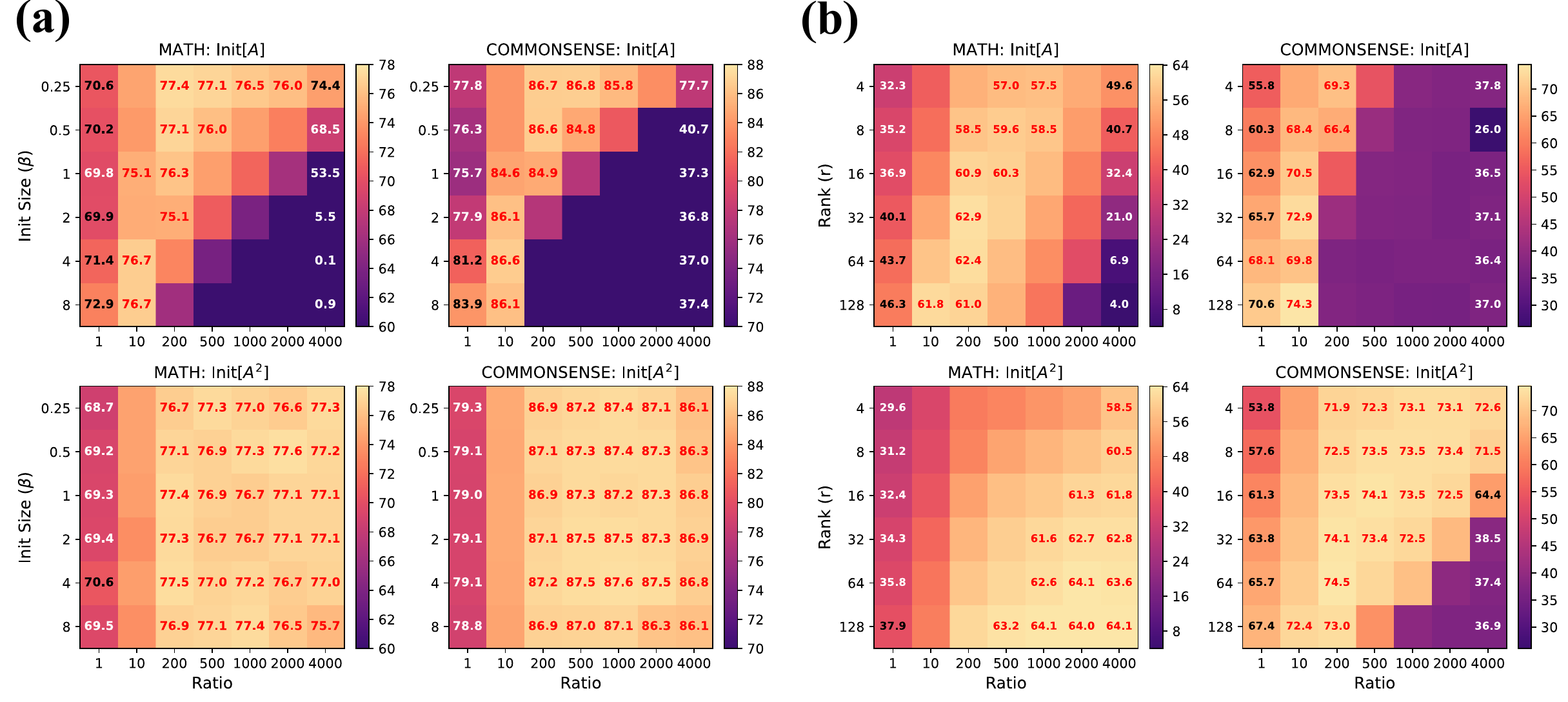}
    \vspace{-1.0em}
    \caption{Llama generation results. (a) Llama 3-8B with varying initialization scale $\beta$. (b) Llama 3.2-1B with varying rank $r$. Each panel reports arithmetic and commonsense accuracy under standard \initA and \initAA while sweeping the LR ratio $\rho=\eta_B/\eta_A$, implemented by fixing $\eta_A=10^{-5}$ and setting $\eta_B=\rho\eta_A$. Numerical labels mark endpoint ratios and near-optimal settings.}
    \label{fig:llama_2x2}
    \vspace{-1.0em}
\end{figure}

\begin{figure}[!t]
    \centering
    \includegraphics[width=1\linewidth]{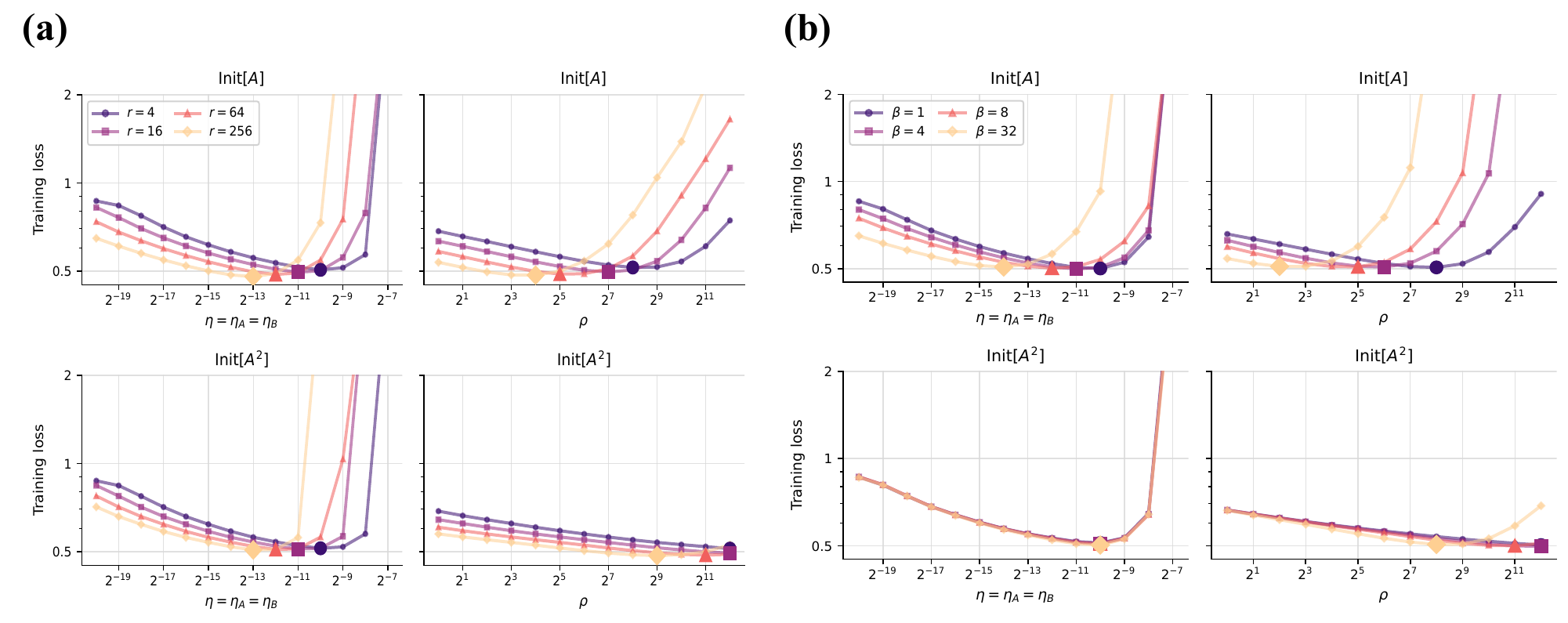}
    \vspace{-0.8em}
    \caption{Final training loss reached during one epoch on the math benchmark. The LoRA multiplier is fixed to $\alpha/r=1$. (a) varies rank $r$, and (b) varies initialization scale $\beta$. The figure includes equal-LR sweeps with $\eta=\eta_A=\eta_B$ and LR-ratio sweeps; in the ratio panels, $\eta_A=10^{-5}$ and $\eta_B=\rho\eta_A$ for $\rho=\eta_B/\eta_A$.}
    \label{fig:loss_plot}
    \vspace{-1.0em}
\end{figure}

\paragraph{Setup.} To validate the effectiveness of our theory for real LLMs, we evaluate both encoder-only and decoder-only LLMs. We use RoBERTa on GLUE for encoder-only results, deferred to \cref{app:lr_sweep_experiment}. For decoder-only generation, we fine-tune Llama 3-8B and Llama 3.2-1B \citep{dubey2024llama} on arithmetic and commonsense benchmarks \citep{hu2023llm}, following the evaluation by \citet{libeyond}: 10k arithmetic/math examples and 170k commonsense examples. For Llama 3-8B, we set $r=\alpha=8$ and sweep $\beta \in \{0.25,0.5,1,2,4,8\}$ and $\rho=\eta_B/\eta_A \in \{1,10,200,500,1000,2000,4000\}$, fixing $\eta_A=10^{-5}$ and $\eta_B=\rho\eta_A$. For Llama 3.2-1B, we fix $\beta=1$, sweep $r\in\{4,8,16,32,64,128\}$, and set $\alpha=r$. The loss study trains one epoch on math while varying $r$, $\beta$, equal LRs, and LR ratios.

\paragraph{Results.} \Cref{fig:llama_2x2} shows that \initAA is much less sensitive to $\beta$ and rank than standard \initA: it keeps broad high-accuracy regions across LR ratios, while \initA drifts toward smaller ratios and collapses at high ratios as $\beta$ or $r$ increases. The training-loss sweep in \cref{fig:loss_plot} is consistent with the component decomposition. Here, $\beta$ appears explicitly only in the initialization term; the signal can still depend on $\beta$ indirectly through the optimizer trajectory. Empirically, the equal-LR optimum and best loss remain nearly unchanged across $\beta$ under \initAA, consistent with suppression of initialization noise. However, because $r$ appears in both the signal and noise terms, the optimum LR also shifts for \initAA when sweeping $r$. Standard \initA does not suppress the noise, so its low-loss region narrows or shifts. In the ratio sweep, some drift can still appear because very large $\eta_B$ can introduce its own stability issue, but \initAA remains broader and more robust. The same qualitative pattern holds for encoder-only RoBERTa and for additional decoder-only Llama, Qwen, and Mistral backbones in \cref{app:lr_sweep_experiment,fig:more_models_lora}.


\subsection{\initAA Enables Multiplier Tuning}
\paragraph{Setup.}
In both the analysis and experiments, we set the multiplier $\alpha / r = \Theta(1)$. When this multiplier increases, it can also amplify initialization noise. This explains why many LoRA studies use relatively modest multipliers, often with $\alpha/r \le 2$ \citep{hulora,dettmers2023qlora,biderman2024lora}. However, a larger multiplier gives the learned low-rank update more influence on the model output and may therefore increase adaptability on hard decoder-only fine-tuning tasks. Our analysis suggests that the obstacle is not only stronger adaptation: $\alpha/r$ also amplifies the same initialization-dependent noise as $\beta$ (\cref{lemma:decomps_2}). If this random-initialization effect is suppressed, larger multipliers need not be intrinsically unstable and may further improve performance. We therefore fix $\beta=0.5$, $\eta_A=10^{-5}$, and $\eta_B=2\times10^{-3}$ on Llama~3-8B arithmetic fine-tuning, then vary $\alpha/r$.

\Needspace{16\baselineskip}
\paragraph{Results.}
\begin{wrapfigure}{r}{0.48\textwidth}
    \vspace{-1.1em}
    \centering
    \includegraphics[width=1\linewidth]{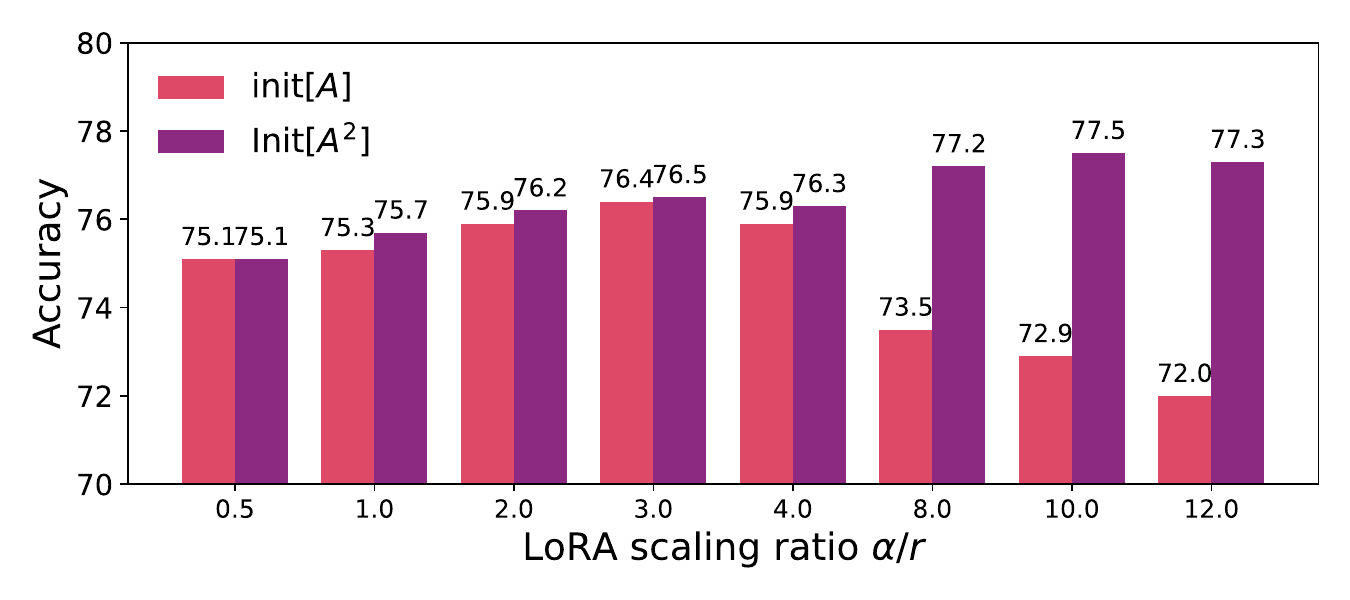}
    \vspace{-0.9em}
    \caption{Arithmetic accuracy as the LoRA multiplier $\alpha/r$ increases. \initAA remains stable at larger multipliers, while \initA degrades beyond $\alpha/r=3$.}
    \label{fig:alpha_plot}
    \vspace{-1.0em}
\end{wrapfigure}
\Cref{fig:alpha_plot} reveals a sharp stability gap. \initA peaks at $\alpha/r=3$ (76.4\%) and drops to 72.0\% at $\alpha/r=12$, whereas \initAA improves up to $\alpha/r=10$ (77.5\%) and remains stable at $\alpha/r=12$ (77.3\%). This matches our analysis: larger $\alpha/r$ amplifies both signal and noise, but \initAA suppresses the initialization-dependent noise enough to exploit stronger low-rank updates without destabilizing the update from the signal term. Thus, \initAA makes tuning the multiplier $\alpha/r$ a viable way to further improve arithmetic fine-tuning performance. Five-seed multiplier validation, matched LR retuning, and longer-horizon comparisons are in \cref{app:multiseed_validation,app:long_horizon}.

\section{Conclusion and Limitations}
We propose \initAA, a simple LoRA initialization that suppresses initialization-induced noise while preserving zero initial adapter output. The central mechanism is that standard \initA leaves a random initial projection in $A$ that is amplified as soon as $B$ moves; \initAA reduces this term by one additional factor of width, making high-ratio asymmetric LRs much less brittle. Practically, this reduces best-ratio drift across initialization scales, lowers dependence on exhaustive HP sweeps, and improves robustness to large update multipliers $(\alpha/r)$, enabling stronger updates and better performance on downstream tasks. These gains are consistent across the controlled HP-transfer experiment, encoder-only RoBERTa fine-tuning, and decoder-only LLM generation benchmarks.

More broadly, the design philosophy only targets the scale of $A_0$ rather than its distribution. It can therefore be applied to structured initializers: we compress the $A_0$ scale of Orthogonal and EVA while retaining their selected directions, and use a factorwise gauge transformation for PiSSA that preserves its initialized low-rank product. In selected Llama arithmetic sweeps, these three variance-compressed variants remain stable under aggressive asymmetric LRs for which their uncompressed counterparts degrade (\cref{app:variance_compressed_initializers}). This extension is empirical: the $B_0=0$ theorem does not directly cover PiSSA, and we do not claim uniform gains over carefully tuned base initializers.

Several limitations remain. Although some configurations are replicated across three to five seeds, most LR-ratio sweeps are single runs because of computational constraints, so small score differences in those sweeps should not be overinterpreted. Our empirical evaluation focuses on supervised fine-tuning of LLMs; extending the study to reinforcement-learning fine-tuning, broader generative-model training, and other adaptation settings is an important next step. Although \initAA makes asymmetric LRs more robust and our analysis identifies factors that can shift the preferred ratio, including $\beta$, $r$, and $\alpha/r$, we do not provide a direct rule for setting $\eta_B/\eta_A$ across arbitrary models and tasks. Developing such model-aware prescriptions, a fuller theory for rank $r$ and other LoRA HP interactions, and a better understanding of how far initialization scaling can be pushed beyond \initAA are useful future directions. The limits of initialization scaling may connect to recent observations on the role of initialization in emergent reasoning behavior \citep{zhang2024initialization}.

\begin{ack}
Disen Liao is supported by the Queen Elizabeth II Graduate Scholarship in Science and Technology (QEII-GSST). This work is supported in part by the Natural Sciences and Engineering Research Council of Canada (NSERC), the Canada CIFAR AI Chairs program and the Ontario Early Researcher program. Resources used in preparing this research were provided, in part, by the Province of Ontario, the Government of Canada through CIFAR, and companies
sponsoring the Vector Institute. We thank Mufan Li for very helpful discussions. The authors declare no competing interests. 
\end{ack}

\clearpage
\bibliography{ref-neurips}
\bibliographystyle{unsrtnat}

\appendix
\crefname{appendix}{Appendix}{Appendices}
\Crefname{appendix}{Appendix}{Appendices}
\crefalias{section}{appendix}
\crefalias{subsection}{appendix}
\crefalias{subsubsection}{appendix}
\setlength{\leftmargin}{2em}
\setlength{\leftmargini}{2em}

\clearpage
\section*{Appendix Table of Contents}
\begingroup
\hypersetup{linkcolor=blue}
\vspace{1.5em}
\noindent\hspace*{1em}\begin{minipage}{\dimexpr\linewidth-2em\relax}
\appendixtocsection{app:theory_proofs}{A. Theoretical Analysis and Proofs}
\appendixtocitem{app:two_scaling_laws}{A.1 Two Scaling Laws}
\appendixtocitem{app:gamma}{A.2 $\gamma$-Calculus}
\appendixtocitem{app:derivation_of_theorem}{A.3 Update Decomposition and Noise-Dominance Theorem}
\appendixtocitem{app:adam_extension}{A.4 Adam and AdamW as Normalized LoRA Updates}
\appendixtocitem{app:spectral_derivation}{A.5 Spectral Derivation of \initAA}
\appendixtocitem{app:spectral_conditions}{A.6 Spectral Conditions for Robust LoRA Feature Learning}
\appendixtocsection{app:related_work}{B. Related Literature Review}
\appendixtocitem{app:related_work_lora_init}{B.1 LoRA Initialization}
\appendixtocitem{app:related_work_lora_hps}{B.2 Other LoRA Hyperparameters}
\appendixtocitem{app:related_work_feature_learning}{B.3 Feature Learning and HP Transfer}
\appendixtocitem{app:related_work_stability}{B.4 Training Stability}
\appendixtocsection{app:experimental_setup}{C. Experimental Setup}
\appendixtocitem{app:init_implementation}{C.1 Initialization Implementation}
\appendixtocitem{app:toy_setup}{C.2 Toy Example}
\appendixtocitem{app:encoder_only_setup}{C.3 Encoder-Only Model Setup}
\appendixtocitem{app:decoder_only_setup}{C.4 Decoder-Only Model Setup}
\appendixtocsection{app:additional_results}{D. Additional Results and Ablations}
\appendixtocitem{app:lr_sweep_experiment}{D.1 LR Sweep Experiment}
\appendixtocitem{app:block_b_important}{D.2 Sensitivity to the $B$ LR}
\appendixtocitem{app:additional_llama}{D.3 Additional Llama Experiments}
\appendixtocitem{app:long_horizon}{D.4 Long-Horizon Arithmetic Fine-Tuning}
\appendixtocitem{app:multiseed_validation}{D.5 Multi-Seed Accuracy and Tuning Robustness}
\appendixtocitem{app:empirical_energy_decomposition}{D.6 Signed Signal--Noise Energy Measurements}
\appendixtocsection{app:comparison_with_others}{E. Comparisons with Other LoRA Initializations}
\appendixtocitem{app:comparison_with_AB}{E.1 Comparison with \initAB}
\appendixtocitem{app:math_equal_lr}{E.2 Equal-LR Sweep on Math Benchmark}
\appendixtocitem{app:variance_compressed_initializers}{E.3 Variance-Compressed Structured Initializers under Asymmetric LRs}
\appendixtocsection{app:responsible_research}{F. Responsible Research Statements}
\appendixtocitem{app:broader_impact}{F.1 Broader Impact}
\appendixtocitem{app:llm_usage}{F.2 LLM Usage}
\end{minipage}
\endgroup
\clearpage

\section{Theoretical Analysis and Proofs}
\label{app:theory_proofs}
\setcounter{theorem}{0}
\setcounter{lemma}{0}
\renewcommand{\theHtheorem}{appendix.theorem.\arabic{theorem}}
\renewcommand{\theHlemma}{appendix.lemma.\arabic{lemma}}
\renewcommand{\theHcondition}{appendix.condition.\arabic{condition}}

\subsection{Two Scaling Laws}
\label{app:two_scaling_laws}
When working with $\gamma$-calculus, we explicitly distinguish between Law of Large Numbers (LLN) and Central Limit Theorem (CLT) scaling. These two scaling regimes are fundamental in maximal-update-parameterization theory \citep{yang2021tensor}, which enables HP transfer \citep{yang2022tensor}. Here we give the two basic statistical rules and explain why they matter.

\paragraph{Law of Large Numbers (LLN).} If $X_1, X_2, \cdots, X_n$ are \iid random variables with finite expected value $\Ebb[X_i] = \mu$, then
\begin{equation*}
    \frac{1}{n} \sum_{i=1}^n X_i \xrightarrow{\text{a.s.}} \mu, \quad \text{as } n \rightarrow \infty.
\end{equation*}

\paragraph{Central Limit Theorem (CLT).} If $X_1, X_2, \cdots, X_n$ are \iid random variables with finite expected value $\Ebb[X_i] = \mu$ and finite variance $\operatorname{Var}(X_i) = \sigma^2$, then
\begin{equation*}
    \frac{1}{\sqrt{n}} \sum_{i=1}^n(X_i - \mu) \xrightarrow{d} \Ncal(0, \sigma^2), \quad \text{as } n \rightarrow \infty.
\end{equation*}
These two statistical laws reveal two scenarios regarding the size of a sum of $X_i$ when $n$ is large:
\begin{equation*}
    \sum_{i=1}^n X_i \approx \underbrace{n \mu}_{\text{LLN}} +  \underbrace{\sqrt{n} \sigma \Ncal(0,1)}_{\text{CLT}}.
\end{equation*}
For $\mu\ne0$, $|\sum_{i=1}^n X_i|/n\to|\mu|$ almost surely. For $\mu=0$ and $0<\sigma^2<\infty$, the sum has RMS magnitude $\sqrt{\Ebb[(\sum_{i=1}^n X_i)^2]}=\sigma\sqrt n$, but no deterministic lower bound follows.

\subsection{$\gamma$-Calculus}
\label{app:gamma}
In the scaling-limit theory of deep learning, we want to track the asymptotic behavior of key quantities such as LRs or initialization as we scale model ingredients such as width or depth. Prior work \citep{hayou2024impact,hayou2024lora+,libeyond} uses $\gamma$-calculus to derive LoRA scaling rules with respect to the model width $n$.

If a variable $v$ has asymptotically polynomial magnitude in $n$, \ie, $|v| = \Theta(n^\lambda)$, we write $\lambda = \gamma[v]$. Its exponent arithmetic follows these rules:

\begin{itemize}
    \item \textbf{Zero:} $\gamma[0] = - \infty$, corresponding to the limit $\gamma[n^\lambda]$ when $\lambda \rightarrow - \infty$.
    \item \textbf{Product:} For two real-valued variables $v, v' \in \Rbb$, we have $\gamma[vv'] = \gamma[v] + \gamma[v']$.
    \item \textbf{Addition:} Whenever the exponents exist, $\gamma[v+v']\le\max(\gamma[v],\gamma[v'])$. Equality holds if the exponents differ or $|v+v'|=\Theta(|v|+|v'|)$; excluding exact cancellation alone is insufficient.
\end{itemize}

\subsection{Update Decomposition and Noise-Dominance Theorem}
\label{app:derivation_of_theorem}

We first show that the key decomposition only requires width-normalized update directions, not the exact SignSGD form.

\begin{lemma}[Update and feature decompositions]
\label{lemma:decomps_2}
For any step $t \ge 1$, both the one-step LoRA feature update and the LoRA feature itself decompose into noise terms depending on the initialization scale $\beta$ and signal terms depending on the LRs:
\begin{align*}
    \Delta Z_B^t
    &= \underbrace{- \frac{\alpha\beta}{r}\,\eta_B\, U_B^t \tilde{A}_0 \underline{Z}}_{\Delta T_{\text{noise}}}
    + \underbrace{\frac{\alpha}{r}\,\eta_A \eta_B\, S_t\, \underline{Z}}_{\Delta T_{\text{signal}}^t},\\
    Z_B^t
    &= \underbrace{-\frac{\alpha \beta}{r} \eta_B \Big(\sum_{j=0}^{t-1} U_B^j \Big) \tilde{A}_0 \underline{Z}}_{T_{\text{noise}}^t}
    + \underbrace{\frac{\alpha}{r} \eta_A \eta_B \Big(\sum_{j=0}^{t-1} U_B^j \Big) \Big(\sum_{i=0}^{t-1} U_A^i \Big) \underline{Z}}_{T_{\text{signal}}^t},
\end{align*}
where
\[
S_t := \Big(\sum_{j=0}^{t-1} U_B^j\Big) U_A^t \;+\; U_B^t \Big(\sum_{i=0}^{t-1} U_A^i\Big) \;+\; U_B^t U_A^t.
\]
\end{lemma}

\begin{proof}
    Let
    \[
    \Delta A_t := A_{t+1} - A_t = -\eta_A U_A^t, \qquad
    \Delta B_t := B_{t+1} - B_t = -\eta_B U_B^t.
    \]
    Repeatedly applying the update recursion gives
    \[
    A_t = A_0 - \eta_A \sum_{i=0}^{t-1} U_A^i,
    \qquad
    B_t = B_0 - \eta_B \sum_{j=0}^{t-1} U_B^j
        = -\eta_B \sum_{j=0}^{t-1} U_B^j,
    \]
    where the last equality uses $B_0=0$.

    For the one-step update, expand
    \[
    B_{t+1} A_{t+1}
    = (B_t + \Delta B_t)(A_t + \Delta A_t)
    = B_t A_t + B_t \Delta A_t + \Delta B_t A_t + \Delta B_t \Delta A_t.
    \]
    Therefore,
    \begin{equation}
    \label{appeq:three_term_decomp}
    \frac{r}{\alpha}(Z_B^{t+1} - Z_B^t)
    = \underbrace{B_t \Delta Z_A^t}_{\delta_1^t}
    + \underbrace{\Delta B_t Z_A^t}_{\delta_2^t}
    + \underbrace{\Delta B_t \Delta Z_A^t}_{\delta_3^t},
    \end{equation}
    since $Z_A^t = A_t \underline{Z}$ and $\Delta Z_A^t = \Delta A_t \underline{Z}$. Substitution into the three terms gives
    \begin{align}
        B_t \Delta A_t \underline{Z}
        &= \Big( -\eta_B \sum_{j=0}^{t-1} U_B^j \Big) \Big( -\eta_A U_A^t \Big) \underline{Z}
        = \eta_A \eta_B \Big( \sum_{j=0}^{t-1} U_B^j \Big) U_A^t \underline{Z}, \label{eq:lemma3_proved}\\
        \Delta B_t A_t \underline{Z}
        &= \Big( -\eta_B U_B^t \Big) \Big( A_0 - \eta_A \sum_{i=0}^{t-1} U_A^i \Big) \underline{Z}
        = -\eta_B U_B^t A_0 \underline{Z} + \eta_A \eta_B\, U_B^t \Big( \sum_{i=0}^{t-1} U_A^i \Big) \underline{Z}, \label{eq:lemma3_proved2}\\
        \Delta B_t \Delta A_t \underline{Z}
        &= \Big( -\eta_B U_B^t \Big) \Big( -\eta_A U_A^t \Big) \underline{Z}
        = \eta_A \eta_B\, U_B^t U_A^t \underline{Z}.
    \end{align}
    Summing these identities in \cref{appeq:three_term_decomp}, defining $S_t$ as above, and substituting $A_0=\beta\tilde A_0$ yields the stated decomposition of $\Delta Z_B^t$.

    For the feature itself, use the same closed forms:
    \[
    Z_B^t
    = \frac{\alpha}{r} \left( -\eta_B \sum_{j=0}^{t-1} U_B^j \right)
      \left( \beta \tilde{A}_0 - \eta_A \sum_{i=0}^{t-1} U_A^i \right) \underline{Z}.
    \]
    Distributing the two terms and rearranging scalars gives the stated decomposition of $Z_B^t$.
\end{proof}

\begin{assumption}[Width-normalized update directions]
\label{ass:normalized_updates}
For all training steps $t \le T(n)$, the optimizer directions satisfy
\[
    \|U_A^t\|_{\max} = O(1), \qquad \|U_B^t\|_{\max} = O(1).
\]
On the analyzed steps after the initial $B$ update, the down-projection update is non-degenerate: for each LoRA row $i\in[r]$, $|(U_A^t)_{i,:}\underline Z|=\Theta(n)$, with constants independent of width. This magnitude condition alone does not rule out cancellation between steps.
\end{assumption}
Under \cref{ass:normalized_updates}, the component-scaling argument applies to any optimizer in this normalized class. For SignSGD, the coordinate bound is immediate and the fixed-input LoRA $A$-gradient has the sign-aligned outer-product structure used by \citet{hayou2024lora+}. \Cref{app:adam_extension} states the additional local conditions under which Adam and AdamW verify the same two optimizer-facing requirements.

\paragraph{Separate rank and temporal assumptions.}
Write $\lambda_r:=\alpha/r$ and $\rho_d(x):=\|x\|_2/\sqrt d$ for the RMS magnitude of a $d$-vector. For the accumulated directions, define
\[
C_t:=\sum_{s=0}^{t-1}U_B^s,\qquad
u:=\tilde A_0\underline Z,\qquad
v_t:=\Big(\sum_{s=0}^{t-1}U_A^s\Big)\underline Z.
\]
The two vectors entering $C_t$ depend differently on initialization and training. We therefore assign distinct rank exponents $\gamma_N,\gamma_S$ and effective multipliers $m_N:=\lambda_r r^{\gamma_N}$, $m_S:=\lambda_r r^{\gamma_S}$. The values $1$ and $1/2$ describe coherent, LLN-like and CLT-like rank scaling, respectively; these are scaling benchmarks, not assertions of a probabilistic limit law. In particular, mean-zero rank terms need not have CLT scaling when they are dependent.

The theorem below assumes linear temporal accumulation separately from the rank exponents. Neither coordinate boundedness nor per-step $A$--input alignment implies these accumulated lower bounds: signs may cancel across steps. The assumptions are imposed on the realized contractions of $u$ and $v_t$, allowing optimizer directions to depend on $A_0$.

\paragraph{Empirical rank coherence.}
The LLN choice is a conditional benchmark, rather than a universal property of LoRA training. To diagnose rank summation, let $u_t=A_t\underline Z\in\Rbb^r$ be the feature on a fixed probe token and $a_t=\|u_t\|_2/\sqrt r$ its RMS magnitude. For a sampled output row $k$, define the normalized contraction magnitude $S_{t,k}(r):=|(U_B^t u_t)_k|/a_t$, with $U_B^t$ taken before the LR and weight decay. After aggregating sampled contractions, fitting $S_t(r)\propto r^{\gamma_t}$ distinguishes CLT-like ($\gamma_t=1/2$) from coherent, LLN-like ($\gamma_t=1$) behavior. The sweep uses $r\in\{4,8,16,32,64,128\}$, $\eta_B/\eta_A\in\{200,1000,2000\}$, both initializations, and three seeds per configuration. We fit the exponent using both median and mean aggregation of sampled contractions.

\begin{figure}[t]
    \centering
    \includegraphics[width=\linewidth]{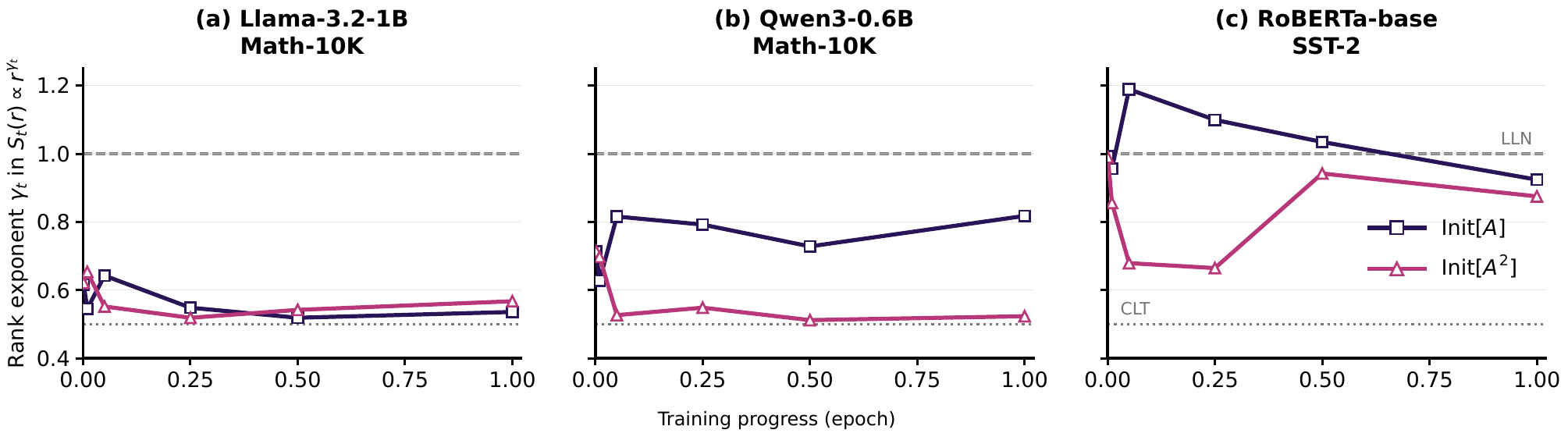}
    \caption{Empirical rank-coherence exponent during Llama-3.2-1B and Qwen3-0.6B fine-tuning on Math-10K, and RoBERTa-base fine-tuning on SST-2. The fitted exponent in $S_t(r)\propto r^{\gamma_t}$ varies with model, initialization, and training stage. Lines mark CLT-like ($1/2$) and LLN-like ($1$) rank scaling.}
    \label{fig:rank_scaling_coherence}
\end{figure}

\begin{table}[htbp]
\centering
\small
\setlength{\tabcolsep}{6pt}
\caption{Rank-coherence exponents under median / mean aggregation of the normalized contraction. Both aggregation choices show that the exponent depends on model and training stage.}
\label{tab:rank_coherence_exponents}
\begin{tabular}{llcc}
\toprule
Model & Checkpoint & \makecell{\initA\\median / mean $\gamma$} & \makecell{\initAA\\median / mean $\gamma$} \\
\midrule
Llama-3.2-1B & First update & 0.616 / 0.652 & 0.621 / 0.655 \\
Llama-3.2-1B & 5\% epoch    & 0.642 / 0.665 & 0.552 / 0.565 \\
Llama-3.2-1B & 50\% epoch   & 0.519 / 0.559 & 0.542 / 0.560 \\
Llama-3.2-1B & End epoch    & 0.536 / 0.572 & 0.567 / 0.581 \\
\midrule
Qwen3-0.6B   & First update & 0.713 / 0.778 & 0.714 / 0.779 \\
Qwen3-0.6B   & End epoch    & 0.818 / 0.819 & 0.523 / 0.533 \\
\midrule
RoBERTa/SST-2 & First update & 0.993 / 0.954 & 0.990 / 0.951 \\
RoBERTa/SST-2 & End epoch    & 0.924 / 0.958 & 0.875 / 0.877 \\
\bottomrule
\end{tabular}
\end{table}

\Cref{fig:rank_scaling_coherence} shows that Llama is largely CLT-like, Qwen depends on the initialization and training stage, and RoBERTa is closer to coherent summation. This motivates exponents appropriate to the model, module, and training window. The probe measures $U_B^tA_t\underline Z$, whereas the theorem concerns $C_tu$ and $C_tv_t$; it therefore does not establish either accumulated exponent or their equality. The theorem allows $\gamma_N\ne\gamma_S$. A shared multiplier $m_{r,\gamma}:=\lambda_r r^\gamma$ is recovered only if both accumulated contractions satisfy the same exponent. At fixed rank and $\lambda_r=\Theta(1)$, both multipliers are $\Theta(1)$, preserving the extra $1/n$ noise comparison under matched assumptions.

The optimizer directions are realized quantities and may depend on $A_0$ through the training trajectory. We therefore do not condition on these directions and then treat $\tilde A_0$ as independent. Instead, the following deterministic identity controls the cross term for every realized trajectory.

\begin{lemma}[Exact realized energy decomposition]
\label{lem:realized_energy_decomposition}
Define
\[
E_N^t:=\frac{1}{n}\|T_{\text{noise}}^t\|_2^2,\qquad
E_S^t:=\frac{1}{n}\|T_{\text{signal}}^t\|_2^2,\qquad
E_X^t:=\frac{2}{n}\langle T_{\text{noise}}^t,T_{\text{signal}}^t\rangle,
\]
and let $E_{\text{total}}^t:=n^{-1}\|Z_B^t\|_2^2$. Then
\[
E_{\text{total}}^t=E_N^t+E_S^t+E_X^t,
\qquad
|E_X^t|\le 2\sqrt{E_N^tE_S^t},
\]
and
\[
(\sqrt{E_N^t}-\sqrt{E_S^t})^2
\le E_{\text{total}}^t
\le(\sqrt{E_N^t}+\sqrt{E_S^t})^2.
\]
Consequently, $E_{\text{total}}^t/E_N^t\to1$ whenever $E_S^t/E_N^t\to0$, and $E_{\text{total}}^t/E_S^t\to1$ whenever $E_N^t/E_S^t\to0$.
\end{lemma}

\begin{proof}
Expanding $n^{-1}\|T_{\text{noise}}^t+T_{\text{signal}}^t\|_2^2$ gives the identity. Cauchy--Schwarz gives $|E_X^t|\le2\sqrt{E_N^tE_S^t}$, and substituting its two extremal signs gives the displayed bounds. For example, if $q_t:=\sqrt{E_S^t/E_N^t}\to0$, division by $E_N^t$ yields $(1-q_t)^2\le E_{\text{total}}^t/E_N^t\le(1+q_t)^2$; the other case is symmetric.
\end{proof}

We now give the full conditional statement and proof of \cref{theorem:lora_feature_beta}.
\begin{theorem}[Component scaling with general rank contractions]
\label{thm:general_rank_components}
    Fix $\underline Z\in\Rbb^n$ with $\|\underline Z\|_2^2=\Theta(n)$ and positive $\alpha,\beta,\eta_A,\eta_B$. Let $B_0=0$, $A_0=\beta\tilde A_0$, with independent $(\tilde A_0)_{ij}\sim\Ncal(0,\sigma_A^2)$, and updates $A_{s+1}=A_s-\eta_AU_A^s$, $B_{s+1}=B_s-\eta_BU_B^s$. Define $C_t,u,v_t,\rho_d$ as above. On a set of analyzed steps $t\ge2$, suppose the realized trajectories satisfy
    \begin{align}
    \rho_r(u)&=\Theta(\sigma_A\sqrt n), &
    \rho_r(v_t)&=\Theta(nt), \label{eq:input_accumulation_scales}\\
    \rho_n(C_tu)&=\Theta(t r^{\gamma_N}\rho_r(u)), &
    \rho_n(C_tv_t)&=\Theta(t r^{\gamma_S}\rho_r(v_t)), \label{eq:rank_contraction_scales}
    \end{align}
    for fixed exponents $\gamma_N,\gamma_S$, with constants uniform in $n,r$ and the analyzed steps. Set $m_N=(\alpha/r)r^{\gamma_N}$ and $m_S=(\alpha/r)r^{\gamma_S}$. Then
    \[
    E_N^t=\Theta(\eta_B^2\beta^2m_N^2\sigma_A^2nt^2),\qquad
    E_S^t=\Theta((\eta_A\eta_Bm_Snt^2)^2).
    \]
    The total energy satisfies
    \[
    (\sqrt{E_N^t}-\sqrt{E_S^t})^2\le E_{\mathrm{total}}^t\le(\sqrt{E_N^t}+\sqrt{E_S^t})^2
    \]
    for every realized trajectory, so its ratio to the larger component tends to one whenever the smaller-to-larger energy ratio tends to zero. No independence between initialization and optimizer directions is required.
\end{theorem}

\begin{proof}
    The unrolled updates give $B_t=-\eta_BC_t$ and $A_t\underline Z=\beta u-\eta_Av_t$. Thus the feature $Z_B^t=(\alpha/r)B_tA_t\underline Z$ has the exact components
    \[
    T_{\mathrm{noise}}^t=-\lambda_r\beta\eta_BC_tu,\qquad
    T_{\mathrm{signal}}^t=\lambda_r\eta_A\eta_BC_tv_t.
    \]
    Using \cref{eq:input_accumulation_scales,eq:rank_contraction_scales} yields
    \begin{align*}
    \sqrt{E_N^t}
    &=\lambda_r\beta\eta_B\rho_n(C_tu)
      =\Theta(\eta_B\beta m_N\sigma_A\sqrt n\,t),\\
    \sqrt{E_S^t}
    &=\lambda_r\eta_A\eta_B\rho_n(C_tv_t)
      =\Theta(\eta_A\eta_Bm_Snt^2).
    \end{align*}
    Squaring gives the RMS component estimates without requiring equal output-coordinate magnitudes. The deterministic bounds and dominance conclusions follow from \cref{lem:realized_energy_decomposition}.
\end{proof}

\paragraph{Scope of the assumptions.}
For fixed $\underline Z$, Gaussian initialization gives independent $u_i\sim\Ncal(0,\sigma_A^2\|\underline Z\|_2^2)$ and $\Ebb[\rho_r(u)^2]=\sigma_A^2\|\underline Z\|_2^2$. This motivates \cref{eq:input_accumulation_scales} but supplies no realized lower bound. The assumption $\rho_r(v_t)=\Theta(nt)$ requires temporal non-cancellation beyond per-step alignment, while the factor $t$ in \cref{eq:rank_contraction_scales} separately requires $B$-accumulation. We analyze only steps satisfying these bounds. With $B_0=0$ and zero optimizer state, $U_A^0=0$ and $v_1=0$, so the signal vanishes at the first step; the exact decomposition and deterministic bounds still hold.

\paragraph{Rank and width consequences.}
The conditional energy ratio is
\[
\frac{E_N^t}{E_S^t}
=\Theta\!\left(\frac{\beta^2\sigma_A^2}{\eta_A^2nt^2}\,r^{2(\gamma_N-\gamma_S)}\right).
\]
The rank factors cancel only if $\gamma_N=\gamma_S$. Taking both exponents equal to $1$ recovers the coherent benchmark; taking both equal to $1/2$ gives CLT-like rank scaling under the same temporal assumptions. For fixed $r$ and $\lambda_r=\Theta(1)$, both $m_N,m_S=\Theta(1)$, giving
\[
E_N^t=\Theta(\eta_B^2\beta^2\sigma_A^2nt^2),\qquad
E_S^t=\Theta((\eta_A\eta_Bnt^2)^2).
\]
For $\sigma_A^2=\Theta(n^{-p})$, the noise estimate becomes $\Theta(\eta_B^2\beta^2m_N^2n^{1-p}t^2)$. Thus changing $p=1$ to $p=2$ supplies the extra $1/n$ factor under matched contraction scales, without requiring LLN-like rank summation or equality of the noise and signal exponents. This comparison preserves the signal's scaling order under matched assumptions; it does not equate the optimizer trajectories across initializations.

We next record the two terms needed for the finite-time efficiency proof. Recall the decomposition:
\begin{equation*}
    \frac{r}{\alpha}\Delta Z_B^t = \underbrace{B_t\Delta Z_A^t}_{\delta_1^t} + \underbrace{\Delta B_t Z_A^t}_{\delta_2^t} + \Delta B_t\Delta Z_A^t.
\end{equation*}
Since $B_0=0$, the first-step contribution $B_0\Delta Z_A^0$ vanishes. For standard gradient-based updates with zero optimizer state, $U_A^0=0$ as well. For later steps $t\ge1$, the two terms have the following forms:
\begin{lemma}
\label{lemma:delta_expression}
    For $t \ge 1$, if the LoRA weights are updated by
    \[
        A_{t+1} = A_t - \eta_A U_A^t, \qquad B_{t+1} = B_t - \eta_B U_B^t,
    \]
    then
    \begin{equation*}
    \begin{cases}
        \delta_1^t = B_{t}\Delta Z_A^t = \eta_A \eta_B \left( \sum_{j=0}^{t-1} U_B^j \right) U_A^{t} \underline{Z}, \\
        \delta_2^t = \Delta B_t Z_A^t
        = -\eta_B U_B^{t} A_0 \underline{Z} + \eta_A \eta_B\, U_B^{t}\left(\sum_{i=0}^{t-1} U_A^i \right) \underline{Z}.
    \end{cases}
    \end{equation*}
\end{lemma}

Only $\delta_2^t$ explicitly contains $A_0$; both can depend on it through optimizer directions.

\begin{proof}
This follows directly from the identities established in \cref{eq:lemma3_proved,eq:lemma3_proved2}.
\end{proof}

Now, we prove \cref{theorem:initAA_feature_learning}. RMS bounds alone do not give coordinatewise lower bounds or control a one-step contraction, so these are stated explicitly.
\begin{theorem}[Finite-time stability and efficiency of \initAA]
    Fix $r=\Theta(1)$ and a finite step $t\ge2$. Let
    \[
    \lambda=\alpha/r=\Theta(1),\qquad
    \beta=\Theta(1),\qquad
    \eta_A=\Theta(n^{-1}),\qquad
    \eta_B=\Theta(1),
    \]
    and define
    \[
    C_t=\sum_{s<t}U_B^s,\qquad
    u=\tilde A_0\underline Z,\qquad
    v_t=\Big(\sum_{s<t}U_A^s\Big)\underline Z.
    \]
    Assume coordinate-bounded optimizer directions and, uniformly in output coordinate $k$,
    \[
    |(C_tv_t)_k|=\Theta(n),\qquad
    |(C_tU_A^t\underline Z)_k|=\Theta(n),\qquad
    |(U_B^tv_t)_k|=\Theta(n).
    \]
    Assume also the initialization projection upper bound from \cref{thm:general_rank_components},
    \[
    \|u\|_2=O(\sigma_A\sqrt n),
    \]
    at fixed rank. Then under \initAA, for which $\sigma_A^2=\Theta(n^{-2})$, every coordinate of $Z_B^t$, $\delta_1^t$, and $\delta_2^t$ has magnitude $\Theta(1)$, and $E_{\mathrm{total}}^t=\Theta(1)$. Hence \initAA satisfies both stability and efficiency in \cref{cond:hayou_criteria} at the analyzed step.

    Under standard \initA, for which $\sigma_A^2=\Theta(n^{-1})$, the corresponding quantities satisfy
    \[
    |(Z_B^t)_k|=O(1),\qquad
    |(\delta_1^t)_k|=\Theta(1),\qquad
    |(\delta_2^t)_k|=O(1),
    \]
    and $E_{\mathrm{total}}^t=O(1)$. Lower bounds of $\Theta(1)$ for $Z_B^t$ and $\delta_2^t$ require non-cancellation.
\end{theorem}

\begin{proof}
From \cref{lemma:decomps_2,lemma:delta_expression},
\[
Z_B^t=-\lambda\beta\eta_B C_tu+\lambda\eta_A\eta_B C_tv_t,
\]
and
\[
\delta_1^t=\eta_A\eta_B C_tU_A^t\underline Z,\qquad
\delta_2^t=-\beta\eta_B U_B^tu+\eta_A\eta_B U_B^tv_t.
\]
Thus $Z_B^t$ and $\delta_2^t$ each consist of an initialization-dependent term plus a learned signal term, whereas $\delta_1^t$ contains only the learned signal.

\paragraph{Initialization-dependent terms.}
Because $r$ and $t$ are fixed and the optimizer directions are coordinate-bounded, every row of $C_t$ and $U_B^t$ has $O(1)$ Euclidean norm. Therefore
\[
|(C_tu)_k|,\ |(U_B^tu)_k|=O(\|u\|_2)=O(\sigma_A\sqrt n).
\]
For standard \initA, $\sigma_A=\Theta(n^{-1/2})$, so these terms are $O(1)$. For \initAA, $\sigma_A=\Theta(n^{-1})$, so they are $O(n^{-1/2})=o(1)$. Multiplication by the $\Theta(1)$ factors $\lambda$, $\beta$, and $\eta_B$ preserves these bounds.

\paragraph{Learned signal terms.}
By assumption,
\[
|(C_tv_t)_k|,\ |(C_tU_A^t\underline Z)_k|,\ |(U_B^tv_t)_k|=\Theta(n).
\]
Since $\eta_A\eta_B=\Theta(n^{-1})$ and $\lambda=\Theta(1)$, each corresponding signal contribution has magnitude $\Theta(1)$.

\paragraph{Combining the two components.}
Under \initAA, each coordinate of $Z_B^t$ and $\delta_2^t$ is therefore a signal of magnitude $\Theta(1)$ plus an $o(1)$ initialization term, and hence retains magnitude $\Theta(1)$. The same is already true for $\delta_1^t$. Thus the stability and efficiency conditions hold coordinatewise at the analyzed step. Averaging the squared coordinates gives
\[
E_{\mathrm{total}}^t=\Theta(1).
\]
The component bounds also give $E_N^t=O(n^{-1})$ and $E_S^t=\Theta(1)$, so \cref{lem:realized_energy_decomposition} yields $E_{\mathrm{total}}^t/E_S^t\to1$. Under the additional projection and noise-contraction lower bounds of \cref{thm:general_rank_components}, the noise estimate sharpens to $E_N^t=\Theta(n^{-1})$.

Under standard \initA, both the signal and initialization-dependent contributions can be $O(1)$. Their sum is therefore $O(1)$, which proves non-explosion, but the two terms may cancel. Consequently, a $\Theta(1)$ lower bound for $Z_B^t$ or $\delta_2^t$ requires an additional non-cancellation condition. The term $\delta_1^t$ contains no explicit initialization component and is already $\Theta(1)$ by assumption.
\end{proof}

\subsection{Adam and AdamW as Normalized LoRA Updates}
\label{app:adam_extension}

The signal--noise decomposition in \cref{lemma:decomps_2} is algebraic once the optimizer supplies update directions $U_A^t$ and $U_B^t$. The role of this section is narrower: we explain when common normalized optimizers produce the two optimizer-facing properties used in \cref{ass:normalized_updates},
\[
\|U_A^t\|_{\max},\|U_B^t\|_{\max}=O(1),
\qquad
|(U_A^t)_{i,:}\underline Z|=\Theta(n)
\quad \text{for each } i\in[r].
\]
The first property is coordinate boundedness. The second is the gradient--input alignment condition for the $A$ branch after its initial zero update. This section does not prove the temporal non-cancellation or rank-contraction assumptions in \cref{eq:input_accumulation_scales,eq:rank_contraction_scales}; these remain separate conditions in \cref{theorem:lora_feature_beta}.

Throughout this section, write
\[
\lambda:=\frac{\alpha}{r}
\]
for the LoRA update multiplier. The LoRA layer output is
\[
\bar Z=W\underline Z+\lambda BA\underline Z.
\]
We consider a local analysis of a single example or token, in which the input $\underline Z\in\Rbb^n$ to the LoRA layer is fixed over the optimizer window. This is the same local viewpoint used in prior LoRA feature-learning analyses \citep{hayou2024lora+}. The full-network setting is more complicated because $\underline Z$ evolves with depth and training time; our alignment claim is conditional on this fixed single-example input. Minibatch training is assessed empirically in \cref{app:empirical_width_alignment}.

\paragraph{Fixed-input chain-rule structure.}
Let $d\bar Z_t:=\partial \Lcal_t/\partial \bar Z_t$. The gradients of the LoRA matrices are
\[
\nabla_B\Lcal_t=\lambda\, d\bar Z_t (Z_A^t)^\top,
\qquad
\nabla_A\Lcal_t=\lambda\, B_t^\top d\bar Z_t\,\underline Z^\top.
\]
Equivalently, for each row $i\in[r]$,
\[
(\nabla_A\Lcal_t)_{ij}=s_i^t\underline Z_j,
\qquad
s_i^t:=\lambda(B_t^\top d\bar Z_t)_i .
\]
Thus each row of the $A$-gradient is a scalar backpropagated signal times the same input vector $\underline Z$. This outer-product structure is the source of the $\Theta(n)$ alignment. It is not a generic property of arbitrary update directions.

Since standard LoRA initializes $B_0=0$, we have $s_i^0=0$, so the first $A$-update is zero. The non-degeneracy statements below apply after $B$ has received a nonzero update, or more generally at any step for which the scalar signal history is nonzero.

\paragraph{Input small-ball condition.}
We assume the fixed input has coordinate size $O(1)$ and non-vanishing mass on order-one coordinates. Concretely, there exist constants $p,z_-,z_+>0$, independent of width $n$, such that
\[
|\{j:z_-\le |\underline Z_j|\le z_+\}|\ge pn,
\qquad
\max_j|\underline Z_j|\le z_+ .
\]
This implies
\[
\|\underline Z\|_1=\Theta(n),
\qquad
\|\underline Z\|_2^2=\Theta(n).
\]
This allows some coordinates to vanish, provided a fixed fraction remain order one.

\subsubsection{SignSGD Alignment}

We first record the simplest case, since it explains why there is no coordinate-wise cancellation in the fixed-input SignSGD model.

\begin{lemma}[SignSGD gives exact input alignment]
\label{lem:signsgd_alignment}
Suppose the row gradient has the fixed-input form
\[
(\nabla_A\Lcal_t)_{ij}=s_i^t\underline Z_j.
\]
Let the optimizer direction be
\[
(U_A^t)_{ij}=\operatorname{sign}((\nabla_A\Lcal_t)_{ij}).
\]
If $s_i^t\neq0$ and $\underline Z$ satisfies the input small-ball condition, then
\[
(U_A^t)_{i,:}\underline Z
=\operatorname{sign}(s_i^t)\sum_{j=1}^n|\underline Z_j|
=\Theta(n).
\]
\end{lemma}

\begin{proof}
By the row-wise outer-product structure,
\[
(U_A^t)_{ij}
=\operatorname{sign}(s_i^t\underline Z_j)
=\operatorname{sign}(s_i^t)\operatorname{sign}(\underline Z_j).
\]
Therefore
\[
(U_A^t)_{i,:}\underline Z
=\sum_{j=1}^n\operatorname{sign}(s_i^t)\operatorname{sign}(\underline Z_j)\underline Z_j
=\operatorname{sign}(s_i^t)\sum_{j=1}^n|\underline Z_j|.
\]
The small-ball condition gives $\sum_{j=1}^n|\underline Z_j|\ge pnz_-$, and the coordinate upper bound gives $\sum_{j=1}^n|\underline Z_j|\le nz_+$. Hence the sum is $\Theta(n)$.
\end{proof}

\subsubsection{Adam Coordinate Boundedness}

At the zero-based LoRA step $t\ge0$, Adam forms moments from $g_0,\ldots,g_t$, with $m_{-1}=v_{-1}=0$:
\[
m_t=(1-\beta_1)\sum_{k=0}^{t}\beta_1^k g_{t-k},\qquad
v_t=(1-\beta_2)\sum_{k=0}^{t}\beta_2^k g_{t-k}^2,
\]
with bias-corrected moments
\[
\widehat m_t=\frac{m_t}{1-\beta_1^{t+1}},\qquad
\widehat v_t=\frac{v_t}{1-\beta_2^{t+1}},
\]
and coordinate update
\[
U_t=\frac{\widehat m_t}{\sqrt{\widehat v_t}+\epsilon}.
\]
Adam hyperparameters $(\beta_1,\beta_2,\epsilon)$ are treated as width-independent constants, with $\epsilon>0$. In the zero-EMA case $\beta_1=\beta_2=0$, the update is $g_t/(|g_t|+\epsilon)$, which approaches $\operatorname{sign}(g_t)$ as $\epsilon\to0$ for $g_t\ne0$ \citep{balles2018dissecting,bernstein2024old}. For the following bound, assume $0\le\beta_1<1$ and $\beta_1^2<\beta_2<1$.

\begin{proposition}[Width-uniform Adam coordinate bound]
\label{prop:adam_coordinate_bound}
\label{prop:adam_coord_bound}
For any scalar gradient history $\{g_0,\ldots,g_t\}$, if $0\le\beta_1<1$ and $\beta_1^2<\beta_2<1$, then
\[
\frac{|\widehat m_t|}{\sqrt{\widehat v_t}+\epsilon}
\le C_{\beta_1,\beta_2}
:=\frac{1}{\sqrt{(1-\beta_2)(1-\beta_1^2/\beta_2)}}
\]
for every $t\ge0$ and $\epsilon>0$. Thus Adam coordinate updates are $O(1)$ uniformly in width.
\end{proposition}

\begin{proof}
If $v_t=0$, then all gradients in the EMA window are zero, $m_t=0$, and the update is zero for $\epsilon>0$. Otherwise, Cauchy--Schwarz gives
\begin{align*}
m_t^2
&=(1-\beta_1)^2
\left(\sum_{k=0}^{t}\beta_1^k g_{t-k}\right)^2 \\
&\le
(1-\beta_1)^2
\left(\sum_{k=0}^{t}\left(\frac{\beta_1^2}{\beta_2}\right)^k\right)
\left(\sum_{k=0}^{t}\beta_2^k g_{t-k}^2\right)\\
&=
\frac{(1-\beta_1)^2}{1-\beta_2}
\left(\sum_{k=0}^{t}\left(\frac{\beta_1^2}{\beta_2}\right)^k\right)v_t.
\end{align*}
Hence
\[
\frac{|m_t|}{\sqrt{v_t}}
\le
\frac{1-\beta_1}{\sqrt{1-\beta_2}}
\left(\sum_{k=0}^{t}\left(\frac{\beta_1^2}{\beta_2}\right)^k\right)^{1/2}
\le
\frac{1-\beta_1}{\sqrt{1-\beta_2}}\frac{1}{\sqrt{1-\beta_1^2/\beta_2}},
\]
where the last inequality uses $\beta_1^2/\beta_2<1$. The bias-correction multiplier satisfies
\[
\frac{\sqrt{1-\beta_2^{t+1}}}{1-\beta_1^{t+1}}\le\frac{1}{1-\beta_1}
\qquad(t\ge0).
\]
Combining these bounds with $\sqrt{\widehat v_t}+\epsilon\ge\sqrt{\widehat v_t}$ gives a constant independent of width and time.
\end{proof}

\subsubsection{Adam as Soft-Sign Gradient--Input Alignment}

The coordinate bound above shows that Adam directions are $O(1)$, but \cref{ass:normalized_updates} also requires $|(U_A^t)_{i,:}\underline Z|=\Theta(n)$. We establish this under scalar EMA non-cancellation and a bound on the relative size of $\epsilon$.

For a fixed row $i$, suppose the $A$-gradient history has the form
\[
(\nabla_A\Lcal_{t-k})_{ij}=s_i^{t-k}\underline Z_j,
\qquad k=0,\ldots,t.
\]
Define the bias-corrected scalar EMA moments
\[
\bar s_i^t:=
\frac{\sum_{k=0}^{t}\beta_1^k s_i^{t-k}}
{\sum_{k=0}^{t}\beta_1^k},
\qquad
\bar q_i^t:=
\frac{\sum_{k=0}^{t}\beta_2^k (s_i^{t-k})^2}
{\sum_{k=0}^{t}\beta_2^k}.
\]
For $\bar q_i^t>0$, define the Adam signal-to-RMS ratio and epsilon ratio
\[
\kappa_i^t:=
\frac{\bar s_i^t}{\sqrt{\bar q_i^t}},
\qquad
\delta_i^t:=
\frac{\epsilon}{\sqrt{\bar q_i^t}}.
\]
The condition $|\kappa_i^t|=\Theta(1)$ means that the first-moment EMA does not cancel relative to the second-moment RMS. The condition $\delta_i^t=O(1)$ prevents $\epsilon$ from introducing an additional width-dependent suppression on order-one input coordinates. If $\bar q_i^t=0$, the row update is zero and no alignment lower bound holds.

\begin{lemma}[Adam soft-sign alignment]
\label{lem:adam_softsign_alignment}
\label{lem:ema_noncancellation}
Suppose the fixed-input gradient history satisfies
\[
(\nabla_A\Lcal_{t-k})_{ij}=s_i^{t-k}\underline Z_j.
\]
Assume that for row $i$,
\[
0<c_\kappa
\le |\kappa_i^t|
\le C_\kappa<\infty,
\qquad
\delta_i^t\le C_\delta<\infty,
\]
with constants independent of width $n$. If $\underline Z$ satisfies the input small-ball condition, then the Adam update satisfies
\[
\left|(U_A^t)_{i,:}\underline Z\right|=\Theta(n).
\]
\end{lemma}

\begin{proof}
Because the gradient factors as $s_i^{t-k}\underline Z_j$, the Adam moments also factor coordinate-wise:
\[
\widehat m_{ij}^t=\bar s_i^t\underline Z_j,
\qquad
\widehat v_{ij}^t=\bar q_i^t\underline Z_j^2.
\]
Therefore
\[
(U_A^t)_{ij}
=
\frac{\bar s_i^t\underline Z_j}
{\sqrt{\bar q_i^t\underline Z_j^2}+\epsilon}
=
\frac{\bar s_i^t\underline Z_j}
{\sqrt{\bar q_i^t}|\underline Z_j|+\epsilon}
=
\kappa_i^t\frac{\underline Z_j}{|\underline Z_j|+\delta_i^t}.
\]
Taking the inner product with $\underline Z$,
\[
(U_A^t)_{i,:}\underline Z
=
\kappa_i^t
\sum_{j=1}^n
\frac{\underline Z_j^2}{|\underline Z_j|+\delta_i^t}.
\]
Every term in the sum is nonnegative. Hence there is no coordinate-wise sign cancellation. For the at least $pn$ coordinates satisfying $z_-\le|\underline Z_j|\le z_+$,
\[
\frac{\underline Z_j^2}{|\underline Z_j|+\delta_i^t}
\ge
\frac{z_-^2}{z_+ + C_\delta}.
\]
Thus
\[
\sum_{j=1}^n
\frac{\underline Z_j^2}{|\underline Z_j|+\delta_i^t}
\ge
pn\frac{z_-^2}{z_+ + C_\delta}
=\Omega(n).
\]
The coordinate upper bound gives the matching upper bound
\[
\sum_{j=1}^n
\frac{\underline Z_j^2}{|\underline Z_j|+\delta_i^t}
\le
\sum_{j=1}^n|\underline Z_j|
\le nz_+
=O(n).
\]
Multiplying by $|\kappa_i^t|=\Theta(1)$ proves the claim.
\end{proof}

\begin{corollary}[Adam verifies the LoRA $A$-branch non-degeneracy]
\label{cor:adam_a_branch_nondegeneracy}
Under the conditions of \cref{lem:adam_softsign_alignment}, Adam satisfies the row-wise $A$-branch requirement in \cref{ass:normalized_updates}:
\[
|(U_A^t)_{i,:}\underline Z|=\Theta(n)
\]
Together with \cref{prop:adam_coordinate_bound}, this verifies the Adam coordinate bound and $A$-branch input product. Temporal non-cancellation and rank contraction remain assumptions in \cref{theorem:lora_feature_beta}.
\end{corollary}

\begin{proof}
Apply \cref{lem:adam_softsign_alignment,prop:adam_coordinate_bound}.
\end{proof}

The epsilon condition matters even for a sign-consistent history. Let $\underline Z_j=1$, $s_i^0=0$, and $s_i^k=1/n$ for $k\ge1$. At fixed $t\ge1$, $|\kappa_i^t|=\Theta(1)$ but $\delta_i^t=\Theta(\epsilon n)$, giving $|(U_A^t)_{i,:}\underline Z|=\Theta(n/(1+\epsilon n))=\Theta(1)$ for fixed $\epsilon>0$. Thus EMA non-cancellation does not ensure $\Theta(n)$ alignment.

\subsubsection{A Weighted-Angle View of Adam Non-Cancellation}

The Adam non-cancellation condition has a simple geometric interpretation. Let
\[
a_k=\beta_1^k,\qquad
b_k=\beta_2^k,\qquad
c_{1,t}=\sum_{k=0}^{t}a_k,\qquad
c_{2,t}=\sum_{k=0}^{t}b_k.
\]
For a fixed row $i$, write the scalar signal history as
\[
s_i=(s_i^{t-k})_{k=0}^{t}\in\Rbb^{t+1}.
\]
Then
\[
\kappa_i^t
=
\frac{\bar s_i^t}{\sqrt{\bar q_i^t}}
=
\frac{\sqrt{c_{2,t}}}{c_{1,t}}
\frac{\sum_{k=0}^{t}a_k s_i^{t-k}}
{\sqrt{\sum_{k=0}^{t}b_k(s_i^{t-k})^2}}.
\]
Let $D_b=\operatorname{diag}(\sqrt{b_0},\ldots,\sqrt{b_t})$. Then
\[
\sum_{k=0}^{t}a_k s_i^{t-k}
=
\langle D_b^{-1}a,D_b s_i\rangle,
\]
and
\[
\sqrt{\sum_{k=0}^{t}b_k(s_i^{t-k})^2}
=
\|D_b s_i\|_2.
\]
Thus
\[
\kappa_i^t
=
\frac{\sqrt{c_{2,t}}}{c_{1,t}}
\|D_b^{-1}a\|_2
\cos\theta_i^t,
\]
where $\theta_i^t$ is the angle between the Adam first-moment kernel $D_b^{-1}a$ and the RMS-weighted scalar signal history $D_b s_i$.

The condition $\beta_1^2<\beta_2$ guarantees
\[
\|D_b^{-1}a\|_2^2
=
\sum_{k=0}^{t}
\frac{\beta_1^{2k}}{\beta_2^k}
=
\sum_{k=0}^{t}
\left(\frac{\beta_1^2}{\beta_2}\right)^k
=\Theta(1).
\]
Since $1\le c_{1,t}\le(1-\beta_1)^{-1}$ and $1\le c_{2,t}\le(1-\beta_2)^{-1}$, the prefactor of $\cos\theta_i^t$ is bounded above and away from zero uniformly in $t$. For nonzero histories, $|\kappa_i^t|=\Theta(1)$ is therefore equivalent to $|\cos\theta_i^t|=\Theta(1)$. Under the fixed-input and small-ball assumptions, this and $\delta_i^t=O(1)$ ensure alignment. For $t\ge1$, it suffices that the entries after the initial zero gradient share a sign and have magnitudes bounded above and below by positive width-independent constants.

\subsubsection{AdamW and Effective Horizons}

AdamW combines Adam normalized directions with decoupled weight decay:
\[
A_{t+1}=(1-\eta_{A,t}\omega_A)A_t-\eta_{A,t}U_A^t,
\qquad
B_{t+1}=(1-\eta_{B,t}\omega_B)B_t-\eta_{B,t}U_B^t,
\]
where $\eta_{A,t},\eta_{B,t}\ge0$ and $\omega_A,\omega_B\ge0$. The directions $U_A^t,U_B^t$ are formed from loss gradients, excluding weight decay, along the AdamW trajectory. Define the decay products
\[
\Gamma_A(s,t):=
\prod_{\ell=s}^{t-1}(1-\eta_{A,\ell}\omega_A),
\qquad
\Gamma_B(s,t):=
\prod_{\ell=s}^{t-1}(1-\eta_{B,\ell}\omega_B).
\]
We use the convention $\Gamma_A(t,t)=\Gamma_B(t,t)=1$.

\begin{proposition}[Exact AdamW decomposition and noise bound]
\label{prop:adamw_effective_horizon}
Assume $B_0=0$ and $A_0=\beta\tilde A_0$. Define
\[
\mathcal A_t:=
\sum_{i=0}^{t-1}
\eta_{A,i}\Gamma_A(i+1,t)U_A^i,
\qquad
\mathcal B_t:=
\sum_{j=0}^{t-1}
\eta_{B,j}\Gamma_B(j+1,t)U_B^j.
\]
Then
\[
Z_B^t
=
-\lambda\beta\,\Gamma_A(0,t)\mathcal B_t\tilde A_0\underline Z
+\lambda \mathcal B_t\mathcal A_t\underline Z.
\]
Denote the first and second terms by $T_{\mathrm{noise}}^t$ and $T_{\mathrm{signal}}^t$, respectively.
If the decay factors are non-expansive,
\[
0\le1-\eta_{A,t}\omega_A\le1,
\qquad
0\le1-\eta_{B,t}\omega_B\le1,
\]
then the noise satisfies
\begin{align*}
\|T_{\mathrm{noise}}^t\|_2
&\le\lambda\beta\Gamma_A(0,t)
\sum_{j<t}\eta_{B,j}\Gamma_B(j+1,t)
\|U_B^j\tilde A_0\underline Z\|_2\\
&\le\lambda\beta\sum_{j<t}\eta_{B,j}
\|U_B^j\tilde A_0\underline Z\|_2.
\end{align*}
\end{proposition}

\begin{proof}
Unrolling the recursions gives
\[
A_t
=
\Gamma_A(0,t)A_0
-
\sum_{i=0}^{t-1}
\eta_{A,i}\Gamma_A(i+1,t)U_A^i,
\]
and, since $B_0=0$,
\[
B_t
=
-
\sum_{j=0}^{t-1}
\eta_{B,j}\Gamma_B(j+1,t)U_B^j.
\]
Using $A_0=\beta\tilde A_0$, this is
\[
A_t=\beta\Gamma_A(0,t)\tilde A_0-\mathcal A_t,
\qquad
B_t=-\mathcal B_t.
\]
Substituting into $Z_B^t=\lambda B_tA_t\underline Z$ yields
\[
Z_B^t
=
-\lambda\beta\,\Gamma_A(0,t)\mathcal B_t\tilde A_0\underline Z
+\lambda\mathcal B_t\mathcal A_t\underline Z.
\]
The first noise bound follows by the triangle inequality. Under non-expansive decay, all decay products lie in $[0,1]$, giving the second bound.
\end{proof}

These upper bounds do not imply that the realized noise decreases with weight decay. For fixed projected contributions $v\ne0$ and $-v$, the undamped sum is zero, whereas damping the earlier term gives $qv-v=-(1-q)v\ne0$ for $0<q<1$. Decay can therefore reduce cancellation; monotonicity of either component does not follow from bounds on the decay factors.

Under non-expansive decay, define the effective horizons
\[
H_A(t):=
\sum_{i=0}^{t-1}
\eta_{A,i}\Gamma_A(i+1,t),
\qquad
H_B(t):=
\sum_{j=0}^{t-1}
\eta_{B,j}\Gamma_B(j+1,t).
\]
If the coordinate bounds hold uniformly, then $\|\mathcal A_t\|_{\max}=O(H_A(t))$ and $\|\mathcal B_t\|_{\max}=O(H_B(t))$. Matching lower bounds require assumptions on the weighted sums themselves.

\paragraph{Conditional component scaling.}
Write $u:=\tilde A_0\underline Z$, $w_t:=\mathcal A_t\underline Z$, and $\rho_d(x):=\|x\|_2/\sqrt d$. On analyzed steps with positive $\Gamma_A(0,t),H_A(t),H_B(t)$, assume the weighted analogues of \cref{eq:input_accumulation_scales,eq:rank_contraction_scales}:
\begin{align*}
\rho_r(u)&=\Theta(\sigma_A\sqrt n), &
\rho_r(w_t)&=\Theta(nH_A(t)),\\
\rho_n(\mathcal B_tu)&=\Theta(H_B(t)r^{\gamma_N}\rho_r(u)), &
\rho_n(\mathcal B_tw_t)&=\Theta(H_B(t)r^{\gamma_S}\rho_r(w_t)),
\end{align*}
for fixed exponents $\gamma_N,\gamma_S$ and constants uniform in width, rank, and the analyzed steps. Setting $m_N:=\lambda r^{\gamma_N}$ and $m_S:=\lambda r^{\gamma_S}$, the exact decomposition gives
\begin{align*}
E_N^t&=\frac{1}{n}\|T_{\mathrm{noise}}^t\|_2^2
=\Theta\!\left(\beta^2m_N^2\Gamma_A(0,t)^2H_B(t)^2\sigma_A^2n\right),\\
E_S^t&=\frac{1}{n}\|T_{\mathrm{signal}}^t\|_2^2
=\Theta\!\left((m_SH_A(t)H_B(t)n)^2\right).
\end{align*}
The total energy still satisfies \cref{lem:realized_energy_decomposition}. The assumptions above impose weighted temporal non-cancellation and rank contraction; the scalar horizons alone do not establish them.

For constant LRs, if $\Gamma_A(0,t),\Gamma_B(0,t)=\Theta(1)$ uniformly on the analyzed window, then $H_A(t)=\Theta(\eta_A t)$ and $H_B(t)=\Theta(\eta_B t)$. Under the weighted contraction assumptions, this recovers the component scales in \cref{theorem:lora_feature_beta}. More generally, changing $\sigma_A^2=\Theta(n^{-1})$ to $\Theta(n^{-2})$ supplies the extra $1/n$ noise-energy factor when the input, multiplier, $\beta$, decay, horizon, and rank-contraction scales are matched across initializations. The optimizer trajectories may differ, and the exact decomposition alone does not establish these matched scales.

\subsubsection{Empirical Width Alignment and Single-Example Scope}
\label{app:empirical_width_alignment}

The alignment derivation above concerns a single example or token with a locally fixed input. With minibatches, each gradient combines multiple inputs and backpropagated signals, so its normalized direction need not retain this outer-product structure. We therefore measure realized alignment rather than assert a general batch extension.

For row $i$ of module $\ell$, let $a_{t,\ell,i}(n):=|(U_{A,\ell}^t)_{i,:}\underline Z_\ell|$. Across Qwen2.5-\{0.5B, 1.5B, 3B, 7B\}, we fit the module-level alignment magnitude as $a_{t,\ell}(n)\propto n^{\kappa_{t,\ell}}$ at sampled optimizer steps. Both initializations use three seeds, giving 24 runs, with Math-10K, $r=\alpha=8$, $\eta_A=10^{-5}$, and $\eta_B/\eta_A=200$.

We sample four \texttt{q\_proj} modules and the final-block \texttt{down\_proj}. The fan-in $n$ is the input width of the measured projection: $n\in\{896,1536,2048,3584\}$ for \texttt{q\_proj}, and $n\in\{4864,8960,11008,18944\}$ for \texttt{down\_proj}. We also fit the width exponent of the normalized alignment $a_{t,\ell,i}(n)/(\|(U_{A,\ell}^t)_{i,:}\|_2\|\underline Z_\ell\|_2)$. The single-example theory predicts a dot-product exponent of $1$ and a cosine exponent of $0$ under its stated non-degeneracy conditions.

\begin{figure}[t]
    \centering
    \includegraphics[width=\linewidth]{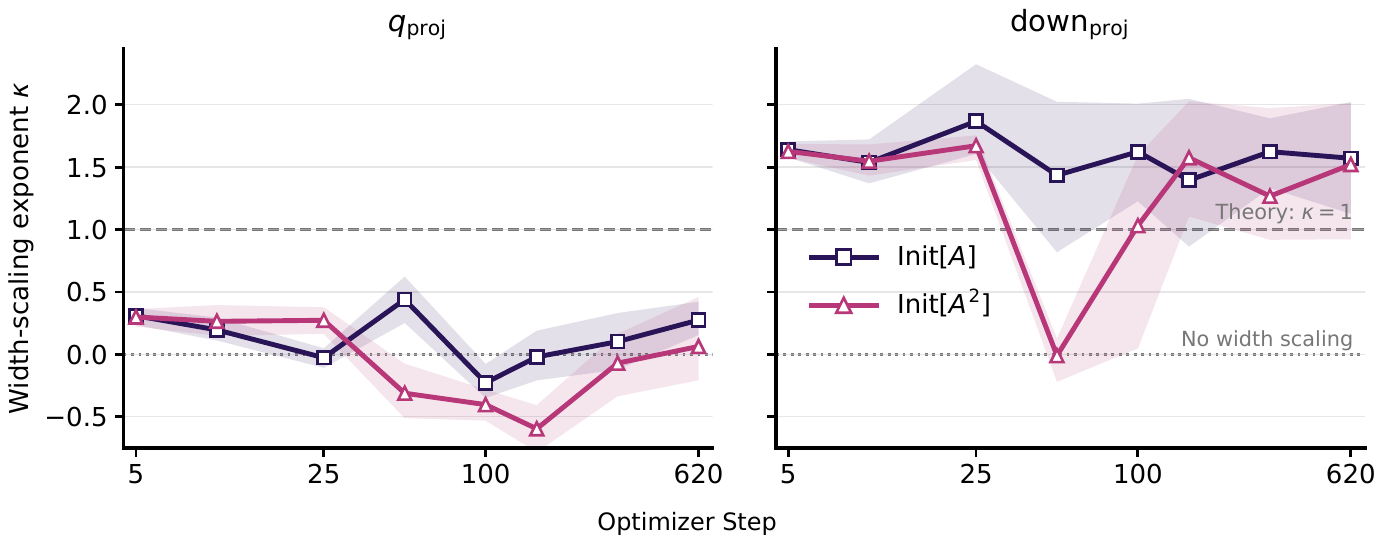}
    \caption{Empirical width-scaling exponent of the processed $A$-update/input alignment during Qwen2.5 minibatch fine-tuning on Math-10K. Panels compare \texttt{q\_proj} and \texttt{down\_proj} under \initA and \initAA. Reference lines mark the single-example prediction $\kappa=1$ and width-independent alignment $\kappa=0$. The exponent varies across modules and optimizer steps.}
    \label{fig:width_alignment}
\end{figure}

\begin{table}[htbp]
\centering
\small
\setlength{\tabcolsep}{3pt}
\caption{Width-alignment fits at steps 5 and 620. Dot-product exponents include 95\% CIs computed over three seeds; the last column gives cosine exponents at the two steps.}
\label{tab:width_alignment_exponents}
\begin{tabular}{llccc}
\toprule
Module & Initialization & \makecell{Step-5 dot $\kappa$\\95\% CI} & \makecell{Step-620 dot $\kappa$\\95\% CI} & \makecell{Cosine exponent\\step 5 / 620} \\
\midrule
\texttt{q\_proj} & \initA  & \makecell{0.309\\$[0.235,\,0.371]$} & \makecell{0.274\\$[0.155,\,0.421]$} & $-0.477 / -0.400$ \\
\texttt{q\_proj} & \initAA & \makecell{0.299\\$[0.223,\,0.362]$} & \makecell{0.063\\$[-0.208,\,0.459]$} & $-0.478 / -0.353$ \\
\midrule
\texttt{down\_proj} & \initA  & \makecell{1.640\\$[1.578,\,1.706]$} & \makecell{1.571\\$[1.126,\,2.021]$} & $0.172 / -0.052$ \\
\texttt{down\_proj} & \initAA & \makecell{1.626\\$[1.573,\,1.686]$} & \makecell{1.521\\$[0.922,\,2.013]$} & $0.174 / -0.098$ \\
\bottomrule
\end{tabular}
\end{table}

\Cref{fig:width_alignment} shows predominantly sublinear alignment in \texttt{q\_proj} and generally superlinear alignment in \texttt{down\_proj}, with substantial changes during training. Thus $\Theta(n)$ alignment is not generic for minibatch fine-tuning. Our formal signal-scaling and efficiency results remain conditional on the fixed single-example analysis; extending them to minibatches requires the realized alignment scale and its temporal accumulation. This restriction affects the signal estimate, while the separate initialization-noise comparison retains the extra $1/n$ factor under matched input and contraction assumptions.

\subsubsection{Summary for the Main Theorem}

\begin{corollary}[Adam and AdamW instantiate the normalized-update analysis]
\label{cor:adam_adamw_normalized_analysis}
\label{cor:adam_instantiates_assumption}
Consider a fixed LoRA-layer input $\underline Z$ satisfying the small-ball condition and zero initial optimizer moments. Assume Adam hyperparameters $(\beta_1,\beta_2,\epsilon)$ are width-independent, with $\epsilon>0$, $0\le\beta_1<1$, and $\beta_1^2<\beta_2<1$. Suppose the $A$-gradient has the local row-wise form
\[
(\nabla_A\Lcal_{t-k})_{ij}=s_i^{t-k}\underline Z_j
\]
over the complete Adam gradient history. If the following bounds hold uniformly over the analyzed rows and steps,
\[
\left|
\frac{\bar s_i^t}{\sqrt{\bar q_i^t}}
\right|
=\Theta(1),
\qquad
\frac{\epsilon}{\sqrt{\bar q_i^t}}=O(1),
\]
then Adam satisfies
\[
\|U_A^t\|_{\max}=O(1),
\qquad
|(U_A^t)_{i,:}\underline Z|=\Theta(n).
\]
The same coordinate bound applies to $U_B^t$. Therefore Adam satisfies the optimizer-facing requirements of \cref{ass:normalized_updates} under these conditions. Temporal non-cancellation in both branches and the rank contractions remain separate assumptions in \cref{theorem:lora_feature_beta}.

For AdamW, the same optimizer-facing bounds hold for its loss-gradient directions under the same conditions. Feature scaling uses the exact decay-weighted decomposition in \cref{prop:adamw_effective_horizon} and additionally requires the weighted temporal and rank-contraction assumptions stated above. The extra $1/n$ noise-energy comparison holds under matched projection, decay, horizon, and contraction scales across initializations; non-expansive decay alone does not guarantee it.
\end{corollary}

\subsection{Spectral Derivation of \initAA}
\label{app:spectral_derivation}

We derive \initAA from spectral conditions for stable feature learning. Spectral and Frobenius norm scalings capture matrix behavior in the large-width limit beyond coordinate-wise estimates.

\subsubsection{Matrix Norm and Low-Rank Equivalence}

The spectral norm of a matrix $A$, denoted $\lVert A \rVert_{\mathrm{op}}$, is its largest singular value, defined as the maximum stretch applied to any nonzero vector:
\begin{equation}
\lVert A \rVert_{\mathrm{op}} = \max_{v \neq 0} \frac{\lVert A v \rVert_2}{\lVert v \rVert_2}.
\end{equation}
The spectral norm can be bounded by the Frobenius norm:
\begin{equation}
    1/\sqrt{r} \lVert A \rVert_F \le \lVert A \rVert_{\mathrm{op}} \le \lVert A \rVert_F.
\end{equation}
Here $r\ge1$ is an upper bound on $\operatorname{rank}(A)$. In the fixed-rank LoRA setting of this subsection, the spectral and Frobenius norms are asymptotically equivalent: $\lVert \cdot \rVert_F \asymp \lVert \cdot \rVert_{\mathrm{op}}$.

\subsubsection{General Spectral Condition}

A central requirement for feature learning in large-width neural networks is that the feature representations must evolve nontrivially without vanishing or exploding. Following the feature-learning desiderata proposed by \citet{yang2023spectral}, we start with the general condition for a weight matrix $W_\ell \in \mathbb{R}^{n_\ell \times n_{\ell-1}}$ at layer $\ell$.

\begin{condition}[Spectral Scaling Condition]
\label{cond:general_spectral}
Consider applying a gradient update $\Delta W_\ell \in \mathbb{R}^{n_\ell \times n_{\ell-1}}$ to the $\ell$-th weight matrix $W_\ell$. To target the feature learning desideratum $\|h_\ell\|_2 = \Theta(\sqrt{n_\ell})$ and $\|\Delta h_\ell\|_2 = \Theta(\sqrt{n_\ell})$, we use the spectral scaling prescription:
\begin{equation}
    \|W_\ell\|_{\mathrm{op}} = \Theta\left(\sqrt{\frac{n_\ell}{n_{\ell-1}}}\right)
    \quad \text{and} \quad
    \|\Delta W_\ell\|_{\mathrm{op}} = \Theta\left(\sqrt{\frac{n_\ell}{n_{\ell-1}}}\right),
    \quad \text{for all layers } \ell.
\end{equation}
\end{condition}

This condition controls the maximum amplification of an input signal. Matching lower bounds for realized features and their updates additionally require alignment and non-cancellation; spectral norms alone do not exclude inputs in a nullspace.

\paragraph{Application to LoRA.}
In the context of LoRA, we identify two distinct linear transformations: the down-projection $A \in \mathbb{R}^{r \times n}$ and the up-projection $B \in \mathbb{R}^{n \times r}$.
For the down-projection $A$, we have $n_{\ell-1} = n$ and $n_\ell = r$. Substituting these into \cref{cond:general_spectral} and noting that $r = \Theta(1)$, we require:
\[
    \|A\|_{\mathrm{op}} = \Theta\left(\sqrt{\frac{r}{n}}\right) = \Theta(n^{-1/2}).
\]
Conversely, for the up-projection $B$, we have $n_{\ell-1} = r$ and $n_\ell = n$. This yields the requirement:
\[
    \|B\|_{\mathrm{op}} = \Theta\left(\sqrt{\frac{n}{r}}\right) = \Theta(n^{1/2}).
\]
The updates $\Delta A$ and $\Delta B$ require the same spectral scalings for stable feature evolution.

\subsubsection{Implementing Feature Learning via \initAA}

At fixed rank $r=\Theta(1)$, spectral and Frobenius norms have the same asymptotic scaling: $\|A\|_{\mathrm{op}} \asymp \|A\|_F$. We therefore derive initialization variances and LRs from Frobenius norms. For an Adam-like optimizer, we model typical update entries as $(\Delta A)_{ij} \sim \eta_A$ and $(\Delta B)_{ij} \sim \eta_B$.

\paragraph{Derivation for block $A$ ($A \in \mathbb{R}^{r \times n}$).}
The spectral condition mandates that the Frobenius norm squared scales as $\|A\|_F^2 \asymp \|A\|_{\mathrm{op}}^2 = \Theta(n^{-1})$. With $A_0=\beta\tilde A_0$, $\beta=\Theta(1)$, and independent centered Gaussian entries of variance $\sigma_A^2$ in $\tilde A_0$, the expected squared Frobenius norm is $\mathbb{E}[\|A_0\|_F^2] = \beta^2 r n \sigma_A^2$. At fixed rank, $\|A_0\|_F^2=\Theta(n\sigma_A^2)$ with high probability. Matching this to $\Theta(n^{-1})$ gives $\sigma_A^2 = \Theta(n^{-2})$. Similarly, under the assumed update-entry scaling, $\|\Delta A\|_F^2 \approx r n \eta_A^2 = \Theta(n \eta_A^2)$. To satisfy $\|\Delta A\|_F^2 = \Theta(n^{-1})$, we set $\eta_A = \Theta(n^{-1})$.

\paragraph{Derivation for block $B$ ($B \in \mathbb{R}^{n \times r}$).}
For the up-projection matrix, the spectral condition requires $\|B\|_F^2 \asymp \|B\|_{\mathrm{op}}^2 = \Theta(n)$. Following the same logic, if $B_0$ were initialized with variance $\sigma_B^2$, the expected norm would be $\mathbb{E}[\|B_0\|_F^2] = n r \sigma_B^2 = \Theta(n \sigma_B^2)$. Matching this to $\Theta(n)$ implies $\sigma_B^2 = \Theta(1)$. While standard LoRA initializes $B_0 = 0$, this transient violation is corrected in the first step provided the updates are properly scaled. For the updates, $\|\Delta B\|_F^2 \approx n r \eta_B^2 = \Theta(n \eta_B^2)$. To achieve the target $\Theta(n)$ scaling, we require $n \eta_B^2 = \Theta(n)$, which simplifies to $\eta_B = \Theta(1)$.

The derivation above dictates the following HPs to satisfy the spectral conditions for stable feature learning: $\sigma_A^2 = \Theta(n^{-2})$ with $\eta_A = \Theta(n^{-1})$, and by extension $\eta_B = \Theta(1)$.
This configuration corresponds to our proposed \initAA strategy, using $\Theta(n^{-2})$ variance for the down-projection rather than the standard $\Theta(n^{-1})$ scale. It controls spectral magnitudes; non-degenerate feature learning additionally requires the alignment and non-cancellation assumptions stated above.

\subsection{Spectral Conditions for Robust LoRA Feature Learning}
\label{app:spectral_conditions}
Prior work \citep{hayou2024impact,hayou2024lora+,libeyond} provides useful intuition for LoRA scaling, but its coordinate-wise, average-case $\gamma$-calculus can obscure a failure mode in practical fine-tuning. We explain how a spectral, worst-case perspective \citep{yang2023spectral} leads to the \initAA scaling.

\subsubsection{Why $\gamma$-Calculus Is Not Sufficient}

The $\gamma$-calculus tracks the coordinate-wise magnitude of random variables. The framework is implicitly average-case, assuming that summations exhibit typical cancellations and that individual coordinates are representative of global behavior. This assumption leads to fundamental limitations.

For a matrix-vector product $Wx$, the scaling also depends on cancellation among the products $W_{ij}x_j$. Under suitable independence and moment assumptions, two benchmark rules are:
\begin{itemize}
    \vspace{-0.1cm}
    \item \textbf{CLT-like (centered summands):} $\gamma[Wx] = \gamma[W] + \gamma[x] + 1/2$
    \vspace{-0.15cm}
    \item \textbf{LLN-like (nonzero-mean summands):} $\gamma[Wx] = \gamma[W] + \gamma[x] + 1$
    \vspace{-0.1cm}
\end{itemize}
Neither rule follows from the marginal mean of $W$ alone when $W$ and $x$ are dependent. In LoRA, the initialization-dependent and optimizer-driven components can have different rank contractions, as captured by $\gamma_N$ and $\gamma_S$ in \cref{theorem:lora_feature_beta}. Tracking only the final feature's coordinate exponent does not distinguish these components or their signed cross term. The explicit decomposition is therefore needed to diagnose noise dominance even when the total feature satisfies a bounded scaling condition.

\subsubsection{Spectral Condition: Worst-Case Control}

Following \citet{yang2023spectral}, the worst-case spectral condition is:

\begin{condition}[Spectral Condition for LoRA Feature Learning]
For $n\to\infty$ and fixed rank $r=\Theta(1)$, the LoRA factors and their updates satisfy:
\begin{equation*}
    \begin{cases}
        \lVert A \rVert_{\mathrm{op}} = \Theta( n^{-1/2}) \\
        \lVert \Delta A \rVert_{\mathrm{op}} = \Theta( n^{-1/2})
    \end{cases}
    \quad \text{and} \quad
    \begin{cases}
    \lVert B \rVert_{\mathrm{op}} = \Theta( n^{1/2}) \\
    \lVert \Delta B \rVert_{\mathrm{op}} = \Theta( n^{1/2})
    \end{cases}
\end{equation*}
\end{condition}
For $\|\underline Z\|_2=\Theta(\sqrt n)$ aligned with a principal singular vector, $\|A\|_{\mathrm{op}}=\Theta(n^{-1/2})$ gives an order-one bottleneck feature $Z_A=A\underline Z$. Conversely, $\|B\|_{\mathrm{op}}=\Theta(n^{1/2})$ maps an aligned order-one bottleneck vector to output scale $\Theta(\sqrt n)$. These worst-case scales do not by themselves give lower bounds for arbitrary inputs.

By \cref{lemma:decomps_2}, $T_{\text{noise}}^t=-(\alpha\beta/r)\eta_B(\sum_{i<t}U_B^i)\tilde A_0\underline Z$. Applying a spectral bound gives
\begin{equation*}
    \| T_{\text{noise}}^t \|_2 \le \frac{\alpha \beta}{r} \eta_B \| \sum_{i=0}^{t-1} U_B^i\|_{\mathrm{op}} \|\tilde{A}_0\|_{\mathrm{op}} \|\underline{Z}\|_2.
\end{equation*}
At fixed rank, coordinate-bounded directions give $\|\sum_{i<t}U_B^i\|_{\mathrm{op}}=O(t\sqrt n)$, while Gaussian initialization gives $\|\tilde A_0\|_{\mathrm{op}}=\Theta(\sigma_A\sqrt n)$ with high probability. Therefore the worst-case noise RMS satisfies
\[
\frac{\|T_{\text{noise}}^t\|_2}{\sqrt n}
=O\!\left(\frac{\alpha\beta}{r}\eta_B t\sigma_A n\right).
\]
For fixed $t$ and $\alpha/r,\beta,\eta_B=\Theta(1)$, this bound is $O(\sqrt n)$ under \initA and $O(1)$ under \initAA. Thus the spectral argument establishes bounded worst-case noise RMS for \initAA, but does not establish vanishing noise or signal dominance for all aligned inputs.

Vanishing noise and signal dominance require the fixed-input projection and contraction assumptions in \cref{theorem:lora_feature_beta,theorem:initAA_feature_learning}; we derive the spectral initialization rule in \cref{app:spectral_derivation}.


\section{Related Literature Review}
\label{app:related_work}
\subsection{LoRA Initialization}
\label{app:related_work_lora_init}

Standard LoRA sets $B_0=0$ and initializes $A_0$ with mean-zero entries of variance $\Theta(n^{-1})$, denoted \initA; PEFT uses Kaiming-uniform initialization (\cref{app:init_implementation}). \citet{hayou2024impact} show that the symmetric configuration \initB, with $A_0=0$ and $B_0$ initialized at entrywise variance $\Theta(1)$, has inferior optimization dynamics and stability compared with \initA. PEFT \citep{peft} also offers initialization strategies to accelerate convergence or enhance stability. LoftQ \citep{li2023loftq} minimizes the quantization error between original and quantized weights. PiSSA \citep{meng2024pissa} and EVA \citep{paischer2024parameter} adopt spectral approaches: PiSSA uses principal singular components of the weight matrix to approximate full-rank updates, while EVA uses a singular value decomposition of layer input activations. CorDA \citep{yang2024corda} uses context-oriented decomposition to preserve world knowledge or enhance task-specific learning. OLoRA \citep{buyukakyuz2024olora} and orthogonal initialization use orthonormal decompositions for stable training.

\subsection{Other LoRA Hyperparameters}
\label{app:related_work_lora_hps}

Beyond initialization, LoRA introduces sensitivity to several coupled HPs, including the rank $r$, scaling factor $\alpha$, and LRs $\eta_A,\eta_B$. The rank controls the expressivity of the adapter: \citet{zengexpressive} characterize the expressive power of low-rank adaptation and clarify how the choice of $r$ constrains the class of attainable updates. However, increasing rank does not automatically improve vanilla LoRA. \citet{kalajdzievski2023rank} show that the conventional multiplier $\alpha/r$ can make the LoRA branch weaker as $r$ grows, so higher ranks may yield diminishing gains in practice. Their rank-stabilized alternative replaces this scaling with $\alpha/\sqrt{r}$, preserving a more consistent update scale across ranks.

LR selection introduces a separate source of brittleness. In the feature-learning view, useful transfer requires both the hidden states and their updates to remain $\Theta(1)$ as width increases. \citet{hayou2024lora+} analyze LoRA in the infinite-width limit and derive stability conditions under which non-degenerate feature updates require a large LR separation, with $\eta_B \gg \eta_A$ and an idealized ratio that scales with width. Recent work further studies the joint width-and-rank limit, showing that optimal LR-rank scaling depends on the initialization and LoRA update multiplier \citep{chen2026learningrate}. These theories provide guidance, but pretrained LLMs typically do not use the corresponding width-scaling parameterization. Empirical studies find that vanilla LoRA can be competitive after careful LR search \citep{lee2026learning}, while the best asymmetric LR ratios observed in practice are often far smaller than the model width \citep{hayou2024lora+}. We address this theory-to-practice gap through initialization: standard \initA injects persistent noise under aggressive asymmetric LRs, whereas \initAA suppresses it and improves robustness.

\subsection{Feature Learning and HP Transfer}
\label{app:related_work_feature_learning}

Feature learning \citep{yang2021tensor} is a central mechanism in modern deep learning, wherein neural networks are trained not merely to match target outputs, but to acquire internal representations that encode useful invariances and hierarchical abstractions. This stands in contrast to the Neural Tangent Kernel (NTK) or “lazy” training regime \citep{jacot2018neural,arora2019exact,liu2022loss}, in which an infinite-width network behaves as a kernel machine with features effectively frozen at their random initialization. In the feature-learning regime, by contrast, the network’s internal feature map evolves substantially over the course of training. To facilitate scalable feature learning, \citet{yang2021tensor} introduced the maximal update parametrization ($\mu$P), which enables nontrivial feature-learning dynamics even in the infinite-width limit. Unlike the NTK parametrization, $\mu$P ensures that pre-activations, gradients, and parameter updates remain $\Theta(1)$ as the network width tends to infinity. Closely related approaches based on mean-field theory \citep{chizat2018global,mei2018mean,sirignano2020mean} likewise analyze the evolution of the distribution of neurons and demonstrate that learned features can emerge outside the kernel regime.

While early scaling analyses focused on width, recent work has extended these ideas to the depth scaling limit, where the number of layers $L \to \infty$. Naively deep networks are prone to vanishing activations or exploding gradients, but these issues can be resolved by scaling residual blocks as $O(1/\sqrt{L})$ \citep{yang2023feature,bordelon2023depthwise}. 

A key use case is HP transfer across model scales \citep{yang2022tensor}. Because $\mu$P maintains consistent training dynamics across width and depth, practitioners can tune LRs and other HPs on small proxy models for use in larger networks. Empirical methods \citep{milsom2025function, large2024scalable, riabinin2025gluon} similarly normalize layer outputs and updates for width transfer. Related work studies transformer width-to-depth ratios \citep{dey2025don}, sharpness-aware optimization \citep{haas2024boldsymbol}, sparse training \citep{dey2024sparse}, and state space models \citep{vankadara2024feature}.

\subsection{Training Stability}
\label{app:related_work_stability}

While feature learning offers insight into representation adaptation, such analyses typically depend on a specific choice of parameterization to guarantee well-behaved limits as network width grows. However, not all models are parameterized in this manner, and identifying an appropriate parameterization can be challenging for complex architectures. This motivates a complementary line of research that elevates training stability to a primary objective, often treated as orthogonal to, or only weakly coupled with, feature-learning optimality.

Historically, stability concerns were most prominent in recurrent neural networks, where unconstrained temporal composition of states led to exploding or vanishing gradients. Architectures such as Long Short-Term Memory (LSTM) \citep{hochreiter1997long} and Gated Recurrent Units (GRUs) \citep{cho2014learning} mitigated these issues by introducing gated, additive state updates that constrain the hidden dynamics to a contractive, well-conditioned manifold, thereby enabling stable gradient propagation over long time horizons. From a geometric perspective, these mechanisms effectively bound the recurrent Jacobian and attenuate sensitivity to HPs.

Muon \citep{jordan2024muon} and the Spectral Sphere Optimizer \citep{xie2026controlled} constrain optimization trajectories to spectrally controlled manifolds, regulating operator norms for stability. Layer Normalization \citep{ba2016layer}, Post-LN Transformers \citep{wang2019learning}, and Pre-LN Transformers \citep{xiong2020layer} can also be interpreted as imposing implicit spectral control. DeepSeek’s manifold hyperconnections \citep{xie2025mhc} constrain learned residual mixing matrices to the doubly stochastic manifold, yielding convex, well-conditioned layer compositions.

\section{Experimental Setup}
\label{app:experimental_setup}

\subsection{Initialization Implementation}
\label{app:init_implementation}

For transformer fine-tuning (RoBERTa, Llama, Qwen, and Mistral), we follow the \texttt{PEFT} convention \citep{peft}: Kaiming-uniform $A$ \citep{he2015delving} and $B_0=0$. Under Kaiming-uniform,
\[
A_{ij} \sim \mathcal{U}(-b, b),
\qquad 
b = \sqrt{3\,\var(A_{ij})},
\]
so that the uniform bound $b$ and the variance are related by $\var(A_{ij}) = b^2/3$.

We parameterize the standard deviation of $A$ by $\beta$. For a LoRA layer with input dimension $\text{fan\_in}$, we set
\[
\var(A_{ij}) = \frac{\beta^2}{(\text{fan\_in})^{2s}},
\]
where \initA uses $s=\tfrac12$ and \initAA uses $s=1$. Thus \initA has variance $\beta^2/\text{fan\_in}$, while \initAA has variance $\beta^2/\text{fan\_in}^2$ and suppresses initialization noise by another factor of $1/\text{fan\_in}$. For rectangular layers, $n$ denotes the fan-in of $A$.

\subsection{Toy Example}
\label{app:toy_setup}

\paragraph{Dataset.} We use MNIST, flattening each image to $x \in \Rbb^{784}$ with one of 10 labels.

\paragraph{Model.} We use $Y=W_{\text{out}}\Fcal(W_0+BA,\Fcal(W_{\text{in}},X))$ with $\Fcal(W,z)=\operatorname{ReLU}(Wz)$, $W_{\text{in}}\in\Rbb^{n\times784}$, $W_{\text{out}}\in\Rbb^{10\times n}$, $A\in\Rbb^{r\times n}$, and $B\in\Rbb^{n\times r}$. We set $r=8$ and sweep $n\in\{16,64,256,1024,4096\}$.

\paragraph{Training.} We train only $A$ and $B$ for 20 epochs at batch size 256 using Adam \citep{kingma2014adam} with $\beta_1=\beta_2=0$ (no moving averages). We fix $W_0=0$, draw and freeze $W_{\text{in}} \sim \Ncal(0,1)$ and $W_{\text{out}} \sim \Ncal(0,1/n^2)$, compare \initA and \initAA with $\eta_B=n\eta_A$, and sweep $\eta_A \in \{2^{-7},2^{-6},\cdots,2^0\}$. The released reconstruction uses ReLU subgradient $1$ at zero (\texttt{clamp\_min(0)}); subgradient $0$ would prevent the first update when $W_0=B_0=0$. The original toy source is not archived, and exact numerical reproduction of the plots has not been verified.

\subsection{Encoder-Only Model Setup}
\label{app:encoder_only_setup}

\paragraph{Setup.} We fine-tune RoBERTa-base \citep{liu2019roberta} on GLUE MNLI, QQP, SST-2, and QNLI \citep{wang2018glue}, applying LoRA to the \texttt{query} and \texttt{value} projections. We use $r=\alpha=8$ and the Kaiming-uniform implementation from \cref{app:init_implementation}. Fixed-ratio sweeps set $\eta_A=10^{-5}$ and vary $\eta_B/\eta_A \in \{1,50,200,400,800\}$; 2D sweeps vary $\eta_A,\eta_B$ directly at $\beta \in \{0.25,8\}$.

\paragraph{Optimization.} We use AdamW \citep{loshchilov2019decoupled} with $(\beta_1,\beta_2,\epsilon)=(0.9,0.99,10^{-8})$, a linear schedule without warmup, sequence length $128$, batch size $32$, and 3 epochs. Runs use NVIDIA L40S GPUs and each individual run takes roughly one hour.

\subsection{Decoder-Only Model Setup}
\label{app:decoder_only_setup}

\paragraph{Datasets and models.} We use the arithmetic/math and commonsense benchmarks from \citet{hu2023llm}, following \citet{libeyond}: 10k arithmetic/math examples and 170k commonsense examples. The main sweeps use Llama 3-8B for initialization scale and Llama 3.2-1B for rank \citep{dubey2024llama}; appendix sweeps add QLoRA, stronger $A$-scalings, Llama 3.2-3B, Qwen2.5-3B \citep{qwen2024qwen25}, Qwen3-0.6B \citep{yang2025qwen3}, and Mistral-v0.3 \citep{jiang2023mistral}. Long-horizon, multi-seed multiplier, and energy-decomposition checks use Llama 3-8B (\cref{app:long_horizon,app:multiseed_validation,app:empirical_energy_decomposition}). By default, we apply LoRA to \texttt{q\_proj}, \texttt{k\_proj}, \texttt{v\_proj}, \texttt{up\_proj}, and \texttt{down\_proj}.
\paragraph{LoRA HPs.} For Llama 3-8B, we use $r=\alpha=8$ and sweep $\beta \in \{0.25,0.5,1,2,4,8\}$ and $\rho=\eta_B/\eta_A \in \{1,10,200,500,1000,2000,4000\}$. For decoder-only ratio sweeps, unless stated otherwise, we fix $\eta_A=10^{-5}$ and set $\eta_B=\rho\eta_A$. The Llama 3.2-1B rank sweep fixes $\beta=1$, varies $r\in\{4,8,16,32,64,128\}$, and sets $\alpha=r$. The QLoRA comparison uses the same Llama 3-8B grid; the model-family sweep uses $r=\alpha=8$, $\beta \in \{0.5,1,2,4,8\}$, and the ratio grid in \cref{fig:more_models_lora}; the initialization-power sweep varies $\beta \in \{4,8,12,16\}$ and $\rho \in \{1,10,50,200,500,1000,2000,4000,8000,10000,20000\}$.
\paragraph{Optimization.} We use AdamW \citep{loshchilov2019decoupled} with $(\beta_1,\beta_2,\epsilon)=(0.9,0.999,10^{-8})$, zero weight decay, global batch size 16, micro-batch size 8, and sequence length 256. We train 1 epoch on commonsense and 3 epochs on arithmetic using NVIDIA L40S GPUs; each run takes roughly three hours. The arithmetic diagnostics in \cref{app:long_horizon} compare separately trained horizons of up to 10 epochs.

\section{Additional Results and Ablations}
\label{app:additional_results}

\subsection{LR Sweep Experiment}
\label{app:lr_sweep_experiment}
\begin{figure}[!h]
    \centering
    \includegraphics[width=1\linewidth]{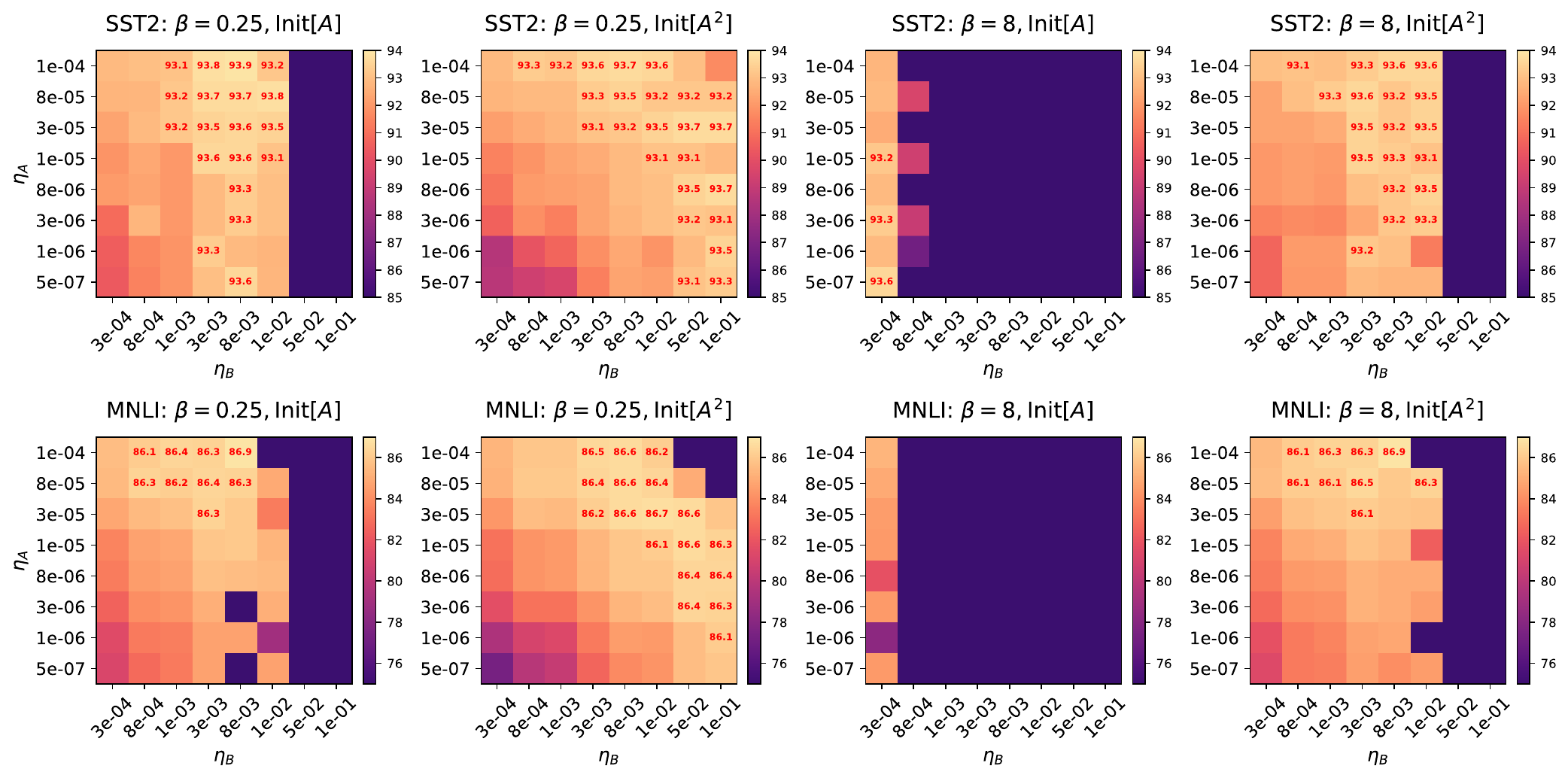}
    \caption{RoBERTa accuracy on SST-2 (top) and MNLI (bottom) under \initA and \initAA. The grids sweep LRs $\eta_A$ and $\eta_B$ at two initialization scales $\beta$.}
    \label{fig:abl_lr_sweep}
\end{figure}

We sweep the individual LRs $\eta_A$ and $\eta_B$ at initialization scales $\beta\in\{0.25,8\}$ to compare the usable regions of \initA and \initAA. In \cref{fig:abl_lr_sweep}, both methods have broad high-accuracy regions at $\beta=0.25$. At $\beta=8$, the high-accuracy region for \initA narrows to low $\eta_B$ values, whereas \initAA retains a broad region across the tested LRs. These results support reduced sensitivity to initialization scale over the tested LR grid.

\subsection{Sensitivity to the $B$ LR}
\label{app:block_b_important}

\begin{figure}[!h]
    \centering
    \includegraphics[width=1\linewidth]{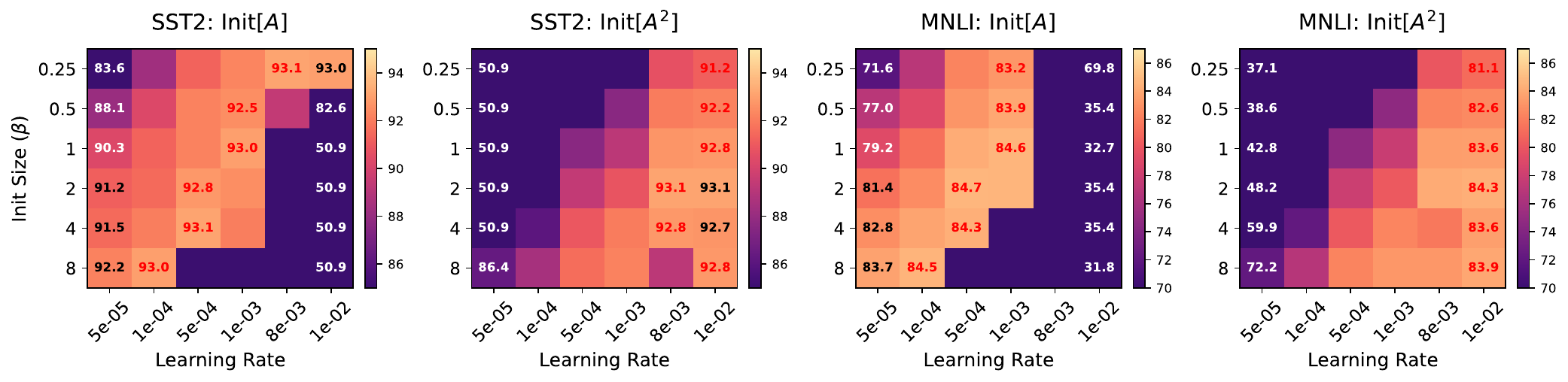}
    \caption{LoRA accuracy on MNLI and SST-2 with frozen $A$ and varying $B$ LR.}
    \label{fig:freeze_A}
\end{figure}

Under the assumptions of \cref{theorem:lora_feature_beta}, $\eta_B$ multiplies both the signal and noise components, while $\eta_A$ appears explicitly only in the signal. At fixed projection and contraction scales, increasing $\beta$ under \initA requires a smaller $\eta_B$ to keep the noise energy at the same level. To examine the role of $\eta_B$, we freeze block $A$ ($\eta_A=0$) in the RoBERTa MNLI and SST-2 setup and sweep $\eta_B\in\{5\times10^{-5},1\times10^{-4},5\times10^{-4},1\times10^{-3},8\times10^{-3},1\times10^{-2}\}$ for both initializations. \Cref{fig:freeze_A} shows that \initAA reaches high accuracy at the largest tested $\eta_B$ values across $\beta$, whereas the best $\eta_B$ for \initA shifts downward as $\beta$ increases.
\subsection{Additional Llama Experiments}
\label{app:additional_llama}

\begin{figure}[t]
    \centering
    \includegraphics[width=\linewidth]{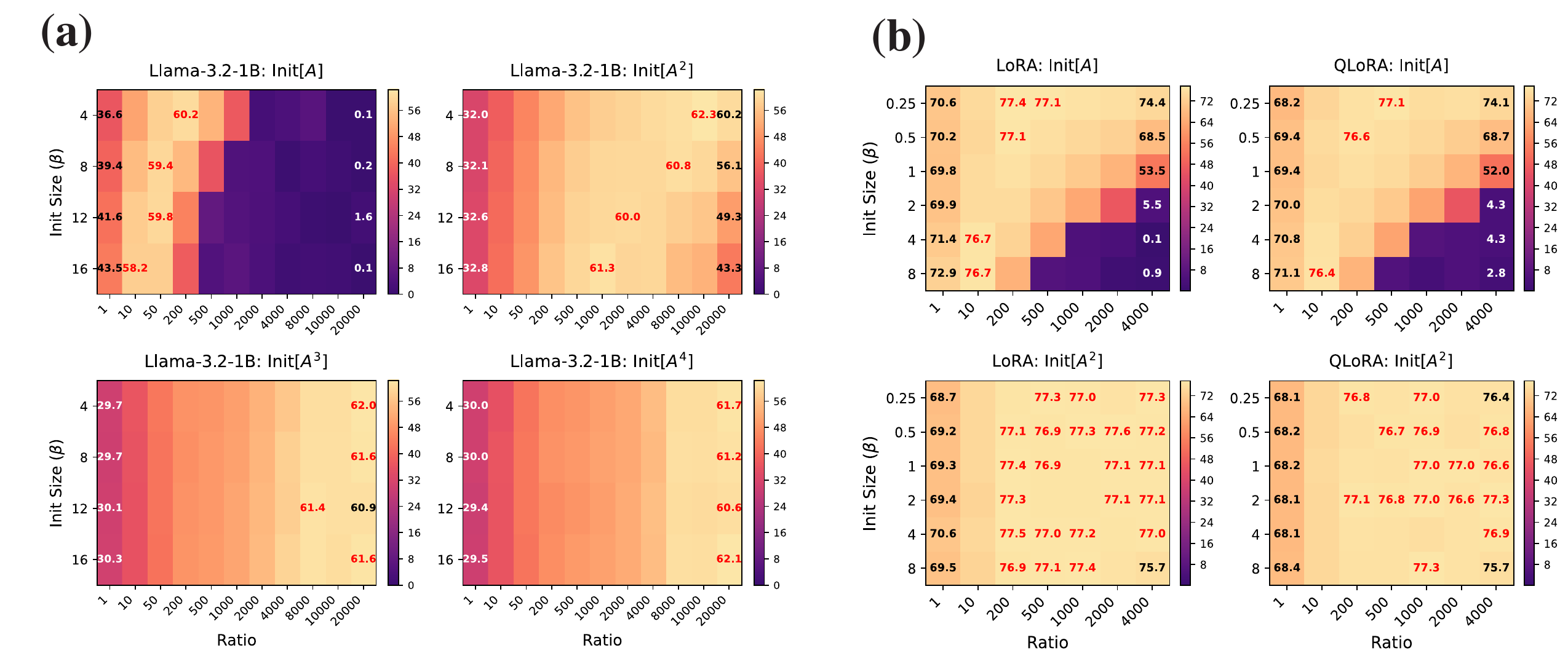}
    \caption{Additional Llama arithmetic experiments. (a) Llama 3.2-1B heatmaps compare \initA, \initAA, \texttt{Init[$\text{A}^3$]}, and \texttt{Init[$\text{A}^4$]} over initialization scale $\beta \in \{4,8,12,16\}$ and LR ratio $\rho=\eta_B/\eta_A \in \{1,10,50,200,500,1000,2000,4000,8000,10000,20000\}$. (b) Llama 3-8B heatmaps compare full-precision LoRA and QLoRA under \initA and \initAA over $\beta \in \{0.25,0.5,1,2,4,8\}$ and $\rho \in \{1,10,200,500,1000,2000,4000\}$. In both panels, each ratio uses $\eta_A=10^{-5}$ and $\eta_B=\rho\eta_A$. QLoRA exhibits the same high-ratio sensitivity under \initA as full-precision LoRA, while \initAA preserves a broad high-accuracy region in both settings.}
    \label{fig:llama_additional_subfigures}
    \label{fig:llama1b_initialization_powers}
    \label{fig:llama8b_lora_vs_qlora}
\end{figure}

\begin{figure}[t]
    \centering
    \includegraphics[width=0.98\linewidth]{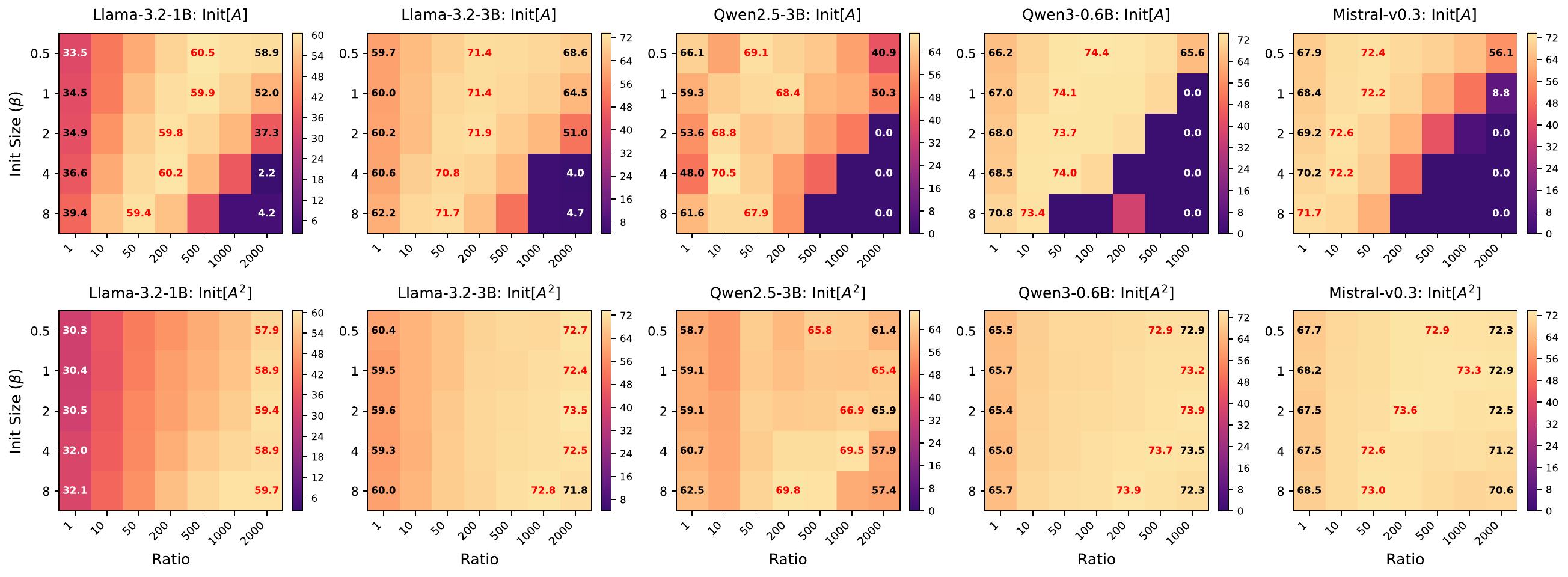}
    \caption{Arithmetic accuracy across model families under standard \initA (top) and \initAA (bottom). Columns correspond to Llama 3.2-1B, Llama 3.2-3B, Qwen2.5-3B, Qwen3-0.6B, and Mistral-v0.3. Each heatmap sweeps initialization scale $\beta \in \{0.5,1,2,4,8\}$ on the vertical axis and the LR ratio $\rho=\eta_B/\eta_A$ on the horizontal axis, with ratio values shown in each panel; each ratio uses $\eta_A=10^{-5}$ and $\eta_B=\rho\eta_A$.}
    \label{fig:more_models_lora}
\end{figure}

We further test whether the initialization behavior persists under quantized fine-tuning, additional model families, and smaller Llama backbones. Unless stated otherwise, these runs follow \cref{app:decoder_only_setup} with $r=\alpha=8$, AdamW, the same LoRA modules, and arithmetic evaluation. For QLoRA \citep{dettmers2023qlora}, the pretrained backbone is quantized while gradients update the LoRA adapters.

Panel~(b) of \cref{fig:llama_additional_subfigures} shows similar LR-ratio sensitivity for LoRA and QLoRA. Under \initA at $\eta_B/\eta_A=2000$, raising $\beta$ from $0.5$ to $8$ lowers accuracy from $71.9\%$ to $1.8\%$ for LoRA and from $70.7\%$ to $4.5\%$ for QLoRA. Under \initAA, LoRA reaches $80.2\%$ at $\beta=0.5$ and stays within $77.7$--$79.3\%$ at ratio $2000$ across tested $\beta$; QLoRA reaches $79.5\%$ and retains roughly $77$--$79\%$ at large ratios. Quantization thus preserves the high-ratio failure under \initA and the broad high-accuracy region under \initAA.

\Cref{fig:more_models_lora} shows the same effect across various models. Under standard \initA, increasing the LR ratio becomes fragile as $\beta$ grows: Llama 3.2-1B, Qwen2.5-3B, Qwen3-0.6B, and Mistral-v0.3 all collapse to near-zero accuracy in the high-ratio, large-$\beta$ region, and Llama 3.2-3B shows the same degradation pattern at $\beta \ge 4$. In contrast, \initAA removes most of this diagonal failure region. The high-ratio columns remain competitive across the model families, reaching roughly $58$--$60\%$ on Llama 3.2-1B, $72$--$74\%$ on Llama 3.2-3B, $57$--$66\%$ on Qwen2.5-3B, $72$--$74\%$ on Qwen3-0.6B, and $71$--$73\%$ on Mistral-v0.3. These results indicate that the noise-suppression benefit of \initAA is not specific to one backbone, but appears consistently across Llama, Qwen, and Mistral models.

Panel~(a) of \cref{fig:llama_additional_subfigures} repeats the same stress test on Llama 3.2-1B and extends the ratio grid to $20000$. Standard \initA again exhibits a drifting optimum and severe high-ratio collapse: the best cells lie around ratios $10$--$200$, while the ratio-$20000$ column is near zero for all tested $\beta$. \initAA substantially improves the stable region, with peak accuracies around $60$--$62\%$ at large ratios, but very large $\beta$ still degrades at the extreme ratio. More aggressive scalings, \texttt{Init[$\text{A}^3$]} and \texttt{Init[$\text{A}^4$]}, further flatten the dependence on $\beta$ and keep accuracy near $61$--$62\%$ at the largest ratios. The tradeoff is visible at small ratios: these stronger scalings start near $30\%$ at ratio $1$, indicating that excessive suppression of the initial $A$ scale requires a larger $\eta_B/\eta_A$ ratio to recover signal learning. Overall, the smaller-model experiment supports the same mechanism: lower initialization variance suppresses high-ratio noise, while the exponent balances noise control and adaptation speed.

\subsection{Long-Horizon Arithmetic Fine-Tuning}
\label{app:long_horizon}

The long-horizon runs compare training budgets of up to 10 epochs. Each horizon is trained separately from initialization, with its LR schedule set for that horizon. We first use a Llama 3-8B arithmetic stress configuration with $r=\alpha=8$, $\beta=1$, $\eta_A=10^{-5}$, and $\rho=\eta_B/\eta_A=4000$; other settings follow \cref{app:decoder_only_setup}.

\begin{table}[!ht]
\centering
\small
\setlength{\tabcolsep}{4pt}
\caption{Separate training-horizon runs on Llama 3-8B arithmetic at $\rho=4000$. Entries report final accuracy (\%) / training loss; the gap is \initAA minus \initA accuracy.}
\label{tab:long_horizon_ratio}
\begin{tabular}{rccc}
\toprule
Epoch & \initA & \initAA & Accuracy gap \\
\midrule
1  & 55.06 / 0.687 & 75.61 / 0.348 & $+20.55$ \\
3  & 56.71 / 0.697 & 77.57 / 0.268 & $+20.86$ \\
6  & 55.73 / 0.687 & 77.30 / 0.182 & $+21.58$ \\
10 & 55.58 / 0.661 & 75.68 / 0.128 & $+20.10$ \\
\bottomrule
\end{tabular}
\end{table}

Across the four horizons in \cref{tab:long_horizon_ratio}, the reported accuracy gap is about $20$--$22$ points. Standard \initA remains below \initAA in the 10-epoch run, while \initAA reaches lower final training loss at longer horizons. Thus, the gap persists across the tested training budgets; these runs do not establish that all \initA configurations behave similarly.

We next compare training horizons on Llama 3-8B, using the nominal LRs and initialization scale from \cref{fig:alpha_plot}: $\beta=0.5$, $\eta_A=10^{-5}$, and $\eta_B=2\times10^{-3}$ with $\alpha/r\in\{4,12\}$. Each run has its own horizon-dependent schedule.

\begin{table*}[t]
\centering
\small
\setlength{\tabcolsep}{7pt}
\caption{Training-horizon comparison for the Llama 3-8B multiplier experiment. Each entry reports final arithmetic accuracy (\%) / training loss from a separate run.}
\label{tab:long_horizon_multiplier}
\begin{tabular}{rcccc}
\toprule
Epoch & \makecell{$\alpha/r=4$\\\initA} & \makecell{$\alpha/r=4$\\\initAA} & \makecell{$\alpha/r=12$\\\initA} & \makecell{$\alpha/r=12$\\\initAA} \\
\midrule
1  & 75.40 / 0.360 & 74.43 / 0.360 & 71.33 / 0.419 & 75.24 / 0.356 \\
3  & 75.27 / 0.256 & 77.40 / 0.294 & 71.18 / 0.339 & 76.89 / 0.279 \\
5  & 74.39 / 0.199 & 76.34 / 0.239 & 70.57 / 0.298 & 75.75 / 0.216 \\
10 & 68.72 / 0.135 & 75.63 / 0.144 & 68.33 / 0.249 & 72.64 / 0.132 \\
\bottomrule
\end{tabular}
\end{table*}

At $\alpha/r=4$, \initA is $0.97$ points higher in the 1-epoch run, but \initAA is higher at the longer reported horizons. At $\alpha/r=12$, \initAA is higher at every horizon. Both methods reach lower training loss with longer budgets, while an accuracy gap remains. These separate runs vary both duration and schedule and lack seed-level uncertainty; they do not isolate early stopping as a cause. They provide a limited horizon comparison, with the broader robustness claim supported by the full sweeps.

\subsection{Multi-Seed Accuracy and Tuning Robustness}
\label{app:multiseed_validation}

\paragraph{Five-seed multiplier validation.}
We repeat selected multiplier configurations on Llama 3-8B arithmetic fine-tuning with five seeds, $r=8$, $\beta=0.5$, $\eta_A=10^{-5}$, and $\rho=\eta_B/\eta_A=200$. \Cref{tab:multiseed_multiplier} reports accuracy and the 95\% confidence interval for the \initAA-minus-\initA difference.

\begin{table}[htbp]
\centering
\small
\setlength{\tabcolsep}{7pt}
\caption{Five-seed Llama 3-8B arithmetic accuracy (\%) at fixed LR ratio $\rho=200$. Accuracy entries are mean $\pm$ SD. The final column gives the reported 95\% CI for the accuracy difference, in percentage points.}
\label{tab:multiseed_multiplier}
\begin{tabular}{rccc}
\toprule
$\alpha/r$ & \initA & \initAA & \makecell{95\% CI for\\\initAA $-$ \initA} \\
\midrule
1  & $76.35\pm0.63$ & $76.42\pm0.67$ & $[-0.28,\,0.43]$ \\
3  & $75.71\pm0.23$ & $76.71\pm0.70$ & $[0.15,\,1.84]$ \\
10 & $71.69\pm0.70$ & $76.42\pm0.16$ & $[3.80,\,5.65]$ \\
12 & $71.44\pm0.52$ & $76.87\pm0.48$ & $[4.68,\,6.20]$ \\
\bottomrule
\end{tabular}
\end{table}

At $\alpha/r=1$, the interval includes zero, so the small difference does not support an improvement claim. At $\alpha/r=10$ and $12$, the intervals remain above zero: the large-multiplier gap persists across seeds. These fixed-LR results support stress-regime robustness without establishing best-tuned performance.

\paragraph{Five-seed usable-region validation.}
For RoBERTa-base on SST-2, we run five seeds at each of 25 LR-grid cells per initialization with $\beta=8$. \Cref{tab:multiseed_roberta} reports peak accuracy and usable-cell counts.

\begin{table}[htbp]
\centering
\small
\setlength{\tabcolsep}{5pt}
\caption{Five-seed RoBERTa/SST-2 LR-grid summary at $\beta=8$. Best accuracy (\%) is mean $\pm$ SD. Counts are reported out of 25 cells; the near-best threshold is one percentage point below each method's own best.}
\label{tab:multiseed_roberta}
\begin{tabular}{lcccc}
\toprule
Initialization & Best accuracy & \makecell{Within 1 point\\of own best} & \makecell{Above 90\%\\accuracy} & \makecell{Chance-level\\cells} \\
\midrule
\initA  & $93.28\pm0.64$ & $5/25$ & $8/25$  & $14/25$ \\
\initAA & $93.53\pm0.52$ & $8/25$ & $16/25$ & $5/25$ \\
\bottomrule
\end{tabular}
\end{table}

The best scores are similar relative to the reported uncertainty, while \initAA increases near-best cells from $5$ to $8$, doubles cells above $90\%$ accuracy, and reduces chance-level cells from $14$ to $5$. This supports a wider usable LR region rather than a universal gain in peak accuracy.

\paragraph{Three-seed matched LR retuning.}
To distinguish initialization robustness from a fixed-LR mismatch, we retune $\eta_B$ in the multiplier experiment while keeping $\eta_A=10^{-5}$. Both methods search the same ratios $\rho\in\{0.1,0.5,1,50,100\}$ with three seeds per cell; \cref{tab:matched_lr_multiseed} reports the best settings.

\begin{table}[htbp]
\centering
\small
\setlength{\tabcolsep}{8pt}
\caption{Three-seed matched LR retuning at two LoRA multipliers. Accuracy (\%) is mean $\pm$ SD; the best ratio is selected from the same grid for both methods.}
\label{tab:matched_lr_multiseed}
\begin{tabular}{rlrc}
\toprule
$\alpha/r$ & Initialization & Best $\rho$ & Best accuracy \\
\midrule
4  & \initA  & 50  & $76.39\pm0.62$ \\
4  & \initAA & 50  & $76.83\pm0.36$ \\
\midrule
12 & \initA  & 50  & $75.51\pm0.34$ \\
12 & \initAA & 100 & $76.84\pm0.31$ \\
\bottomrule
\end{tabular}
\end{table}

At $\alpha/r=4$, the best scores are similar after retuning. At $\alpha/r=12$, retuning recovers much of the fixed-ratio loss of \initA, while \initAA remains $1.33$ points higher within the searched grid. Together, these experiments separate robustness to fixed aggressive HPs from the best attainable score after matched tuning. Most other LR-ratio sweeps remain single-run, so small differences in those sweeps should not be interpreted as established improvements.

\subsection{Signed Signal--Noise Energy Measurements}
\label{app:empirical_energy_decomposition}

We measure the realized noise, signal, and signed cross energies defined after Lemma~1 on Llama 3.2-1B and Llama 3-8B, using Math-10K, three seeds, and three epochs. \Cref{tab:empirical_energy_decomposition} reports each component's final mean percentage contribution to $E_{\mathrm{total}}$. The identity was checked to relative error below $6\times10^{-7}$, confirming numerical consistency but not the mechanism alone.

\begin{table}[htbp]
\centering
\small
\setlength{\tabcolsep}{7pt}
\caption{Final mean energy contributions (\%) over three seeds. Noise and signal energies are nonnegative, while the cross term is signed; noise can exceed 100\% when the cross term is destructive. Rounded entries need not sum to exactly 100\%.}
\label{tab:empirical_energy_decomposition}
\begin{tabular}{llrrr}
\toprule
Model & Initialization & Noise & Signal & Cross \\
\midrule
Llama-3.2-1B & \initA  & 117.2 & 4.9  & $-22.1$ \\
Llama-3.2-1B & \initAA & 3.1   & 85.9 & $+11.1$ \\
\midrule
Llama 3-8B & \initA  & 117.4 & 8.1  & $-25.5$ \\
Llama 3-8B & \initAA & 1.5   & 90.4 & $+8.1$ \\
\bottomrule
\end{tabular}
\end{table}

Under \initA, initialization noise dominates and the cross term is destructive. Under \initAA, the signal contributes approximately $86$--$90\%$ of total energy and the cross term is constructive. The cross term is nonzero under both initializations, supporting the use of deterministic dominance bounds rather than cancellation. The noise/signal ordering emerges within the first $5\%$ of an epoch and persists through training, consistent with \cref{fig:llama_feature_decomposition}. These two finite widths show a transition to signal dominance, without directly verifying the asymptotic $1/n$ noise-suppression rate.

\section{Comparison with Other LoRA Initializations}
\label{app:comparison_with_others}
We report matched-setting comparisons that avoid method-specific tuned LR advantages.

\subsection{Comparison with \initAB}
\label{app:comparison_with_AB}

\begin{figure}[!h]
    \centering
    \includegraphics[width=1\linewidth]{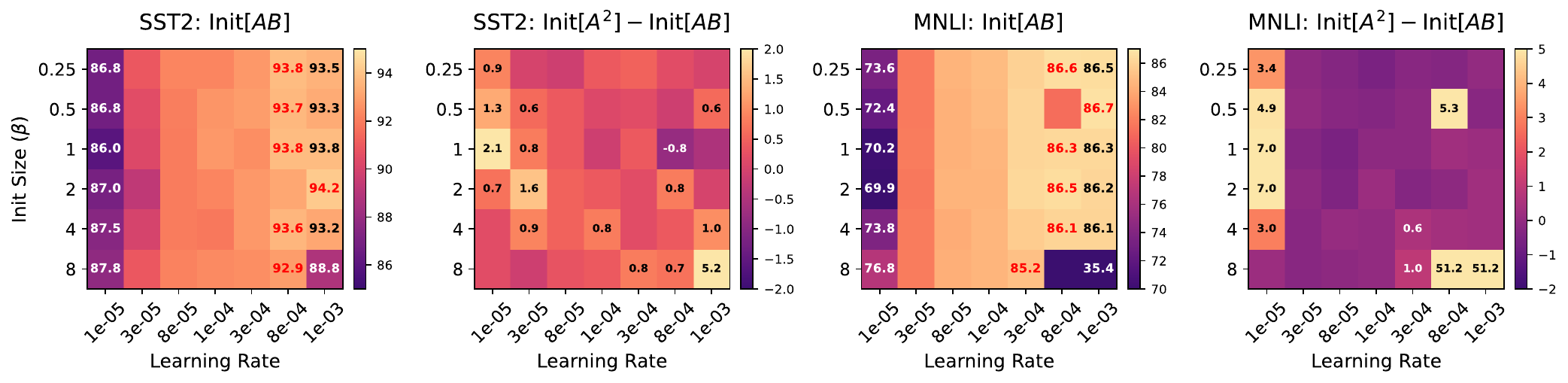}
    \caption{MNLI and SST-2 accuracy under \initAA and \initAB with symmetric LRs ($\eta_A = \eta_B$). Annotated cells have gaps $|\Delta| > 0.5\%$.}
    \label{fig:AB_A_comparison_heatmap}
\end{figure}

\citet{libeyond} apply the $\gamma$-calculus to derive another initialization strategy, \initAB, that can be more robust to suboptimal LRs. The robustness is attributed to the condition:
\begin{equation*}
    \gamma[A_0] + \gamma[B_0] = -1.
\end{equation*}
Strictly speaking, since $B_0=0$, $\gamma[B_0]=-\infty$ under the $\gamma$-calculus convention. The relevant comparison is instead with the first nonzero $B$ update, whose target/update scale under $\eta_B=\Theta(1)$ is $\gamma[\Delta B]=0$. Thus \initAA attains the desired post-update scaling, $\gamma[A_0]=-1$ and $\gamma[\Delta B]=0$, with zero initial adapter output. Its two advantages over \initAB are:
\begin{enumerate}
    \item \textbf{Internal Stability}: \initAA maintains bounded internal features $Z_A$, preventing the feature explosion characteristic of \initAB (\cref{fig:norms})
    \item \textbf{Implementation Efficiency}: Unlike \initAB, which necessitates nonzero initialization for both matrices and subsequent weight subtraction to preserve the pretrained mapping, \initAA allows $B=0$. This simplifies implementation and avoids inference overhead.
\end{enumerate}
To validate these advantages, we compare the two strategies on SST-2 and MNLI across varying initialization scales $\beta \in \{0.25,0.5,1,2,4,8\}$ and symmetric LRs $\eta_A = \eta_B = \eta \in \{1 \times 10^{-5}, 3 \times 10^{-5}, 8 \times 10^{-5}, 1 \times 10^{-4}, 3 \times 10^{-4}, 8 \times 10^{-4}, 1 \times 10^{-3}\}$. In \cref{fig:AB_A_comparison_heatmap}, the methods perform similarly in standard regimes, while \initAA has higher accuracy at the extreme LRs ($10^{-5}, 10^{-3}$) when $\beta=8$. This is consistent with greater internal stability under \initAA.

\subsection{Equal-LR Sweep on Math Benchmark}
\label{app:math_equal_lr}

\begin{table}[t]
\centering
\small
\setlength{\tabcolsep}{4pt}
\caption{Average accuracy (\%) on the math benchmark across initialization methods under equal LRs $\eta_A=\eta_B=\eta$. Best result in each column is bolded.}
\label{tab:math-init-lr}
\resizebox{\linewidth}{!}{%
\begin{tabular}{lccccccccc}
\toprule
Initialization & $1\mathrm{e}{-6}$ & $3\mathrm{e}{-6}$ & $1\mathrm{e}{-5}$ & $3\mathrm{e}{-5}$ & $1\mathrm{e}{-4}$ & $3\mathrm{e}{-4}$ & $1\mathrm{e}{-3}$ & $3\mathrm{e}{-3}$ & $8\mathrm{e}{-3}$ \\
\midrule
EVA & 24.23 & 29.36 & 34.01 & 38.82 & 46.01 & 55.95 & 61.78 & 60.15 & 0.00 \\
OLoRA & 29.19 & 30.29 & 33.19 & 38.61 & 52.13 & 58.23 & 60.07 & 58.22 & 7.28 \\
PiSSA & 30.36 & \textbf{34.03} & \textbf{40.25} & 47.19 & 54.77 & 58.42 & 60.97 & 58.80 & \textbf{22.30} \\
\initA ($\beta=1$) & 23.51 & 28.69 & 34.74 & 40.91 & 48.26 & 57.83 & \textbf{62.15} & 60.72 & 2.41 \\
\initA ($\beta=8$) & \textbf{30.88} & 32.85 & 39.56 & \textbf{47.90} & \textbf{56.08} & \textbf{60.27} & 60.15 & 57.10 & 0.53 \\
\initAA ($\beta=1$) & 24.13 & 28.33 & 31.02 & 35.22 & 45.51 & 51.48 & 60.77 & \textbf{60.98} & 2.14 \\
\initAA ($\beta=8$) & 23.35 & 28.57 & 31.84 & 36.72 & 45.62 & 52.81 & 61.93 & 60.43 & 5.43 \\
\initAB ($\beta=1$) & 23.00 & 31.49 & 35.38 & 42.38 & 52.43 & 58.78 & 61.64 & 60.22 & 4.09 \\
\initAB ($\beta=8$) & 16.42 & 0.91 & 2.86 & 0.00 & 2.01 & 4.13 & 5.43 & 4.19 & 1.25 \\
\bottomrule
\end{tabular}%
}
\end{table}

\Cref{tab:math-init-lr} shows that no initialization dominates across equal LRs; the best method varies with the LR. The smaller initial $A$ scale can slow early adaptation for \initAA, but it remains competitive near the best equal-LR results and avoids the severe large-$\beta$ degradation of \initAB.

\subsection{Variance-Compressed Structured Initializers under Asymmetric LRs}
\label{app:variance_compressed_initializers}

The \initAA philosophy is not tied to Gaussian initialization. For a layer with fan-in $n_\ell$, define
\[
\operatorname{RMS}(A_{0,\ell})
:=\frac{\|A_{0,\ell}\|_{F}}{\sqrt{r n_\ell}}.
\]
An $A^2$-style variant targets $\operatorname{RMS}(A_{0,\ell})=\beta/n_\ell$ while retaining the base initializer's structure. We apply this to three initializers:
\begin{itemize}
    \item \textbf{Orthogonal-$A^2$.} For $Q_A\in\mathbb{R}^{r\times n_\ell}$ with orthonormal rows, set $A_{0,\ell}^{(2)}=(\beta/\sqrt{n_\ell})Q_A$ and $B_{0,\ell}^{(2)}=0$. Then $\operatorname{RMS}(A_{0,\ell}^{(2)})=\beta/n_\ell$ with orthogonal $A$-directions and zero initial adapter output.
    \item \textbf{EVA-$A^2$.} Rescale EVA's activation-derived $A_{0,\ell}$ \citep{paischer2024parameter} by $c_\ell=(\beta/n_\ell)/\operatorname{RMS}(A_{0,\ell})$. Set $A_{0,\ell}^{(2)}=c_\ell A_{0,\ell}$ and $B_{0,\ell}^{(2)}=0$, preserving its row space.
    \item \textbf{PiSSA-$A^2$.} PiSSA initializes both factors from leading singular components of the pretrained weight and compensates for their initial product in the stored base weight \citep{meng2024pissa}. We use the same $c_\ell$ but apply the gauge transformation $A_{0,\ell}^{(2)}=c_\ell A_{0,\ell}$ and $B_{0,\ell}^{(2)}=c_\ell^{-1}B_{0,\ell}$. Hence $B_{0,\ell}^{(2)}A_{0,\ell}^{(2)}=B_{0,\ell}A_{0,\ell}$: the PiSSA subspaces, initialized low-rank product, and initial model function are preserved, although the optimization geometry changes.
\end{itemize}

We evaluate the three base initializers and their variance-compressed variants on Llama 3-8B arithmetic fine-tuning using the setup in \cref{app:decoder_only_setup}. We sweep $\rho=\eta_B/\eta_A\in\{1,200,4000\}$ and use $\beta\in\{0.5,1\}$ for the compressed variants.

\begin{table}[!h]
\centering
\small
\setlength{\tabcolsep}{7pt}
\caption{Arithmetic accuracy (\%) for structured LoRA initializers and their variance-compressed $A^2$ variants under asymmetric LR ratios. Results are single-run values; the best result within each initializer family and ratio is bolded.}
\label{tab:variance_compressed_initializers}
\begin{tabular}{llccc}
\toprule
Initializer & $\beta$ & $\rho=1$ & $\rho=200$ & $\rho=4000$ \\
\midrule
\initAA & 0.5 & \textbf{71.44} & \textbf{76.72} & 76.70 \\
\initAA & 1 & 71.24 & 76.63 & \textbf{77.93} \\
\midrule
PEFT Orthogonal & -- & \textbf{72.73} & 70.06 & 2.49 \\
Orthogonal-$A^2$ & 0.5 & 71.35 & 76.76 & 77.72 \\
Orthogonal-$A^2$ & 1 & 71.66 & \textbf{77.23} & \textbf{77.74} \\
\midrule
EVA & -- & \textbf{75.19} & 74.26 & 5.10 \\
EVA-$A^2$ & 0.5 & 71.56 & \textbf{77.28} & \textbf{77.19} \\
EVA-$A^2$ & 1 & 71.82 & 76.85 & 76.53 \\
\midrule
PiSSA & -- & \textbf{72.87} & \textbf{74.80} & 48.10 \\
PiSSA-$A^2$ & 0.5 & 68.90 & 69.52 & 69.90 \\
PiSSA-$A^2$ & 1 & 72.74 & 73.59 & \textbf{73.80} \\
\bottomrule
\end{tabular}
\end{table}

At $\rho=4000$, the uncompressed Orthogonal, EVA, and PiSSA initializers fall to $2.49\%$, $5.10\%$, and $48.10\%$, while their best compressed variants retain $77.74\%$, $77.19\%$, and $73.80\%$. Compression is not uniformly beneficial at $\rho\in\{1,200\}$. These single-run results suggest that controlling $A_0$ scale helps structured initializers tolerate aggressive asymmetric LRs; small differences remain uncertain.


\section{Responsible Research Statements}
\label{app:responsible_research}
\subsection{Broader Impact}
\label{app:broader_impact}
We study LoRA initialization to improve the understanding and robustness of parameter-efficient fine-tuning. We identify no direct negative societal impacts unique to this initialization; broader impacts depend on the underlying pretrained models and their downstream uses.

\subsection{LLM Usage}
\label{app:llm_usage}
We used large language models to polish the language and improve presentation clarity. We also use Codex to clean up the code and maintain the code base. The core research ideas, technical content, theoretical analysis, experimental design, and scientific conclusions were generated by the authors.

\clearpage
\section*{NeurIPS Paper Checklist}

\begin{enumerate}

\item {\bf Claims}
    \item[] Question: Do the main claims made in the abstract and introduction accurately reflect the paper's contributions and scope?
    \item[] Answer: \answerYes{}
    \item[] Justification: The abstract and introduction state the proposed \initAA initialization, the signal--noise mechanism, and the empirical claims. The analysis and experiments in the main paper and appendix support these claims within the stated assumptions and evaluated settings.
    \item[] Guidelines:
    \begin{itemize}
        \item The answer \answerNA{} means that the abstract and introduction do not include the claims made in the paper.
        \item The abstract and/or introduction should clearly state the claims made, including the contributions made in the paper and important assumptions and limitations. A \answerNo{} or \answerNA{} answer to this question will not be perceived well by the reviewers. 
        \item The claims made should match theoretical and experimental results, and reflect how much the results can be expected to generalize to other settings. 
        \item It is fine to include aspirational goals as motivation as long as it is clear that these goals are not attained by the paper. 
    \end{itemize}

\item {\bf Limitations}
    \item[] Question: Does the paper discuss the limitations of the work performed by the authors?
    \item[] Answer: \answerYes{}
    \item[] Justification: The conclusion discusses open limitations, including the focus on supervised fine-tuning, the lack of a direct rule for choosing asymmetric LR ratios across arbitrary models and tasks, the need for a fuller theory of other LoRA HPs such as rank $r$, and the uncertainty around more aggressive initialization scaling.
    \item[] Guidelines:
    \begin{itemize}
        \item The answer \answerNA{} means that the paper has no limitation while the answer \answerNo{} means that the paper has limitations, but those are not discussed in the paper. 
        \item The authors are encouraged to create a separate ``Limitations'' section in their paper.
        \item The paper should point out any strong assumptions and how robust the results are to violations of these assumptions (e.g., independence assumptions, noiseless settings, model well-specification, asymptotic approximations only holding locally). The authors should reflect on how these assumptions might be violated in practice and what the implications would be.
        \item The authors should reflect on the scope of the claims made, e.g., if the approach was only tested on a few datasets or with a few runs. In general, empirical results often depend on implicit assumptions, which should be articulated.
        \item The authors should reflect on the factors that influence the performance of the approach. For example, a facial recognition algorithm may perform poorly when image resolution is low or images are taken in low lighting. Or a speech-to-text system might not be used reliably to provide closed captions for online lectures because it fails to handle technical jargon.
        \item The authors should discuss the computational efficiency of the proposed algorithms and how they scale with dataset size.
        \item If applicable, the authors should discuss possible limitations of their approach to address problems of privacy and fairness.
        \item While the authors might fear that complete honesty about limitations might be used by reviewers as grounds for rejection, a worse outcome might be that reviewers discover limitations that aren't acknowledged in the paper. The authors should use their best judgment and recognize that individual actions in favor of transparency play an important role in developing norms that preserve the integrity of the community. Reviewers will be specifically instructed to not penalize honesty concerning limitations.
    \end{itemize}

\item {\bf Theory assumptions and proofs}
    \item[] Question: For each theoretical result, does the paper provide the full set of assumptions and a complete (and correct) proof?
    \item[] Answer: \answerYes{}
    \item[] Justification: The main text states the informal theorem assumptions, and the appendix provides the full decomposition, normalized-update assumptions, Adam/AdamW conditions, and proofs.
    \item[] Guidelines:
    \begin{itemize}
        \item The answer \answerNA{} means that the paper does not include theoretical results. 
        \item All the theorems, formulas, and proofs in the paper should be numbered and cross-referenced.
        \item All assumptions should be clearly stated or referenced in the statement of any theorems.
        \item The proofs can either appear in the main paper or the supplemental material, but if they appear in the supplemental material, the authors are encouraged to provide a short proof sketch to provide intuition. 
        \item Inversely, any informal proof provided in the core of the paper should be complemented by formal proofs provided in appendix or supplemental material.
        \item Theorems and Lemmas that the proof relies upon should be properly referenced. 
    \end{itemize}

    \item {\bf Experimental result reproducibility}
    \item[] Question: Does the paper fully disclose all the information needed to reproduce the main experimental results of the paper to the extent that it affects the main claims and/or conclusions of the paper (regardless of whether the code and data are provided or not)?
    \item[] Answer: \answerYes{}
    \item[] Justification: The experimental setup, datasets, model choices, LoRA HPs, optimizer settings, and sweep ranges are described in the main experiment section and in the experimental-details appendix. Reconstruction details and limitations are documented in the public experiment guide and \cref{app:toy_setup}.
    \item[] Guidelines:
    \begin{itemize}
        \item The answer \answerNA{} means that the paper does not include experiments.
        \item If the paper includes experiments, a \answerNo{} answer to this question will not be perceived well by the reviewers: Making the paper reproducible is important, regardless of whether the code and data are provided or not.
        \item If the contribution is a dataset and\slash or model, the authors should describe the steps taken to make their results reproducible or verifiable. 
        \item Depending on the contribution, reproducibility can be accomplished in various ways. For example, if the contribution is a novel architecture, describing the architecture fully might suffice, or if the contribution is a specific model and empirical evaluation, it may be necessary to either make it possible for others to replicate the model with the same dataset, or provide access to the model. In general, releasing code and data is often one good way to accomplish this, but reproducibility can also be provided via detailed instructions for how to replicate the results, access to a hosted model (e.g., in the case of a large language model), releasing of a model checkpoint, or other means that are appropriate to the research performed.
        \item While NeurIPS does not require releasing code, the conference does require all submissions to provide some reasonable avenue for reproducibility, which may depend on the nature of the contribution. For example
        \begin{enumerate}
            \item If the contribution is primarily a new algorithm, the paper should make it clear how to reproduce that algorithm.
            \item If the contribution is primarily a new model architecture, the paper should describe the architecture clearly and fully.
            \item If the contribution is a new model (e.g., a large language model), then there should either be a way to access this model for reproducing the results or a way to reproduce the model (e.g., with an open-source dataset or instructions for how to construct the dataset).
            \item We recognize that reproducibility may be tricky in some cases, in which case authors are welcome to describe the particular way they provide for reproducibility. In the case of closed-source models, it may be that access to the model is limited in some way (e.g., to registered users), but it should be possible for other researchers to have some path to reproducing or verifying the results.
        \end{enumerate}
    \end{itemize}

\item {\bf Open access to data and code}
    \item[] Question: Does the paper provide open access to the data and code, with sufficient instructions to faithfully reproduce the main experimental results, as described in supplemental material?
    \item[] Answer: \answerYes{}
    \item[] Justification: The paper uses public datasets and models. Code, setup instructions, and data-access recipes are publicly available at \url{https://github.com/liaodisen/LoRA-initA2}. The experimental setup and hyperparameter sweeps are documented in \cref{app:experimental_setup}; the public experiment guide documents reconstruction details and limitations.
    \item[] Guidelines:
    \begin{itemize}
        \item The answer \answerNA{} means that paper does not include experiments requiring code.
        \item Please see the NeurIPS code and data submission guidelines (\url{https://neurips.cc/public/guides/CodeSubmissionPolicy}) for more details.
        \item While we encourage the release of code and data, we understand that this might not be possible, so \answerNo{} is an acceptable answer. Papers cannot be rejected simply for not including code, unless this is central to the contribution (e.g., for a new open-source benchmark).
        \item The instructions should contain the exact command and environment needed to run to reproduce the results. See the NeurIPS code and data submission guidelines (\url{https://neurips.cc/public/guides/CodeSubmissionPolicy}) for more details.
        \item The authors should provide instructions on data access and preparation, including how to access the raw data, preprocessed data, intermediate data, and generated data, etc.
        \item The authors should provide scripts to reproduce all experimental results for the new proposed method and baselines. If only a subset of experiments are reproducible, they should state which ones are omitted from the script and why.
        \item At submission time, to preserve anonymity, the authors should release anonymized versions (if applicable).
        \item Providing as much information as possible in supplemental material (appended to the paper) is recommended, but including URLs to data and code is permitted.
    \end{itemize}

\item {\bf Experimental setting/details}
    \item[] Question: Does the paper specify all the training and test details (e.g., data splits, hyperparameters, how they were chosen, type of optimizer) necessary to understand the results?
    \item[] Answer: \answerYes{}
    \item[] Justification: The main text and appendix specify the benchmark tasks, models, LoRA modules, rank, scaling factor, LR sweeps, optimizers, batch sizes, sequence lengths, and training epochs.
    \item[] Guidelines:
    \begin{itemize}
        \item The answer \answerNA{} means that the paper does not include experiments.
        \item The experimental setting should be presented in the core of the paper to a level of detail that is necessary to appreciate the results and make sense of them.
        \item The full details can be provided either with the code, in appendix, or as supplemental material.
    \end{itemize}

\item {\bf Experiment statistical significance}
    \item[] Question: Does the paper report error bars suitably and correctly defined or other appropriate information about the statistical significance of the experiments?
    \item[] Answer: \answerNo{}
    \item[] Justification: The appendix reports selected three- and five-seed accuracy summaries and confidence intervals, while most main language-understanding and generation heatmaps remain single-run. The reported uncertainty types and interval-construction methods are not yet fully documented.
    \item[] Guidelines:
    \begin{itemize}
        \item The answer \answerNA{} means that the paper does not include experiments.
        \item The authors should answer \answerYes{} if the results are accompanied by error bars, confidence intervals, or statistical significance tests, at least for the experiments that support the main claims of the paper.
        \item The factors of variability that the error bars are capturing should be clearly stated (for example, train/test split, initialization, random drawing of some parameter, or overall run with given experimental conditions).
        \item The method for calculating the error bars should be explained (closed form formula, call to a library function, bootstrap, etc.)
        \item The assumptions made should be given (e.g., Normally distributed errors).
        \item It should be clear whether the error bar is the standard deviation or the standard error of the mean.
        \item It is OK to report 1-sigma error bars, but one should state it. The authors should preferably report a 2-sigma error bar than state that they have a 96\% CI, if the hypothesis of Normality of errors is not verified.
        \item For asymmetric distributions, the authors should be careful not to show in tables or figures symmetric error bars that would yield results that are out of range (e.g., negative error rates).
        \item If error bars are reported in tables or plots, the authors should explain in the text how they were calculated and reference the corresponding figures or tables in the text.
    \end{itemize}

\item {\bf Experiments compute resources}
    \item[] Question: For each experiment, does the paper provide sufficient information on the computer resources (type of compute workers, memory, time of execution) needed to reproduce the experiments?
    \item[] Answer: \answerYes{}
    \item[] Justification: The appendix reports the GPU type used for the experiments. Each individual fine-tuning run is relatively cheap and can be completed within roughly three hours on a single L40s GPU, but the full set of reported experiments requires substantially more total compute because we sweep many combinations of initialization scale, LR ratio, rank, and LoRA multiplier.
    \item[] Guidelines:
    \begin{itemize}
        \item The answer \answerNA{} means that the paper does not include experiments.
        \item The paper should indicate the type of compute workers CPU or GPU, internal cluster, or cloud provider, including relevant memory and storage.
        \item The paper should provide the amount of compute required for each of the individual experimental runs as well as estimate the total compute. 
        \item The paper should disclose whether the full research project required more compute than the experiments reported in the paper (e.g., preliminary or failed experiments that didn't make it into the paper). 
    \end{itemize}
    
\item {\bf Code of ethics}
    \item[] Question: Does the research conducted in the paper conform, in every respect, with the NeurIPS Code of Ethics \url{https://neurips.cc/public/EthicsGuidelines}?
    \item[] Answer: \answerYes{}
    \item[] Justification: The work studies initialization and optimization behavior for parameter-efficient fine-tuning and does not involve data collection from human subjects, unsafe data release, or intentionally harmful applications.
    \item[] Guidelines:
    \begin{itemize}
        \item The answer \answerNA{} means that the authors have not reviewed the NeurIPS Code of Ethics.
        \item If the authors answer \answerNo{}, they should explain the special circumstances that require a deviation from the Code of Ethics.
        \item The authors should make sure to preserve anonymity (e.g., if there is a special consideration due to laws or regulations in their jurisdiction).
    \end{itemize}

\item {\bf Broader impacts}
    \item[] Question: Does the paper discuss both potential positive societal impacts and negative societal impacts of the work performed?
    \item[] Answer: \answerYes{}
    \item[] Justification: The broader-impact section states that the work aims to improve understanding and robustness of LoRA and that any downstream impacts are inherited from the underlying pretrained models and applications.
    \item[] Guidelines:
    \begin{itemize}
        \item The answer \answerNA{} means that there is no societal impact of the work performed.
        \item If the authors answer \answerNA{} or \answerNo{}, they should explain why their work has no societal impact or why the paper does not address societal impact.
        \item Examples of negative societal impacts include potential malicious or unintended uses (e.g., disinformation, generating fake profiles, surveillance), fairness considerations (e.g., deployment of technologies that could make decisions that unfairly impact specific groups), privacy considerations, and security considerations.
        \item The conference expects that many papers will be foundational research and not tied to particular applications, let alone deployments. However, if there is a direct path to any negative applications, the authors should point it out. For example, it is legitimate to point out that an improvement in the quality of generative models could be used to generate Deepfakes for disinformation. On the other hand, it is not needed to point out that a generic algorithm for optimizing neural networks could enable people to train models that generate Deepfakes faster.
        \item The authors should consider possible harms that could arise when the technology is being used as intended and functioning correctly, harms that could arise when the technology is being used as intended but gives incorrect results, and harms following from (intentional or unintentional) misuse of the technology.
        \item If there are negative societal impacts, the authors could also discuss possible mitigation strategies (e.g., gated release of models, providing defenses in addition to attacks, mechanisms for monitoring misuse, mechanisms to monitor how a system learns from feedback over time, improving the efficiency and accessibility of ML).
    \end{itemize}
    
\item {\bf Safeguards}
    \item[] Question: Does the paper describe safeguards that have been put in place for responsible release of data or models that have a high risk for misuse (e.g., pre-trained language models, image generators, or scraped datasets)?
    \item[] Answer: \answerNA{}
    \item[] Justification: The paper does not release new pretrained models, high-risk datasets, or other assets requiring release safeguards.
    \item[] Guidelines:
    \begin{itemize}
        \item The answer \answerNA{} means that the paper poses no such risks.
        \item Released models that have a high risk for misuse or dual-use should be released with necessary safeguards to allow for controlled use of the model, for example by requiring that users adhere to usage guidelines or restrictions to access the model or implementing safety filters. 
        \item Datasets that have been scraped from the Internet could pose safety risks. The authors should describe how they avoided releasing unsafe images.
        \item We recognize that providing effective safeguards is challenging, and many papers do not require this, but we encourage authors to take this into account and make a best faith effort.
    \end{itemize}

\item {\bf Licenses for existing assets}
    \item[] Question: Are the creators or original owners of assets (e.g., code, data, models), used in the paper, properly credited and are the license and terms of use explicitly mentioned and properly respected?
    \item[] Answer: \answerNo{}
    \item[] Justification: The paper cites the existing datasets, models, and software used in the experiments, but the current draft does not explicitly list their licenses or terms of use.
    \item[] Guidelines:
    \begin{itemize}
        \item The answer \answerNA{} means that the paper does not use existing assets.
        \item The authors should cite the original paper that produced the code package or dataset.
        \item The authors should state which version of the asset is used and, if possible, include a URL.
        \item The name of the license (e.g., CC-BY 4.0) should be included for each asset.
        \item For scraped data from a particular source (e.g., website), the copyright and terms of service of that source should be provided.
        \item If assets are released, the license, copyright information, and terms of use in the package should be provided. For popular datasets, \url{paperswithcode.com/datasets} has curated licenses for some datasets. Their licensing guide can help determine the license of a dataset.
        \item For existing datasets that are re-packaged, both the original license and the license of the derived asset (if it has changed) should be provided.
        \item If this information is not available online, the authors are encouraged to reach out to the asset's creators.
    \end{itemize}

\item {\bf New assets}
    \item[] Question: Are new assets introduced in the paper well documented and is the documentation provided alongside the assets?
    \item[] Answer: \answerNA{}
    \item[] Justification: We do not release a new dataset, pretrained model, or benchmark. The public repository provides the implementation of the proposed method.
    \item[] Guidelines:
    \begin{itemize}
        \item The answer \answerNA{} means that the paper does not release new assets.
        \item Researchers should communicate the details of the dataset\slash code\slash model as part of their submissions via structured templates. This includes details about training, license, limitations, etc. 
        \item The paper should discuss whether and how consent was obtained from people whose asset is used.
        \item At submission time, remember to anonymize your assets (if applicable). You can either create an anonymized URL or include an anonymized zip file.
    \end{itemize}

\item {\bf Crowdsourcing and research with human subjects}
    \item[] Question: For crowdsourcing experiments and research with human subjects, does the paper include the full text of instructions given to participants and screenshots, if applicable, as well as details about compensation (if any)? 
    \item[] Answer: \answerNA{}
    \item[] Justification: The research does not involve crowdsourcing experiments or human subjects.
    \item[] Guidelines:
    \begin{itemize}
        \item The answer \answerNA{} means that the paper does not involve crowdsourcing nor research with human subjects.
        \item Including this information in the supplemental material is fine, but if the main contribution of the paper involves human subjects, then as much detail as possible should be included in the main paper. 
        \item According to the NeurIPS Code of Ethics, workers involved in data collection, curation, or other labor should be paid at least the minimum wage in the country of the data collector. 
    \end{itemize}

\item {\bf Institutional review board (IRB) approvals or equivalent for research with human subjects}
    \item[] Question: Does the paper describe potential risks incurred by study participants, whether such risks were disclosed to the subjects, and whether Institutional Review Board (IRB) approvals (or an equivalent approval/review based on the requirements of your country or institution) were obtained?
    \item[] Answer: \answerNA{}
    \item[] Justification: The research does not involve crowdsourcing experiments or human subjects, so IRB approval is not applicable.
    \item[] Guidelines:
    \begin{itemize}
        \item The answer \answerNA{} means that the paper does not involve crowdsourcing nor research with human subjects.
        \item Depending on the country in which research is conducted, IRB approval (or equivalent) may be required for any human subjects research. If you obtained IRB approval, you should clearly state this in the paper. 
        \item We recognize that the procedures for this may vary significantly between institutions and locations, and we expect authors to adhere to the NeurIPS Code of Ethics and the guidelines for their institution. 
        \item For initial submissions, do not include any information that would break anonymity (if applicable), such as the institution conducting the review.
    \end{itemize}

\item {\bf Declaration of LLM usage}
    \item[] Question: Does the paper describe the usage of LLMs if it is an important, original, or non-standard component of the core methods in this research? Note that if the LLM is used only for writing, editing, or formatting purposes and does \emph{not} impact the core methodology, scientific rigor, or originality of the research, declaration is not required.
    \item[] Answer: \answerNA{}
    \item[] Justification: LLMs are the experimental backbone models, but LLM usage is not an original or non-standard component of the proposed method development.
    \item[] Guidelines:
    \begin{itemize}
        \item The answer \answerNA{} means that the core method development in this research does not involve LLMs as any important, original, or non-standard components.
        \item Please refer to our LLM policy in the NeurIPS handbook for what should or should not be described.
    \end{itemize}

\end{enumerate}

\end{document}